\documentclass[11pt]{article}
\usepackage{arxiv}

\usepackage[utf8]{inputenc}
\usepackage[T1]{fontenc}
\usepackage{microtype}
\usepackage{amsmath,amssymb,amsfonts,amsthm}
\usepackage{graphicx}
\usepackage{enumerate}
\usepackage{xcolor}
\usepackage{natbib}
\usepackage{url}
\usepackage{comment}
\usepackage{mathrsfs}
\usepackage{amsbsy}
\usepackage{mathtools}
\usepackage{latexsym}
\usepackage{bbm}
\usepackage{bm}
\usepackage{nicefrac}
\usepackage{booktabs}
\usepackage{multirow}
\usepackage{siunitx}
\usepackage{enumitem}
\usepackage{wrapfig}
\usepackage[caption=false,subrefformat=parens,labelformat=parens]{subfig}
\usepackage{chemformula}
\usepackage{verbatim}
\usepackage[hidelinks]{hyperref}

\usepackage{algpseudocodex}
\newcounter{breakablealgorithm}
\renewcommand{\thebreakablealgorithm}{\arabic{breakablealgorithm}}

\def\E{{\mathbb E}}

\DeclareMathOperator*{\argmax}{argmax}

\theoremstyle{plain}
\newtheorem{theorem}{Theorem}

\newtheorem{proposition}[theorem]{Proposition}

\newtheorem{assumption}[theorem]{Assumption}

\theoremstyle{definition}

\theoremstyle{remark}

\newtheorem{example}[theorem]{Example}

\title{Bias-Aware Adjoint Calibration and Optimal Control of Stochastic Multiscale Bioprocess Digital Twins}
\renewcommand{\shorttitle}{Bias-Aware Calibration and Control of Bioprocess Twins}
\author
{Keilung Choy\\Northeastern University
\And  Wei Xie \thanks{Corresponding author. Email: w.xie@northeastern.edu} \\
Northeastern University}
\date{}

\begin{document}

\maketitle

\begin{abstract}
We develop a bias-aware digital-twin calibration and control framework for multiscale bioprocess models within a biological systems-of-systems (Bio-SoS) paradigm.
The digital twin is represented by a stochastic differential equation (SDE) model and calibrated from sparse observations using quasi-likelihood estimation and adjoint sensitivity analysis. SDE generator-based moment expansions characterize truncation-induced parameter bias, while forward--backward adjoints quantify how calibration uncertainty propagates to value functions and policy performance. The resulting parameter-error distribution supports both policy-directed data acquisition and uncertainty-aware policy optimization through a second-order Gaussian-averaged objective. We characterize the asymptotic behavior of the resulting exploration criterion and derive a physical-system performance under the optimized policy.
To implement these ideas, we develop an Actor--Simulator algorithm that jointly updates model parameters, selects informative experiments, and optimizes control policies. Numerical studies demonstrate improved calibration accuracy, sample efficiency, and control performance relative to state-of-the-art baselines.

\end{abstract}

\section{INTRODUCTION}
\label{sec:intro}
Advances in biopharmaceutical manufacturing, particularly personalized cell and gene therapies, have increased the need for reliable, adaptive, and data-driven process control. Because living cells serve as production engines, bioprocesses function as Biological Systems-of-Systems (Bio-SoS), where biological, physical, and chemical (a.k.a. \textit{biophysicochemical}) mechanisms interact across molecular, cellular, and macroscopic scales.
These multiscale interactions govern productivity and critical quality attributes (CQAs) while exhibiting strong nonlinearity, stochasticity, environmental sensitivity, and delayed effects.
Consequently, product quality is inherently trajectory-dependent, reflecting the principle that \textit{the process is the product}.
A useful digital twin must therefore capture the multiscale mechanisms, causal pathways, and uncertainties linking control actions to future process outcomes.

Stochastic differential equations (SDEs) provide a natural framework for modeling Bio-SoS dynamics, capturing reaction kinetics, molecular interactions, and intrinsic stochasticity in continuous time.
Guided by domain ontologies \citep{choy2026biological}, these dynamics can be decomposed into reusable mechanistic modules, such as molecular and cellular agents, interaction rules, and reaction regulatory mechanisms, and assembled into multiscale digital twins. This modular architecture enables the rapid development of mechanistic-hybrid digital twins across diverse biomanufacturing platforms, including monoclonal antibodies \citep{wang2025multi,zeng2026glycomac}, cell and gene therapies \citep{wang2024metabolic,zheng2024stochastic,cheng2026intracellular}, vaccines \citep{KeqiIVT}, and nanoparticle delivery systems
\citep{yang2025multi}.
Coupled across scales, the SDE modules form interpretable digital twins in which molecular and cellular mechanisms govern macroscopic process behavior. 
The resulting parameters retain biophysicochemical interpretability, enabling calibration and control sensitivities to be linked to specific pathways and regulatory mechanisms.

We assume that the mechanistic SDE structure is specified from established process knowledge, while unknown parameters are inferred from a limited number of discretely observed state-action trajectories. Since exact finite-time transition densities are generally unavailable for nonlinear systems, calibration relies on Gaussian quasi-likelihoods based on approximate conditional moments. Truncation of these moment expansions introduces systematic approximation bias alongside finite-sample uncertainty. Although additional data reduce estimation variability, they do not eliminate the pseudo-true parameter shift induced by a fixed-order approximation at a nonzero observation interval. Consequently, process-control-oriented calibration must quantify both uncertainty and approximation bias and propagate their effects through policy evaluation and decision-making.

The ultimate goal of digital twin calibration is to support optimal control of the physical systems. Because transient dynamics influence yield and CQAs, steady-state approximations are often inadequate for batch and fed-batch operations.
We therefore embed the calibrated SDE digital twin within a finite-horizon Markov decision process (MDP), where parameter uncertainty drives experiment design and model refinement, while intrinsic stochasticity governs process variability during policy evaluation.
This framework separates mechanism learning from process optimization. At each learning round, we jointly update an \emph{exploration policy} that selects informative physical experiments to improve digital-twin fidelity and a \emph{target control policy} that maximizes a bias-corrected estimate of physical-system performance.

Building on sensitivity analysis and model-based reinforcement learning (RL) for biomanufacturing systems \citep{xie2022sequential,zheng2024stochastic}, we develop a bias-aware Actor--Simulator framework that unifies digital-twin calibration, uncertainty quantification, experiment design, and control optimization. After each physical trajectory, Gaussian quasi-likelihood estimation updates mechanistic parameters while quantifying both finite-sample uncertainty and fixed-interval calibration bias. A shared trajectory-level adjoint framework couples forward simulation with backward stochastic adjoint equations to efficiently compute sensitivities for calibration, uncertainty propagation, and policy evaluation. By propagating sensitivities across process trajectories, the framework preserves Bio-SoS interpretability and interoperability, consistent with the principle that the process is the product.

To improve digital-twin fidelity, we design physical experiments using a policy-relevant uncertainty criterion based on digital--physical $Q$-value discrepancies, prioritizing state--action regions with the greatest impact on control performance. Exploration decisions adapt online to observed states, while accumulated data are recursively incorporated to recalibrate the digital twin, update calibration-error estimates, and redesign subsequent experiments.

For control optimization, we maximize a second-order approximation of the expected return under the calibration-error distribution, explicitly accounting for the propagation of calibration bias and parameter uncertainty through the digital twin's mechanistic pathways and affect policy performance. The resulting framework enables sample-efficient policy learning under sparse observations and model misspecification by unifying uncertainty quantification, adaptive experimentation, and bias-aware control.

Physics-informed digital twins primarily focus on state prediction and parameter inference, while classical computer-model calibration emphasizes input--output discrepancies and surrogate representations \citep{kennedy2001bayesian,raissi2019physics,karniadakis2021physics}. Model-based RL leverages mechanistic models for decision-making but rarely propagates calibration bias and parameter uncertainty into control objectives. Our framework bridges these areas by unifying calibration, uncertainty quantification, adaptive experimentation, and bias-aware policy optimization.
A key feature of the proposed framework is the separation of calibration and control objectives. The exploration sequential design of experiments (DoEs) targets policy-relevant digital--physical value discrepancies to improve digital-twin fidelity, whereas the control policy maximizes an expected return that explicitly incorporates calibration bias and parameter uncertainty.

The primary contributions are:
\begin{enumerate}
\item \textbf{Bias-Aware Calibration.} 
We develop a Gaussian quasi-likelihood based on truncated conditional moments, quantify both finite-sample uncertainty and fixed-interval pseudo-true parameter bias, and derive a one-step correction for the leading truncation error.

\item \textbf{Adjoint Sensitivity Analysis and Value Correction.} We derive forward--backward stochastic adjoint equations for efficient first- and second-order sensitivity analysis. Exploiting the duality between state and parameter sensitivities, the framework scales to high-dimensional Bio-SoS mechanistic models and propagates calibration bias and covariance into a Gaussian-averaged approximation of finite-horizon policy value.

\item \textbf{Actor--Simulator Learning.} 
We develop a trajectory-level framework that separates exploration DoEs from control, establish a near-optimality bound accounting for sampling uncertainty, calibration bias, and moment-truncation error, and demonstrate its effectiveness on multiscale biomanufacturing applications.
\end{enumerate}

Section~\ref{sec:literature review} reviews the related literature. Section~\ref{sec:problemDescription} formulates the multiscale SDE and finite-horizon MDP frameworks and outlines the proposed approach. Section~\ref{sec: Calibration Framework} presents the calibration, uncertainty quantification, adjoint sensitivity analysis, and policy optimization methods. Section~\ref{sec:convergence} establishes the theoretical results, Section~\ref{sec: experiments} reports empirical results, and Section~\ref{sec:conclusion} concludes. Technical proofs and algorithmic details are deferred to the appendix.


\section{Literature Review}
\label{sec:literature review}
This work lies at the intersection of physics-informed digital twins, adaptive experimental design, uncertainty-aware calibration, and model-based optimal control. We review these areas and highlight the gap addressed by the proposed framework.

\textbf{(i) Physics-Informed Digital Twins and Calibration.}
Physics-informed machine learning incorporates differential-equation structure into learning, improving physical consistency and reducing data requirements relative to black-box models \citep{raissi2019physics,karniadakis2021physics}. 
Such models are increasingly deployed as digital twins and calibrated using simulation and experimental data \citep{yang2024data,sel2023physics,barimah2025scalable,kazemzadeh2025physics}. In biomanufacturing, mechanistic and hybrid digital twins support process monitoring, experiment design, and control optimization \citep{zhao2023biodt,choy2025symbolic}. However, existing studies focus on state prediction and parameter inference, with limited attention to how calibration uncertainty and approximation bias in SDE-based digital twins propagate to downstream control decisions.

Classical computer-model calibration, including Gaussian-process discrepancy models \citep{kennedy2001bayesian}, primarily adopts an input--output perspective. In contrast, we consider finite-horizon trajectories generated by continuous-time stochastic nonlinear dynamics. 
A key challenge is that exact SDE transition densities are rarely available, requiring finite-interval approximations of conditional moments. These approximations can induce a pseudo-true parameter bias that persists even as the sample size grows. Our focus is therefore on quantifying this calibration bias, correcting its leading effects, and assessing its impact on downstream decision-making.

\textbf{(ii) Adjoint Sensitivity Analysis for SDEs.}
Adjoint methods compute parameter sensitivities by propagating system states forward and adjoint variables backward, avoiding separate sensitivity equations for each parameter. They underpin differentiable simulators and neural differential equations \citep{chen2018neuralode,rackauckas2020universal}, with extensions to stochastic systems enabling scalable SDE sensitivity analysis \citep{li2020scalable}. We leverage stochastic adjoint equations to efficiently compute first- and second-order sensitivities of value functions and policy returns, enabling uncertainty and bias propagation to control performance while accelerating policy optimization.

\textbf{(iii) Model-Based Control and Adaptive Experimentation.}
Model-based reinforcement learning and model predictive control use learned dynamics to improve decision-making while reducing reliance on costly experiments \citep{deisenroth2011pilco,moerland2023model}. Bayesian RL further exploits parameter uncertainty to balance exploration and exploitation \citep{russo2018tutorial,moradipari2023improved}. In biomanufacturing, recent studies have combined mechanistic models, sequential experimentation, and policy optimization \citep{xie2022sequential,zheng2024stochastic}, while physics-informed surrogates have been integrated into feedback-control frameworks \citep{nicodemus2022physics}.

Despite these advances, existing approaches rarely separate the effects of calibration uncertainty and calibration bias on decision quality. In our setting, both contribute to the digital--physical value mismatch driving experiment design, whereas target-policy optimization depends on both directional return shifts induced by bias and curvature effects induced by parameter uncertainty. The proposed Actor--Simulator framework explicitly distinguishes these roles: exploration is guided by a policy-relevant value-discrepancy metric, while exploitation employs a bias-corrected return approximation. This unifies SDE calibration, uncertainty quantification, adaptive experimentation, and finite-horizon control within a single decision-oriented digital-twin framework.


\section{Problem Description and Proposed Framework}
\label{sec:problemDescription}

We first formulate physical and digital-twin MDPs for a multiscale bioprocess, then outline an Actor--Simulator framework that leverages sparse trajectory observations to calibrate Bio-SoS dynamics and optimize control under parameter uncertainty and model bias.

\subsection{Physical System and Digital Twin}
\label{sec:multi-scaleBioprocess}

Biomanufacturing systems operate as complex Bio-SoS, where biophysicochemical mechanisms interact across molecular, cellular, and macroscopic scales. Consistent with the principle that the process is the product, process conditions continuously shape molecular structure and interactions, cellular behavior, and the local microenvironment, ultimately determining product yield and CQAs. A representative application is biomanufacturing such as induced pluripotent stem cell (iPSC) aggregate culture, where intracellular metabolism, cell aggregation, extracellular transport, population dynamics, and bioreactor heterogeneity are tightly coupled \citep{kwok2018scalable,zheng2024stochastic}. Consequently, process interventions, such as feeding, medium exchange, pH and agitation control, propagate across scales and influence aggregate growth, population composition, productivity, and product quality.

\vspace{0.05in}
\begin{example}[Modular Bio-SoS Digital Twin]  \label{ex:iPSC_AggregateCulture}
Figure~\ref{fig:iPSC_example2} illustrates a multiscale digital twin for iPSC aggregate culture \citep{zheng2024stochastic}. The model integrates: (i) a population-balance module for aggregate growth and collisions, (ii) a reaction--diffusion module for nutrient and metabolite transport within aggregates, and (iii) a stochastic single-cell metabolic-network module capturing intracellular dynamics and heterogeneity.
Together, these modules link operating conditions, intracellular metabolism, aggregate heterogeneity, and CQAs across scales. Spatial gradients in oxygen and nutrient availability generate metabolic heterogeneity, leading to distinct flux distributions between inner and outer cell populations.

\end{example}

\vspace{-0.1in}
\begin{figure}[htbp]
\centering
\includegraphics[width=1\textwidth]{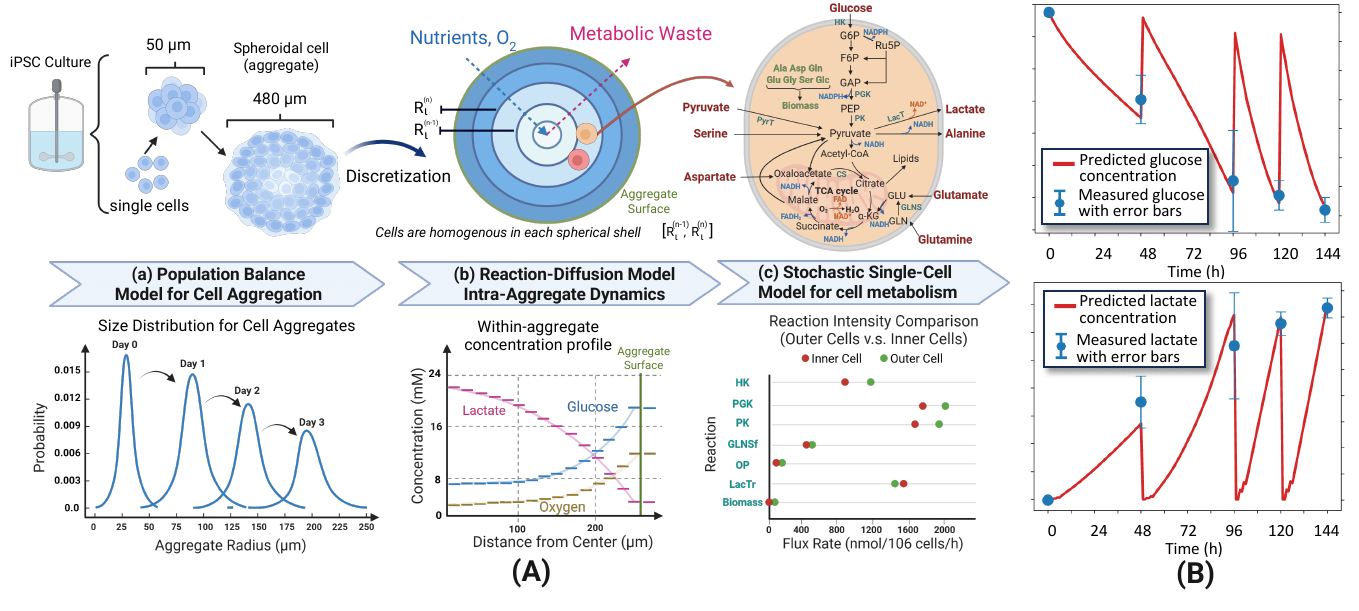}
\vspace{-0.3in}
\caption{\textbf{(A)} An illustration of multi-scale mechanistic
        foundation model with modular design characterizing iPSC cultures \citep{zheng2024stochastic} including the {modules}. 
        \textbf{(B)} Glucose and lactate trends predicted by our Bio-SoS model—trained solely on 2D monolayer data \citep{wang2024metabolic}—closely align with observations from 3D aggregate cultures.}
\label{fig:iPSC_example2}
\end{figure}
\vspace{-0.in}

These interactions can be modeled by modular SDEs, in which reusable modules representing distinct biophysicochemical processes are coupled through shared states to form interoperable mechanistic-hybrid digital twins.
While motivated by cell culture application, the proposed framework generalizes to modular multiscale digital twins for engineered systems such as cell and gene therapy manufacturing, organoid production, 3D bioprinting, additive manufacturing, advanced materials processing, and semiconductor fabrication.

We assume that the digital-twin structure is specified by established scientific knowledge, while its mechanistic parameters remain unknown. 
Let $\pmb s_t\in\mathcal S\subset\mathbb R^M$ denote the process state and $\bm\theta\in\Theta\subset\mathbb R^K$ the vector of mechanistic parameters governing Bio-SoS dynamics and variations, with true value $\bm\theta^\ast$.
Between interventions, the physical system evolves according to the controlled It\^o stochastic differential equation (SDE)
\begin{equation}
d\pmb s_t
=
\pmb\mu(\pmb s_t,\bm\theta)\,dt
+
\pmb\sigma(\pmb s_t,\bm\theta)\,dW_t,
\label{SDE_state}
\end{equation}
where the drift $\pmb\mu$ represents mechanistic reaction, transport, and regulatory dynamics, while the diffusion $\pmb\sigma$ captures intrinsic biological variability (e.g., thermodynamic fluctuations and cell-to-cell heterogeneity) and other unresolved stochastic effects. Appendix~\ref{app:reaction-network} gives a reaction-network derivation motivating this SDE representation.

At each decision epoch $t_i$, a control action $\pmb a_{t_i}$ induces a known post-decision state transition 
$\pmb s_{t_i^+}
=
\pmb f(\pmb s_{t_i},\pmb a_{t_i})$,
representing deterministic operational effects such as dilution from feeding. The stochastic dynamics \eqref{SDE_state} then govern the evolution from $\pmb s_{t_i^+}$ to $\pmb s_{t_{i+1}}$.
For simplicity, observations are collected at equally spaced decision epochs $t_i=t_0+i\Delta t$.

The $j$-th physical experiment generates a trajectory
$\pmb\tau_j
:=
\{(\pmb s_{t_i}^j,\pmb a_{t_i}^j,\pmb s_{t_{i+1}}^j)\}_{i=0}^{T-1},
$
and the cumulative dataset after $n$ experiments is $\mathcal D_n:=\{\pmb\tau_1,\ldots,\pmb\tau_n\}$ with $ 1\le n\le N_{\max},$ where $N_{\max}$ denotes the available experimental budget.
The intervention map $\pmb f$ separates known operational effects from unknown bioprocess dynamics. Consequently, the calibration procedure conditions on the post-decision state $\pmb s_{t_i^+}$, preventing the effects of known interventions from being incorrectly attributed to mechanistic parameters.

The physical system ($p$) and its calibrated digital twin ($d$) are modeled as finite-horizon MDPs that share the same state space $\mathcal S$, action space $\mathcal A$, reward function $r$, and initial-state distribution $p_0$, but differ in their transition dynamics. Let $\Pi$ denote the set of admissible policies. The one-step transition kernel induced by \eqref{SDE_state} over interval $\Delta t=t_{i+1}-t_i$ is
\begin{equation*} P_{\bm\theta}(B \mid \pmb s,\pmb a) := \Pr_{\bm\theta} \!\left( \pmb s_{t_{i+1}}\in B \;\middle|\; \pmb s_{t_i^+} = \pmb f(\pmb s,\pmb a) \right), \qquad B\in\mathcal B(\mathcal S), \end{equation*}
where $\mathcal B(\mathcal S)$ denotes the Borel $\sigma$-algebra on $\mathcal S$. 
The physical system and digital twin are therefore represented as 
$ \mathcal M^p = (\mathcal S,\mathcal A,r,p_0,P_{\bm\theta^\ast})$ and $ \mathcal M_n^d = (\mathcal S,\mathcal A,r,p_0,P_{\widehat{\bm\theta}_n}),$
where $\widehat{\bm\theta}_n$ denotes the parameter estimate obtained from historical data $\mathcal{D}_n$.

For a policy $\pi\in\Pi$, define the parameterized finite-horizon discounted return as
\begin{equation} J(\pi;\bm\theta) := \mathbb E^\pi_{\bm\theta} \left[ \sum_{i=0}^{T-1} \gamma^i r(\pmb s_{t_i},\pmb a_{t_i}) \right],
\label{eq:parameterized-policy-value}
\end{equation} 
where $\gamma\in(0,1]$ is the discount factor. The dependence on $\bm\theta$ reflects the fact that multiscale mechanistic foundation parameters determine the stochastic transition dynamics and therefore the resulting policy performance.
For a state $\pmb s$ and an action $\pmb a$ selected at decision stage $i$, define the action-value Q-function
\begin{equation}
Q_i^\pi(\pmb s,\pmb a;\bm\theta)
:=
\mathbb E_{\bm\theta}^{\pi}\!\left[
\sum_{k=i}^{T-1}\gamma^{k-i}
r(\pmb s_{t_k},\pmb a_{t_k})
\,\middle|\,
\pmb s_{t_i}=\pmb s,\,
\pmb a_{t_i}=\pmb a
\right],
\label{eq:conditional-Q}
\end{equation}
which represents the expected future return when action $\pmb a$ is applied at stage $i$ and policy $\pi$ governs all subsequent decisions. 

The optimal policy for the physical system is
\begin{equation*} 
\pi^\ast \in \arg\max_{\pi\in\Pi} J(\pi;\bm\theta^\ast). 
\end{equation*}
In biomanufacturing, where \emph{the process is the product}, uncertainty in mechanistic parameters and state perturbations propagates through process trajectories to policy performance. Because the true parameter $\bm\theta^\ast$ must be inferred from sparse trajectory data, parameter estimates inevitably exhibit both bias and uncertainty, which can distort policy evaluation and optimization. This raises three fundamental questions: (i) how parameter bias and uncertainty propagate through mechanistic trajectories to decision performance; (ii) how these effects can be quantified and corrected during policy optimization; and (iii) how experiments should be designed to most efficiently reduce policy-relevant uncertainty and accelerate the discovery of end-to-end optimal control policies. Addressing these questions is the central objective of this work.

\subsection{Proposed Actor--Simulator Framework}
\label{sec:framework}
To address the challenges outlined above, we develop a bias-aware Actor--Simulator framework that unifies uncertainty quantification, policy-relevant sequential design of experiments, digital-twin calibration, and process optimal control within a closed-loop learning architecture. At each iteration, the framework estimates mechanistic parameters together with their calibration bias and uncertainty, designs physical-system experiments that maximally reduce uncertainty in policy performance, and updates the control policy using a bias-corrected estimate of physical-system value.

A central component of the framework is a shared trajectory-level adjoint sensitivity analysis that supports both digital-twin calibration and policy optimization. In biomanufacturing, where \emph{the process is the product}, uncertainty in mechanistic parameters propagates through the entire evolution of molecular, cellular, and process states, ultimately affecting control performance. By quantifying how calibration bias and uncertainty influence process trajectories and policy value, the proposed framework establishes a closed-loop connection among uncertainty quantification, experiment design, calibration, and decision-making.

Figure~\ref{fig:flowchart} summarizes the proposed framework. Operating at the trajectory level, it consists of four interconnected learning modules:
\begin{enumerate}
\item \textbf{Trajectory-Level Adjoint Sensitivity Analysis.}
Action-induced state changes and parameter perturbations propagate through process trajectories, shaping rewards, product quality, and ultimately policy performance.
Treating $(\pmb s_t,\pmb\theta)$ as an augmented state, we develop a unified forward--backward adjoint framework that computes sensitivities of both conditional action values and cumulative returns.
These sensitivities quantify how parameter bias and uncertainty propagate to decision performance.

\item \textbf{Policy-Relevant Digital Twin Calibration.} 
Given the current target policy $\widehat{\pi}_{n-1}$, we design an exploration policy that prioritizes state--action regions where model uncertainty has the greatest impact on policy performance. Specifically, define the physical--digital value discrepancy
\begin{equation} 
\Delta Q_{n-1,i}^{\widehat{\pi}_{n-1}}(\pmb s,\pmb a) := Q_i^{\widehat{\pi}_{n-1}}(\pmb s,\pmb a;\bm\theta^\ast) - Q_i^{\widehat{\pi}_{n-1}}(\pmb s,\pmb a;\widehat{\bm\theta}_{n-1}), 
\label{eq.physical--digital value discrepancy}
\end{equation}
and select actions according to \[ \pmb{a}_i=\eta_n(\pmb{s}_i), \qquad \eta_n(\pmb{s}_i) \in \arg\max_{\pmb a} \Delta Q_{n-1,i}^{\widehat{\pi}_{n-1}}(\pmb s_i,\pmb a), \quad i=0,1,\ldots,T-1. \] 
The resulting trajectory $\pmb\tau_n$ is incorporated into the dataset through
$\mathcal D_n=\mathcal D_{n-1}\cup\{\pmb\tau_n\}$. This step actively acquires data that are most informative for reducing uncertainty.

\item \textbf{Bias-Aware Model Estimation.}
Using the accumulated trajectory data $\mathcal D_n$, we estimate the mechanistic parameters and characterize calibration uncertainty through
$
\pmb\theta^*-\widehat{\pmb\theta}_n
\ \dot\sim\ 
N\!\left(\widehat{\pmb b}_n,\widehat\Sigma_n\right),
$ where $\widehat{\pmb b}_n$ and $\widehat\Sigma_n$ are the estimated bias and covariance. They capture digital-twin uncertainty due to sparse observations and approximation error.

\item \textbf{Bias-Corrected Policy Optimization.} Using full-return adjoint sensitivities and $(\widehat{\pmb b}_n,\widehat\Sigma_n)$, we propagate calibration bias and parameter uncertainty into policy evaluation. Conditional on $\mathcal D_n$, let $F_n$ denote the distribution of 
\[
\widetilde{\bm\theta}_n:=\widehat{\bm\theta}_n+\widetilde{\bm e}_n, \qquad \widetilde{\bm e}_n\sim N(\widehat{\bm b}_n,\widehat\Sigma_n).
\]
The resulting bias-aware objective is
\begin{equation}
J_n^{\mathrm G}(\pi)
:=\mathbb E_{\widetilde{\bm\theta}_n\sim F_n}
\bigl[J(\pi;\widetilde{\bm\theta}_n)\bigr].
\label{eq.DT-policy optimization}
\end{equation}
Section~\ref{sec:policyOptimizationstep} derives an implementable second-order approximation $\bar J_n(\pi)$ and updates $\widehat\pi_n\in\arg\max_{\pi\in\Pi}\bar J_n(\pi)$.
\end{enumerate}

\begin{figure}[htbp]
\centering
\vspace{-0.1in}
\includegraphics[width=0.86\textwidth]{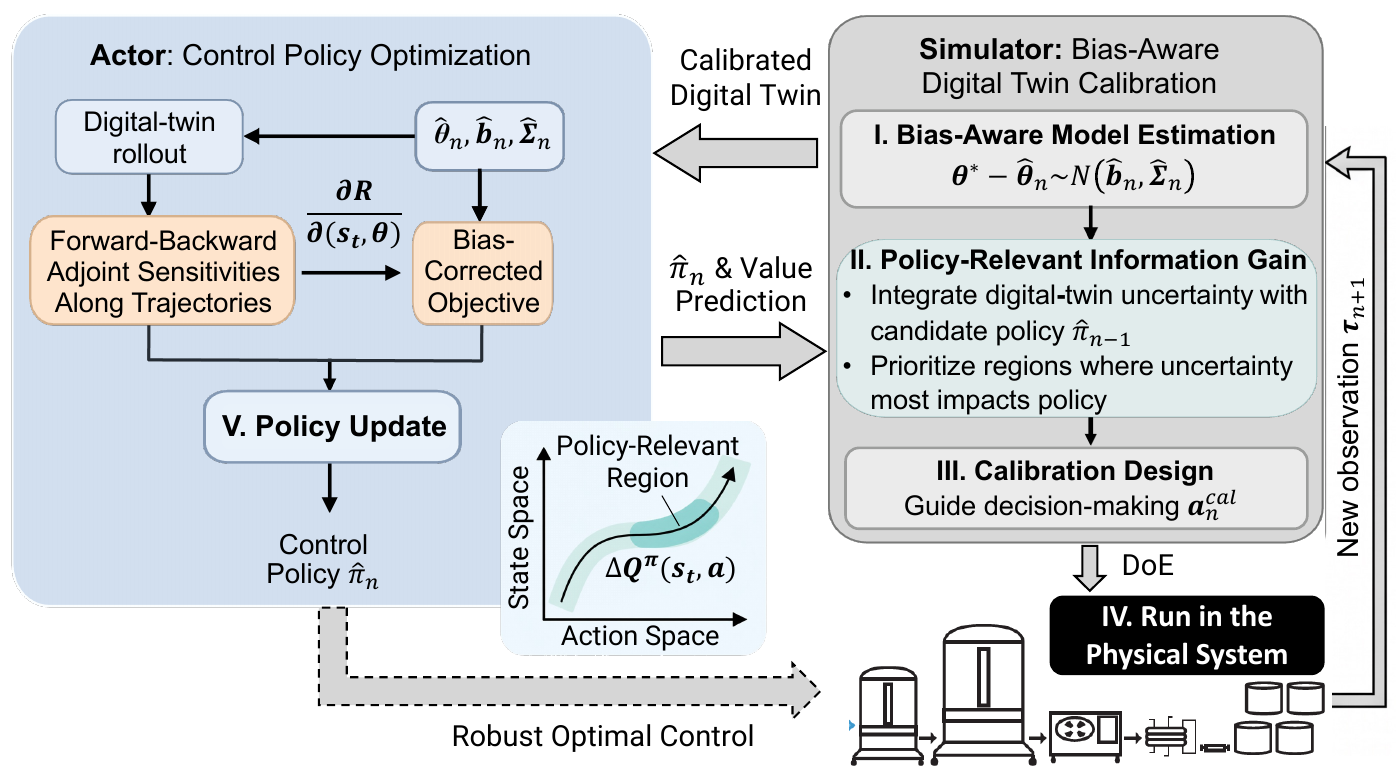}
\vspace{-0.in}
\caption{Schematic flowchart illustration of the proposed actor-simulator framework.}
\vspace{-0.in}
\label{fig:flowchart}
\end{figure}

The physical system provides expensive but informative trajectories, while the digital twin supplies inexpensive simulation rollouts, sensitivity analysis, and uncertainty propagation. At each $n$-th iteration, the resulting Actor--Simulator cycle is
\[
\underbrace{\widehat\pi_{n-1}\to\eta_n}_{\text{Adjoint-Guided Exploration}}
\overbrace{\to\pmb\tau_n\to\mathcal D_n}^{\mathclap{\text{Data Acquisition}}}
\underbrace{\to
(\widehat{\bm\theta}_n,\widehat{\bm b}_n,\widehat\Sigma_n)}_{\text{Bias-Aware Estimation}}
\overbrace{\to\widehat\pi_n}^{\mathclap{\text{Control Policy Optimization}}}.
\]
The exploration policy $\eta_n$ and target policy $\widehat{\pi}_n$ play complementary roles. The exploration policy targets state--action regions with large policy-relevant value discrepancy, whereas the target policy maximizes a bias-corrected Gaussian-averaged value approximation. Algorithm~\ref{actorsimulatoralgo} in Appendix~\ref{sec:algorithm} summarizes the complete iterative procedure.


\section{Calibration of SDE-Based Digital Twins}
\label{sec: Calibration Framework}

This section presents the proposed Actor--Simulator framework, which jointly calibrates the digital twin through policy-directed experimentation and optimizes process control policies.
It comprises four tightly integrated components.
First, we estimate the SDE parameters while quantifying both moment-truncation bias and finite-sample estimation uncertainty (Section~\ref{subsec:MU&SA}). Second, we develop a trajectory-level forward--backward adjoint framework to efficiently compute sensitivities of conditional action values and policy returns with respect to model parameters and state perturbations (Section~\ref{sec:adjoint-engine}). 
Third, these sensitivities are used to guide policy-relevant digital twin calibration by adaptively selecting physical experiments in state--action regions that are most informative for improving control performance under sparse observations
(Section~\ref{sec:exploration-objective}). Finally, we propagate the estimated calibration bias and parameter uncertainty through mechanistic pathways to obtain a bias-corrected policy optimization procedure (Section~\ref{sec:policyOptimizationstep}).

Together, these components establish a closed-loop framework linking uncertainty-aware calibration, adaptive experimental design, and policy optimization.
For simplicity, we focus on a scalar performance metric, although the framework extends naturally to multivariate responses.
Unlike black-box response-surface approaches, such as Gaussian-process (GP) metamodeling, which often struggle to capture complex multivariate dependencies \citep{higdon2008computer}, this framework preserves the mechanistic coupling among responses and propagates parameter uncertainty through the underlying system dynamics.

\subsection{Bias-Aware Model Parameter Estimation}
\label{subsec:MU&SA}
Given the physical trajectories $\pmb\tau_j:=\{(\pmb s_{t_i}^j,\pmb a_{t_i}^j,\pmb s_{t_{i+1}}^j)\}_{i=0}^{T-1}$, $j=1,\ldots,n$, the intervention map determines the post-decision state $\pmb s_{t_i^+}^j=\pmb f(\pmb s_{t_i}^j,\pmb a_{t_i}^j)$. Parameter estimation would ideally rely on the finite-time transition law of \eqref{SDE_state} from $\pmb s_{t_i^+}^j$ to $\pmb s_{t_{i+1}}^j$. Because this distribution is generally unavailable for nonlinear SDEs, we approximate its first two conditional moments using a finite-order generator expansion and estimate unknown $\bm\theta^*$ via Gaussian quasi-maximum likelihood (QMLE). We then quantify the systematic bias introduced by moment truncation and combine it with the sandwich covariance of the QMLE to construct a bias-aware Gaussian approximation of the parameter error $\bm\theta^\ast-\widehat{\bm\theta}_n$.

The following assumptions justify the conditional-moment expansion, score-bias characterization, and QMLE asymptotics. Assumptions (i)--(ii) ensure the validity of the moment approximation and truncation-bias analysis, while (iii)--(iv) establish consistency and the sandwich covariance limit. For a specified kinetic model, smoothness and growth bounds can be checked analytically, while covariance and Hessian eigenvalues and trajectory-score moments can be diagnosed over the calibrated regime. The stronger dense-sampling results in Section~\ref{sec:convergence} additionally require a nonsingular limiting covariance-information matrix.
\begin{assumption}[Moment-expansion and QMLE regularity]
\label{Regularity Cond}
For a fixed conditional moment expansion order $l\geq1$:
\begin{enumerate}[label=(\roman*)]
    \item $\Theta$ is compact with $\bm\theta^\ast$ in its interior. The drift and diffusion are globally Lipschitz in state uniformly over $\Theta$, three times continuously differentiable in $\bm\theta$, and sufficiently smooth in state. The SDE has a unique nonexplosive strong solution with moments dominating the generator expansions and their parameter derivatives.
    \item Coordinate and quadratic test functions admit generator iterates through order $l+2$ with polynomial-growth bounds uniform over $\Theta$. The order-$l$ truncated conditional covariance is uniformly positive definite on the relevant state--parameter domain.
    \item Under any data-collection policy $\eta$, trajectories are i.i.d. and the score has finite second moment. The population quasi-likelihood has a unique interior maximizer with nonsingular expected Hessian.

\end{enumerate}
\end{assumption}

\textbf{(1) Generator-Based Moment Approximation.}
For a transition from post-decision state $\pmb s$ over interval $\Delta t$, define the exact first and second conditional moments
\begin{equation}
\begin{aligned}
\pmb M_1(\Delta t,\pmb s,\bm\theta)
:=\mathbb E_{\bm\theta}[\pmb s_{t_{i+1}}\mid\pmb s_{t_i^+}=\pmb s], ~~~ 
\pmb M_2(\Delta t,\pmb s,\bm\theta)
:=\operatorname{Var}_{\bm\theta}(\pmb s_{t_{i+1}}\mid\pmb s_{t_i^+}=\pmb s).
\end{aligned}
\label{eq:exact-conditional-moments}
\end{equation}
Because these finite-time moments are generally unavailable for nonlinear SDEs, we approximate them using the infinitesimal generator of \eqref{SDE_state},
\[
\mathcal L_{\bm\theta}g(\pmb s)
=
\pmb\mu(\pmb s,\bm\theta)^\top\nabla_sg(\pmb s)
+\frac12\operatorname{tr}\!\left[
\pmb\sigma(\pmb s,\bm\theta)\pmb\sigma(\pmb s,\bm\theta)^\top
\nabla_s^2g(\pmb s)
\right].
\]
Under Assumption~\ref{Regularity Cond}, the Markov semigroup expansion \citep{kessler1997estimation}  yields
\begin{equation}
\mathbb E_{\bm\theta}[g(\pmb s_{t+\Delta t})\mid\pmb s_t=\pmb s]
=
\sum_{r=0}^{l}\frac{(\Delta t)^r}{r!}\mathcal L_{\bm\theta}^rg(\pmb s)
+\mathcal O((\Delta t)^{l+1}),
\label{eq:semigroup-expansion}
\end{equation}
for any admissible test function $g$.
Applying Equation~\eqref{eq:semigroup-expansion} to the coordinate and quadratic functions gives the explicit order-$l$ conditional mean and covariance
\begin{equation}
\begin{aligned}
\widetilde{\pmb m}_{1,l}(\Delta t,\pmb s,\bm\theta)
&:=
\sum_{r=0}^{l}\frac{(\Delta t)^r}{r!}\mathcal L_{\bm\theta}^r\pmb s, \\
\widetilde{\pmb m}_{2,l}(\Delta t,\pmb s,\bm\theta)
&:=
\sum_{r=0}^{l}\frac{(\Delta t)^r}{r!}
\mathcal L_{\bm\theta}^r(\pmb s\pmb s^\top)
-
\widetilde{\pmb m}_{1,l}\widetilde{\pmb m}_{1,l}^\top.
\end{aligned}
\label{eq:moment-expansion-summary}
\end{equation}

The exact and truncated moments satisfy
\begin{equation*}
\pmb M_q(\Delta t,\pmb s,\bm\theta)
=
\widetilde{\pmb m}_{q,l}(\Delta t,\pmb s,\bm\theta)
+R_{q,l}(\Delta t,\pmb s,\bm\theta),
\qquad q=1,2,
\end{equation*}
where the remainder terms satisfy the uniform bound
$
\|R_{1,l}\|+\|R_{2,l}\|_F \le C_1(\Delta t)^{l+1} \bigl(1+\|\pmb s\|\bigr)^{C_2}$
for finite constants $C_1,C_2>0$ independent of $\Delta t$, $\pmb s$, and $\bm\theta$. The leading omitted generator terms provide $\mathcal O((\Delta t)^{l+1})$ approximations of $R_{1,l}$ and $R_{2,l}$ and are later used to estimate truncation bias. Importantly, a finite observation interval does not itself induce bias when the exact finite-time moments are used. The systematic bias studied here arises solely from replacing $\pmb M_1$ and $\pmb M_2$ by the truncated approximations in \eqref{eq:moment-expansion-summary}, and is therefore distinct from numerical integration error associated with the simulation step size $h_{\mathrm{sim}}$.

\vspace{0.05in}
\textbf{(2) Quasi-Likelihood and Truncation Bias.}
Replacing the exact moments in \eqref{eq:exact-conditional-moments} by \eqref{eq:moment-expansion-summary} yields a computable Gaussian quasi-likelihood. Let
$
m_{q,i}^j(\bm\theta)
:=
\widetilde{\pmb m}_{q,l}(\Delta t,\pmb s_{t_i^+}^j,\bm\theta)$ and $
\varepsilon_i^j(\bm\theta)
:=
\pmb s_{t_{i+1}}^j-m_{1,i}^j(\bm\theta)$.
The corresponding order-$l$ quasi-log-likelihood contribution is
\begin{equation}
\ell_{l,i}^j(\bm\theta)
:=
-\frac12\left[
\varepsilon_i^j(\bm\theta)^\top
m_{2,i}^j(\bm\theta)^{-1}\varepsilon_i^j(\bm\theta)
+\log\det m_{2,i}^j(\bm\theta)
\right],
\label{eq:transition-quasilikelihood}
\end{equation}
and the QMLE is
\[\widehat{\bm\theta}_n
\in
\arg\max_{\bm\theta\in\Theta}
\frac1n\sum_{j=1}^n\ell_l(\pmb\tau_j;\bm\theta) 
~~~ \mbox{with} ~~~
\ell_l(\pmb\tau_j;\bm\theta)
:=
\sum_{i=0}^{T-1}\ell_{l,i}^j(\bm\theta).\]

The Gaussian criterion does not require Gaussian transition dynamics. If the exact conditional moments $(\pmb M_1,\pmb M_2)$ were available, its population score would be centered at the physical parameter $\bm\theta^\ast$; higher-order moments affect efficiency but not score centering. Replacing the exact moments with their order-$l$ approximations breaks this property and shifts the population maximizer away from $\bm\theta^\ast$.

To quantify this truncation effect, define the population criterion, score, and {pseudo-true parameter} $\widetilde{\bm\theta}_l$ under the data-collection policy $\eta$:
\begin{equation*}
L_l^{\mathrm{pop}}(\bm\theta)
:=
\mathbb E_{\bm\theta^\ast,\eta}[\ell_l(\pmb\tau;\bm\theta)],
~~~
G_l(\bm\theta):=\nabla_{\bm\theta}L_l^{\mathrm{pop}}(\bm\theta),
~~~
\widetilde{\bm\theta}_l
:=
\arg\max_{\bm\theta\in\Theta}L_l^{\mathrm{pop}}(\bm\theta).
\end{equation*}
By construction, $G_l(\widetilde{\bm\theta}_l)=\bm0$. Under Assumption~\ref{Regularity Cond}, differentiation and expectation can be interchanged. The nonzero population score at $\bm\theta^\ast$ is
\begin{equation}
\bm\delta_l:=G_l(\bm\theta^\ast)
=\sum_{i=0}^{T-1}\bm\delta_{l,i},
\qquad
\bm\delta_{l,i}
:=
\mathbb E_{\bm\theta^\ast,\eta}
[\nabla_{\bm\theta}\ell_{l,i}(\pmb\tau;\bm\theta^\ast)],
\label{eq:population-score-bias}
\end{equation}
where the expectation is taken with respect to the trajectory distribution induced by $\bm\theta^\ast$ under policy $\eta$.
Applying the mean-value theorem to $G_l(\widetilde{\bm\theta}_l)-G_l(\bm\theta^\ast)$ gives
\[
-\bm\delta_l
=G_l(\widetilde{\bm\theta}_l)-G_l(\bm\theta^\ast)
=\overline\Lambda_l(\widetilde{\bm\theta}_l-\bm\theta^\ast),
\]
where the average score-curvature matrix along the segment joining $\bm\theta^\ast$ and $\widetilde{\bm\theta}_l$ is
\begin{equation}
\overline\Lambda_l
:=
\int_0^1
\nabla_{\bm\theta}G_l\!\left(
\bm\theta^\ast
+u(\widetilde{\bm\theta}_l-\bm\theta^\ast)
\right)\,du.
\label{eq:integrated-sensitivity}
\end{equation}
Under correct specification, $\overline\Lambda_l$ reduces locally to the Fisher-information curvature. Under moment truncation, however, the information identity generally fails, so $\overline\Lambda_l$ is interpreted as an average score-curvature matrix rather than Fisher information.

Thus, the moment remainders contribute to $\bm\delta_l$, and the curvature of the population criterion converts this nonzero score into parameter bias. Appendix~\ref{app:score-bias} derives the contributions of $R_{1,l}$ and $R_{2,l}$ to $\bm\delta_l$ and proves the following proposition.
\begin{proposition}[Score-induced parameter bias]
\label{asymptotic bias}
If $\overline\Lambda_l$ is nonsingular, then
\begin{equation}
\bm b_{\Delta,l}
:=
\bm\theta^\ast-\widetilde{\bm\theta}_l
=
\overline\Lambda_l^{-1}\bm\delta_l.
\label{eq:exact-parameter-bias}
\end{equation}
\end{proposition}

\textbf{(3) Bias-Aware Parameter Uncertainty Quantification.}
Proposition~\ref{asymptotic bias} characterizes the deterministic displacement of the pseudo-true parameter $\widetilde{\bm\theta}_l$ from the physical parameter $\bm\theta^\ast$. We now quantify the sampling variability of $\widehat{\bm\theta}_n$ around this pseudo-true target.
Define the trajectory score
$\bm\psi_l(\pmb\tau_j;\bm\theta)
:=
\nabla_{\bm\theta}\ell_l(\pmb\tau_j;\bm\theta)$
and the empirical variability and curvature matrices
$
A_{l,n}(\bm\theta)
:=\frac1n\sum_{j=1}^n
\bm\psi_l(\pmb\tau_j;\bm\theta)
\bm\psi_l(\pmb\tau_j;\bm\theta)^\top$ and 
$B_{l,n}(\bm\theta)
:=\frac1n\sum_{j=1}^n
\nabla_{\bm\theta\bm\theta}^2
\ell_l(\pmb\tau_j;\bm\theta)$
with population counterparts $A_l(\bm\theta)$ and $B_l(\bm\theta)$. The resulting sandwich covariance is 
\[ C_l(\bm\theta) := B_l(\bm\theta)^{-1} A_l(\bm\theta) B_l(\bm\theta)^{-\top}. \]

Both the truncation bias $\bm b_{\Delta,l}$ and the sampling covariance $C_l(\widetilde{\bm\theta}_l)$ are unknown in practice. To estimate the score displacement $\bm\delta_l$, we approximate the leading omitted moment terms by the difference between the $(l+1)$th- and $l$th-order generator expansions: $\widehat R_{q,l}^{\mathrm{next}} := \widetilde{\pmb m}_{q,l+1} - \widetilde{\pmb m}_{q,l}$ for $q=1,2.$
Substituting these approximations into the first-order score expansion yields an empirical score-bias estimator $\widehat{\bm\delta}_n$ (Appendix~\ref{app:score-bias}), from which we obtain \begin{equation} \widehat\Lambda_{l,n} := B_{l,n}(\widehat{\bm\theta}_n), \qquad \widehat{\bm b}_n := \widehat\Lambda_{l,n}^{-1}\widehat{\bm\delta}_n. \label{eq:empirical-parameter-bias} \end{equation}
Thus, $\widehat{\bm b}_n$ provides a one-step estimate of the truncation-induced parameter \textit{bias} by combining the estimated score displacement with the local curvature of the quasi-likelihood.
The corresponding \textit{covariance estimator} is \begin{equation} \widehat C_{l,n} = B_{l,n}(\widehat{\bm\theta}_n)^{-1} A_{l,n}(\widehat{\bm\theta}_n) B_{l,n}(\widehat{\bm\theta}_n)^{-\top}, \qquad \widehat\Sigma_n := \widehat C_{l,n}/n. \label{eq:uncertainty-covariance} \end{equation}

The following result characterizes the asymptotic model estimation error under fixed moment truncation.

\begin{proposition}[Calibration-error limit under fixed misspecification; \citealp{white1982maximum}]
\label{asymptotic normality}
Under Assumption~\ref{Regularity Cond},
\[
\widehat{\bm\theta}_n\xrightarrow{p}\widetilde{\bm\theta}_l,
\qquad
\sqrt n\bigl[(\bm\theta^\ast-\widehat{\bm\theta}_n)-\bm b_{\Delta,l}\bigr]
\xrightarrow{d}
N\!\left(\bm0,C_l(\widetilde{\bm\theta}_l)\right).
\]
\end{proposition}

Proposition~\ref{asymptotic normality} decomposes the calibration error $\bm\theta^\ast-\widehat{\bm\theta}_n$ into a persistent truncation bias component $\bm b_{\Delta,l}$ and a sampling-uncertainty component with covariance $C_l(\widetilde{\bm\theta}_l)/n$.
\[ \bm\theta^\ast-\widehat{\bm\theta}_n \approx N\!\left( \bm b_{\Delta,l}, \, C_l(\widetilde{\bm\theta}_l)/n \right). \]
For downstream decision-making, we replace these unknown quantities by their estimators and introduce the Gaussian surrogate
\begin{equation}
\widetilde{\bm e}_n\mid\mathcal D_n
\sim N\!\left(\widehat{\bm b}_n,\widehat\Sigma_n\right).
\label{eq:gaussian-error-surrogate}
\end{equation}
This auxiliary draw is used to propagate parameter estimation uncertainty through the calibration and policy-optimization objectives and is not intended to represent the exact conditional distribution of $\bm\theta^\ast-\widehat{\bm\theta}_n$.
Unlike $\widehat\Sigma_n$, which shrinks as more trajectories are collected, $\widehat{\bm b}_n$ generally persists when the observation interval and moment-expansion order remain fixed.
Under the dense-sampling regime of Section~\ref{sec:convergence}, $\widehat C_{l,n}=\mathcal O_p(T^{-1})$ and hence $\widehat\Sigma_n=\mathcal O_p((nT)^{-1})$. 
Preserving this bias--variance decomposition is essential for exploration and policy optimization: the bias component governs the direction of value distortion, whereas the covariance quantifies residual uncertainty, guides policy-directed experimentation, and contributes a second-order correction to Gaussian-averaged objectives.


\subsection{Shared Forward--Backward Adjoint Sensitivities}
\label{sec:adjoint-engine}

The proposed framework reuses parameter sensitivities in two downstream decision problems. Section~\ref{sec:exploration-objective} guides policy-directed experimentation by targeting state--action regions where parameter uncertainty most affects the physical--digital value discrepancy 
$\Delta Q_{n-1,i}^{\widehat{\pi}_{n-1}}(\pmb s,\pmb a)$.
Section~\ref{sec:policyOptimizationstep} optimizes the bias-corrected finite-horizon return in \eqref{eq.DT-policy optimization}. Both tasks rely on second-order uncertainty propagation and require sensitivities of conditional action values $Q_i^\pi$ and policy returns $J$ with respect to the calibration parameters $\bm\theta$.

Computing these derivatives is challenging in Bio-SoS models because system performance depends on many interacting parameters that jointly govern bioprocess dynamics, spatiotemporal heterogeneity, and stochasticity.
Finite-difference methods require repeated simulations for parameter perturbations, while forward sensitivity methods propagate an $M\times K$ state-parameter Jacobian throughout each trajectory \citep{li2020scalable}. In contrast, stochastic adjoint sensitivity analysis propagates value gradients backward along the same realized Brownian path using vector--Jacobian products, avoiding explicit Jacobian evolution. The resulting computational advantage becomes increasingly significant as the parameter dimension grows. Moreover, trajectory-level adjoint analysis embodies the principle that \textit{the process is the product}, enabling forward prediction, backward sensitivity propagation, and interoperable learning.

\vspace{0.05in}

\textbf{(1) Trajectory-Level Adjoint Sensitivities over $(\pmb{s},\pmb{\theta})$.}
For a fixed policy $\pi$ and Brownian sample path $\xi$, define the pathwise return
\[
\mathcal R^\pi(\bm\theta,\xi)
:=
\sum_{i=0}^{T-1}\gamma^i
r(\pmb s_{t_i}^{\bm\theta,\xi},
\pmb a_{t_i}^{\bm\theta,\xi}),
\qquad
J(\pi;\bm\theta)
=
\mathbb E[\mathcal R^\pi(\bm\theta,\xi)].
\]
A forward pass simulates the controlled SDE, and a backward pass propagates first- and second-order sensitivities along the same realization of the Brownian motion. To compute state and parameter derivatives simultaneously, we augment the state as 
$
z:=(\pmb s,\bm\theta)$ and $
\mathcal F_{u,t}(z,\xi)
:=(\Phi_{u,t}(\pmb s;\bm\theta,\xi),\bm\theta)$, 
where $\Phi_{u,t}$ denotes the stochastic forward flow on an intervention-free interval. Appendix~\ref{app:adjoint-flows} develops the forward and inverse flows, and Appendix~\ref{app:adjoint-state-parameter} gives their state and parameter sensitivities. Action sensitivities are subsequently recovered through the intervention map 
\[\frac{\partial\pmb s_{t_{i+1}}}{\partial\pmb a_{t_i}} = \frac{\partial\pmb s_{t_{i+1}}}{\partial\pmb s_{t_i^+}} \frac{\partial\pmb s_{t_i^+}}{\partial\pmb a_{t_i}}.
\]
 
For a smooth terminal functional $\rho(\pmb s)$ (e.g., productivity, CQA out-of-specification probability, or terminal reward), define the pathwise gradient and Hessian fields
\begin{equation}
p_{u,v}(z,\xi)
:=\nabla_z\bigl[\rho(\Phi_{u,v}(\pmb s;\bm\theta,\xi))\bigr],
\qquad
\mathscr H_{u,v}(z,\xi)
:=D_z^2\bigl[\rho(\Phi_{u,v}(\pmb s;\bm\theta,\xi))\bigr],
\label{eq:augmented-pathwise-sensitivities}
\end{equation}
where $D_z$ denotes differentiation with respect to the augmented variable $z=(\pmb s,\bm\theta)$. The gradient field $p_{u,v}$ describes how perturbations in biological states and mechanistic parameters propagate through the stochastic bioprocess to influence terminal performance, while $\mathscr H_{u,v}$ captures the corresponding second-order interactions.
The inverse-flow fields $\widetilde p_{t,v}$ and $\widetilde{\mathscr H}_{t,v}$ are the backward counterparts of the corresponding forward fields. Let $ \widetilde A^s_{t,v}:=[\widetilde p_{t,v}]_s^\top$ and $ \widetilde A^\theta_{t,v}:=[\widetilde p_{t,v}]_\theta^\top$
denote the state and parameter adjoint processes. These are the state and parameter components of the pathwise value gradient propagated backward from $\rho$ along a realized stochastic trajectory. The following proposition states the pathwise identity precisely; its proof is in Appendix~\ref{app:adjoint-pathwise}.

\begin{proposition}[Pathwise adjoint identities]
\label{prop:shared-adjoint-sensitivities}
Under the conditions in Appendix~\ref{app:adjoint-regularity},
for almost every Brownian path, let $z_u=(\pmb s,\bm\theta)$ and
$z_v=\mathcal F_{u,v}(z_u,\xi)$. Then
\begin{equation}
\begin{aligned}
\widetilde p_{u,v}(z_v,\xi)
&=\left.\nabla_z[\rho(\Phi_{u,v}(\pmb s;\bm\theta,\xi))]\right|_{z=z_u},\\
\widetilde{\mathscr H}_{u,v}(z_v,\xi)
&=\left.D_z^2[\rho(\Phi_{u,v}(\pmb s;\bm\theta,\xi))]\right|_{z=z_u}.
\end{aligned}
\label{eq:shared-adjoint-composition}
\end{equation}
\end{proposition}

The next proposition summarizes the computational benefit of adjoints relative to forward sensitivity propagation. Proposition~\ref{prop:adjoint-size} shows that adjoint propagation avoids maintaining the full state--parameter derivative tensor, yielding substantial computational savings when the numbers of states and calibration parameters are large.
Its proof is in Appendix~\ref{app:adjoint-size-proof}.
\begin{proposition}[Cost of pathwise sensitivity propagation]
\label{prop:adjoint-size}
For one numerical step in SDE simulation, the derivative-propagation costs are
\[
\begin{array}{c|cc}
&\text{forward}&\text{adjoint}\\
\text{first order}&\mathcal O((R+1)M^2K)&\mathcal O((R+1)(M^2+MK))\\
\text{full second order}&\mathcal O((R+1)MK(M+K)^2)&\mathcal O((R+1)M(M+K)^2).
\end{array}
\]
\end{proposition}

\vspace{0.05in}
\textbf{(2) Estimation of $J$- and $Q$-Function Sensitivities.}
Pathwise sensitivities of the finite-horizon return are obtained by composing the interval adjoints in Proposition~\ref{prop:shared-adjoint-sensitivities} through rewards and decision maps; Appendix~\ref{app:adjoint-finite-horizon} gives the chain rules. For a fixed policy $\pi$, let $\mathcal G^\pi:=\nabla_\theta\mathcal R^\pi$ and $\mathcal H^\pi:=\nabla_\theta^2\mathcal R^\pi$ denote the resulting pathwise gradient and Hessian of the total return. Under the conditions in Appendix~\ref{app:adjoint-regularity}, differentiation and expectation can be interchanged, yielding
$\nabla_\theta J(\pi;\bm\theta)=\E[\mathcal G^\pi]$ and $ \nabla_\theta^2J(\pi;\bm\theta)=\E[\mathcal H^\pi].$
Taking the average of $\mathcal G^\pi$ and $\mathcal H^\pi$ over rollouts then becomes the estimate of the corresponding sensitivities of $J$.

For calibration experiment design, let $\mathcal R_i^\pi(\pmb s,\pmb a;\bm\theta,\xi)$ denote the return conditioned on state-action pair $(\pmb s,\pmb a)$ at decision epoch $t_i$, with action-value function
\[ Q_i^\pi(\pmb s,\pmb a;\bm\theta) = \E[\mathcal R_i^\pi(\pmb s,\pmb a;\bm\theta,\xi)]. \] 
Applying the same forward-backward adjoint recursion yields
\begin{equation} 
\bm g_{n,i}^\pi(\pmb s,\pmb a) = \nabla_\theta Q_i^\pi(\pmb s,\pmb a;\widehat{\bm\theta}_n), \qquad H_{n,i}^\pi(\pmb s,\pmb a) = \nabla_\theta^2 Q_i^\pi(\pmb s,\pmb a;\widehat{\bm\theta}_n), 
\label{eq.gradients of Q}
\end{equation}
together with their Monte Carlo estimators. Thus, a single adjoint framework provides the first- and second-order sensitivities required for policy optimization and experiment design.

\subsection{Digital Twin Calibration}
\label{sec:exploration-objective}

Let $\bm e_n:=\bm\theta^\ast-\widehat{\bm\theta}_n$ denote the calibration error. Using (\ref{eq.gradients of Q}), a second-order expansion of the action-value function around $\widehat{\bm\theta}_n$ yields
\begin{equation*}
Q_i^\pi(\pmb s,\pmb a;\bm\theta^\ast)
-
Q_i^\pi(\pmb s,\pmb a;\widehat{\bm\theta}_n)
=
(\bm g_{n,i}^\pi)^\top\bm e_n
+
\frac12\bm e_n^\top H_{n,i}^\pi\bm e_n
+
\mathcal E_{n,i}^\pi,
\end{equation*}
where $\mathcal E_{n,i}^\pi=O(\|\bm e_n\|^3)$.
This motivates the value-discrepancy criterion
\begin{equation}
u_{n,i}^\pi(\pmb s,\pmb a)
:=
\left|
(\bm g_{n,i}^\pi)^\top\bm e_n
+
\frac12\bm e_n^\top H_{n,i}^\pi\bm e_n
\right|,
\label{eq:value-discrepancy-criterion}
\end{equation}
which measures the impact of calibration error on the action value under policy $\pi$.

Because the actual calibration error $\bm e_n$ is unknown, we use the auxiliary draw $\widetilde{\bm e}_n$ in Equation~\eqref{eq:gaussian-error-surrogate} to approximate the value discrepancy. We define
\begin{equation}
\widehat u_{n,i}^\pi(\pmb s,\pmb a)
:=
\mathbb E_{\widetilde{\bm e}_n}
\left[
\left|
(\bm g_{n,i}^\pi)^\top\widetilde{\bm e}_n
+
\frac12
\widetilde{\bm e}_n^\top H_{n,i}^\pi\widetilde{\bm e}_n
\right|
\right].
\label{eq:implementable-uncertainty}
\end{equation}
For a candidate policy $\pi$, \eqref{eq:implementable-uncertainty} weights parameter uncertainty by its impact on future reward from the current state--action pair. Thus, uncertainty is emphasized only when it induces meaningful digital--physical value discrepancy. Because $\widehat u_{n,i}^\pi$ is state dependent, it naturally guides adaptive experimentation along process trajectories.

The exploration action for physical trajectory $n$ is selected as
\begin{equation} \eta_{n,i}(\pmb s) \in \arg\max_{\pmb a\in\mathcal A} \widehat u_{n-1,i}^{\widehat\pi_{n-1}} (\pmb s,\pmb a), 
\qquad i=0,1,\ldots,T-1,\label{eq:design_objective} \end{equation}
where $\widehat u_{n-1,i}^{\widehat\pi_{n-1}}$ is estimated via the Monte Carlo procedure in Appendix~\ref{app:gaussian-error}. The calibrated model and target policy $\widehat\pi_{n-1}$ remain fixed throughout trajectory $n$, while \eqref{eq:design_objective} is reevaluated at each visited state. Consequently, the design policy favors actions that maximize policy-relevant digital--physical value mismatch rather than overall parameter identifiability.

\subsection{Model Uncertainty-Aware Policy Optimization}
\label{sec:policyOptimizationstep}
Policy optimization targets the expected return under a Gaussian approximation to the distribution of the unknown calibration parameter conditional on the current estimator $\widehat{\bm\theta}_n$. Specifically, using the estimated calibration bias and covariance from \eqref{eq:gaussian-error-surrogate}, we define the Gaussian-averaged return
\[
J_n^{\mathrm G}(\pi)
:=\E
\left[
\E_{\xi}\left[
\mathcal R^\pi(\widehat{\bm\theta}_n+\widetilde{\bm e}_n,\xi)
\right]
\mid\mathcal D_n\right].
\]
Direct evaluation requires repeated return computations over parameter samples. To avoid this cost, we approximate the objective by propagating the estimated calibration bias and covariance through a local second-order expansion of $J(\pi;\bm\theta)$ about $\widehat{\bm\theta}_n$. Using the policy-value gradient and Hessian obtained from adjoint sensitivity analysis gives
\begin{equation}
\bar J_n(\pi)
=J(\pi;\widehat{\pmb\theta}_n)
+\underbrace{\nabla_{\pmb\theta}J(\pi;\widehat{\pmb\theta}_n)^\top\widehat{\pmb b}_n
+\frac12\widehat{\pmb b}_n^\top
\nabla_{\pmb\theta}^2J(\pi;\widehat{\pmb\theta}_n)
\widehat{\pmb b}_n}_{\substack{\text{Bias Correction}
}}
+ \underbrace{\frac12 \mbox{tr}\left( \nabla_{\pmb\theta}^2J(\pi;\widehat{\pmb\theta}_n) \widehat\Sigma_n
\right)}_{\substack{\text{Uncertainty Correction}
}}, 
\label{eq.policyOptimizationObjective}
\end{equation}
which yields a tractable approximation of $J_n^{\mathrm G}(\pi)$.
The target policy is defined by
\begin{equation}
\widehat\pi_n\in\argmax_{\pi\in\Pi} \bar J_n(\pi).
\label{eq:bias_corrected_objective_intro}
\end{equation}
Appendix~\ref{app:gaussian-error} derives the Gaussian moment expansion and its remainder.

\section{Convergence Analysis}
\label{sec:convergence}

We analyze the asymptotic behavior of the proposed uncertainty-aware exploration score and the physical-system performance of the resulting target policy.
Section~\ref{subsec:MU&SA} established consistency and asymptotic normality of the fixed-design QMLE under a fixed trajectory length $T$, yielding the classical estimation rate $\mathcal O_p(n^{-1/2})$.
To quantify the benefit of increasingly frequent observations and their impact on policy optimization, we now consider a dense-sampling regime in which $n,T\to\infty$ while the physical horizon length $t_T-t_0=T\Delta t$ remains fixed. Throughout, the moment-expansion order $l$ and parameter dimension $K$ are fixed, and $\widetilde{\bm\theta}_l$ and $\bm b_{\Delta,l}$ suppress their dependence on $\Delta t=(t_T-t_0)/T$. Proofs are deferred to Appendix~\ref{app:proofs}.

For the dense-sampling regime, define
$\mathsf a(\pmb s;\bm\theta)
:=\pmb\sigma(\pmb s,\bm\theta)\pmb\sigma(\pmb s,\bm\theta)^\top$.
At a fixed state and parameter, Equation~\eqref{eq:moment-expansion-summary} gives (derived in Appendix~\ref{app:dense-sampling-rate}):
\begin{equation}
\begin{aligned}
\widetilde{\pmb m}_{2,l}(\Delta t,\pmb s,\bm\theta)
&=\mathsf a(\pmb s;\bm\theta)\Delta t
+\mathcal O((\Delta t)^2),\\
\bigl[\widetilde{\pmb m}_{2,l}^{-1}
\partial_{\theta_k}\widetilde{\pmb m}_{2,l}\bigr]
(\Delta t,\pmb s,\bm\theta)
&=\mathsf a(\pmb s;\bm\theta)^{-1}
\partial_{\theta_k}\mathsf a(\pmb s;\bm\theta)
+\mathcal O(\Delta t).
\end{aligned}
\label{eq:dense-covariance-expansion}
\end{equation}
Let $\eta_T$ denote the fixed data-collection design over the $T$ sampling intervals. 
Using the quasi-score in Equation~\eqref{eq:transition-quasilikelihood} and the covariance expansion in Equation~\eqref{eq:dense-covariance-expansion}
yields the dense-sampling information matrix
\begin{equation}
\begin{aligned}
[\mathcal I_\sigma(\bm\theta^\ast)]_{kh} := \lim_{T\to\infty}\frac{1}{2T} \sum_{i=0}^{T-1} \mathbb E_{\bm\theta^\ast,\eta_T} \operatorname{tr} \Bigl\{ \mathsf a_i^{-1} \partial_{\theta_k}\mathsf a_i \cdot \mathsf a_i^{-1} \partial_{\theta_h}\mathsf a_i \Bigr\}, 
\end{aligned}
\label{eq:dense-covariance-information}
\end{equation}
for $k,h=1,\ldots,K$, where $\mathsf a_i=\mathsf a(\pmb s_{t_i^+};\bm\theta^\ast)$.
Assume the dense-sampling identifiability condition $\lambda_{\min}\{\mathcal I_\sigma(\bm\theta^\ast)\}>0$, which ensures that every nonzero parameter perturbation induces a detectable change in the diffusion covariance along the observed trajectories.
Combined with the small-interval moment bounds, conditional Lindeberg condition, and uniform score-Hessian convergence established in Appendix~\ref{app:dense-sampling-rate}, we obtain
\begin{equation}
\widehat{\bm\theta}_n-\widetilde{\bm\theta}_l
=\mathcal O_p((nT)^{-1/2}),
\qquad
\widehat C_{l,n}=\mathcal O_p(T^{-1}),
\label{eq:dense-qmle-rate}
\end{equation}
showing that dense sampling improves the fixed-$T$ estimation rate from $\mathcal O_p(n^{-1/2})$ to $\mathcal O_p((nT)^{-1/2})$ by leveraging information from increasingly frequent observations.

To characterize the asymptotic target of the one-step bias estimator in Equation~\eqref{eq:empirical-parameter-bias}, let $d_{l,T}(\pmb\tau;\bm\theta)$ denote the trajectory-level score-correction vector defined in Appendix~\ref{app:one-step-bias}.
The population counterpart of the empirical bias estimator $\widehat{\bm b}_n$ is
\[
\bm b_{\Delta,l}^{\mathrm{plug}}
:=
B_l(\widetilde{\bm\theta}_l)^{-1}
\mathbb E_{\bm\theta^\ast,\eta_T}
\!\left[
d_{l,T}(\pmb\tau;\widetilde{\bm\theta}_l)
\right].
\]
This quantity represents the leading-order parameter bias induced by moment truncation and serves as the asymptotic target of the implementable bias correction.

Theorem~\ref{convergence1} characterizes the asymptotic behavior of the exploration score in \eqref{eq:implementable-uncertainty}. It shows that uncertainty-driven exploration asymptotically targets the local value distortion induced by calibration bias, while estimation error vanishes at the dense-sampling rate $(nT)^{-1/2}$. Under fixed-$T$ asymptotics, the corresponding remainder is only $\mathcal O_p(n^{-1/2})$.
The leading discrepancy $(\bm g_{\infty,i}^\pi)^\top\bm b_{\Delta,l}
+
\frac12
\bm b_{\Delta,l}^\top
H_{\infty,i}^\pi
\bm b_{\Delta,l}$ can vanish despite nonzero parameter bias when the continuation value is locally insensitive to the biased direction or when the first- and second-order contributions cancel.

\begin{theorem}[Pointwise exploration limit]
\label{convergence1}
Fix $(i,\pmb s,\pmb a,\pi)$ and suppose that $Q_i^\pi(\pmb s,\pmb a;\bm\theta)$ is three times continuously differentiable in a neighborhood of $\widetilde{\bm\theta}_l$, with derivatives uniformly bounded in $T$.
Let
$\bm b_{\Delta,l}=\bm\theta^\ast-\widetilde{\bm\theta}_l$, and denote by
$\bm g_{\infty,i}^\pi$ and $H_{\infty,i}^\pi$ the gradient and Hessian of $Q_i^\pi(\pmb s,\pmb a;\bm\theta)$ evaluated at $\widetilde{\bm\theta}_l$. Under the dense-sampling conditions leading to \eqref{eq:dense-qmle-rate},
\[
u_{n,i}^\pi(\pmb s,\pmb a)
=
\left|
(\bm g_{\infty,i}^\pi)^\top\bm b_{\Delta,l}
+
\frac12
\bm b_{\Delta,l}^\top
H_{\infty,i}^\pi
\bm b_{\Delta,l}
\right|
+
\mathcal O_p((nT)^{-1/2}).
\]

Furthermore, let
$Z_n\mid\mathcal D_n\sim N(\bm0,\widehat\Sigma_n)$.
If
$\widehat{\bm b}_n
=
\bm b_{\Delta,l}^{\mathrm{plug}}+\mathcal O_p((nT)^{-1/2})$
and $\widehat C_{l,n}=\mathcal O_p(T^{-1})$, the Gaussian surrogate in Equation~\eqref{eq:implementable-uncertainty} satisfies
\[
\begin{aligned}
\widehat u_{n,i}^\pi(\pmb s,\pmb a)
=\mathbb E_{Z_n\mid\mathcal D_n}\Bigl[
\bigl| (\bm g_{\infty,i}^\pi)^\top
(\bm b_{\Delta,l}^{\mathrm{plug}}+Z_n)+\frac12(\bm b_{\Delta,l}^{\mathrm{plug}}+Z_n)^\top
H_{\infty,i}^\pi(\bm b_{\Delta,l}^{\mathrm{plug}}+Z_n)
\bigr|\Bigr]
+\mathcal O_p((nT)^{-1/2}),
\end{aligned}
\]
and consequently
\[
\widehat u_{n,i}^\pi(\pmb s,\pmb a)
=
\left|
(\bm g_{\infty,i}^\pi)^\top\bm b_{\Delta,l}^{\mathrm{plug}}
+
\frac12
(\bm b_{\Delta,l}^{\mathrm{plug}})^\top
H_{\infty,i}^\pi
\bm b_{\Delta,l}^{\mathrm{plug}}
\right|
+
\mathcal O_p((nT)^{-1/2}).
\]
\end{theorem}

We next quantify how residual calibration error propagates to physical-system policy performance. When the number of sampling intervals varies, write
\[
J(\pi;\bm\theta)
:=\mathbb E_{\bm\theta}^{\pi}\!\left[
\sum_{i=0}^{T-1}\gamma^i r(\pmb s_{t_i},\pmb a_{t_i})
\right],
\qquad \Delta t=(t_T-t_0)/T.
\]
To compare the simulator value at $\widetilde{\bm\theta}_l$ with the physical-system value at $\bm\theta^\ast$, we decompose the value error into bias and uncertainty components. The leading bias contribution is
\[
C_{\mathrm{bias}}(\pi)
:=
\nabla_\theta J(\pi;\widetilde{\bm\theta}_l)^\top
\bm b_{\Delta,l}
+
\frac12
\bm b_{\Delta,l}^\top
\nabla_\theta^2J(\pi;\widetilde{\bm\theta}_l)
\bm b_{\Delta,l},
\]
which captures the value distortion induced by the truncation bias $\bm b_{\Delta,l}=\bm\theta^\ast-\widetilde{\bm\theta}_l$. 

The leading uncertainty contribution is $\frac{1}{2n} \operatorname{tr}\!\Bigl( \nabla_\theta^2J(\pi;\widetilde{\bm\theta}_l) C_l(\widetilde{\bm\theta}_l) \Bigr)$ which captures the effect of QMLE estimation uncertainty on policy value.
Under the uniform value-smoothness condition in Appendix~\ref{app:proofs}, the bias rate $\|\bm b_{\Delta,l}\|=\mathcal O((\Delta t)^l)$ from Appendix~\ref{app:score-bias}, and $C_l(\widetilde{\bm\theta}_l)=\mathcal O(T^{-1})$ from Appendix~\ref{app:dense-sampling-rate}, the remaining higher-order error satisfies
\begin{equation}
\begin{aligned}
\epsilon_{\Delta t,n}
:={}&\sup_{\pi\in\Pi}\Bigl|
J(\pi;\bm\theta^\ast)
-J(\pi;\widetilde{\bm\theta}_l)-C_{\mathrm{bias}}(\pi)-\frac{1}{2n}\operatorname{tr}\!\left(
\nabla_\theta^2J(\pi;\widetilde{\bm\theta}_l)
C_l(\widetilde{\bm\theta}_l)\right)\Bigr|\\
={}&\mathcal O((\Delta t)^{3l})+\mathcal O((nT)^{-1}).
\end{aligned}
\label{eq:oracle-value-remainder}
\end{equation}

Proposition~\ref{app:prop:one-step-accuracy} in Appendix~\ref{app:one-step-bias} shows that the plug-in correction
$\bm b_{\Delta,l}^{\mathrm{plug}}$ approximates the exact truncation bias
$\bm b_{\Delta,l}
=
\overline\Lambda_l^{-1}\bm\delta_l$
with error
$\mathcal O((\Delta t)^{l+1})$.
Therefore, the residual error consists of two components: a deterministic bias-approximation error of order $\mathcal O((\Delta t)^{l+1})$ and a statistical estimation error of order $\mathcal O_p((nT)^{-1/2})$. This decomposition leads directly to the following physical-system performance guarantee.

\begin{theorem}[Physical-system policy performance]
\label{convergence2}
Under the dense-sampling regime, the uniform value-smoothness condition in Appendix~\ref{app:proofs}, the conditions of Proposition~\ref{app:prop:dense-qmle} in Appendix~\ref{app:dense-sampling-rate}, and those of Propositions~\ref{app:prop:dense-bias}--\ref{app:prop:one-step-accuracy} in Appendix~\ref{app:one-step-bias}, the target policy $\widehat\pi_n$ satisfies
\begin{align*}
0
\leq
J(\pi^\ast;\bm\theta^\ast)
-
J(\widehat\pi_n;\bm\theta^\ast)
=
\mathcal O((\Delta t)^{l+1})
+\mathcal O_p((nT)^{-1/2}).
\end{align*}
\end{theorem}

Theorem~\ref{convergence2} shows that the optimality gap vanishes as the sampling interval decreases and the amount of data increases. Consequently, dense sampling not only improves calibration accuracy but also ensures asymptotically optimal policy performance in the physical system.

\section{Empirical Study}
\label{sec: experiments}

We evaluate the finite-sample performance of the proposed Actor--Simulator framework using validated iPSC monolayer and aggregate-culture models \citep{zheng2024stochastic,wang2023metabolic}. Because feeding policies must balance productivity and cell viability, these systems provide representative testbeds for digital-twin calibration and process control. 
Details of the monolayer case study are provided in Appendix~\ref{subsec: kinetic model}.

All experiments are conducted in a synthetic environment, where a simulator validated against experimental data serves as the physical system \(\mathcal M^p\). Although trajectories are generated using the ground-truth parameter vector $\bm\theta^\ast$, the Actor--Simulator observes only the resulting state--action trajectories during calibration, uncertainty-guided exploration, bias correction, and policy optimization. The ground-truth parameters are used solely to specify the initialization range and to compute post hoc calibration error. At each policy checkpoint, the current policy is frozen and evaluated on 1,000 independent simulator trajectories. These evaluation trajectories are excluded from $\mathcal D_n$, do not update the model or policy, and are not counted toward the training budget $N_{\max}$. Each experimental configuration is repeated over five independent training seeds. Unless otherwise stated, solid curves denote seed-wise means and shaded regions indicate pointwise 95\% confidence intervals.

Calibration-error curves quantify recovery of the underlying mechanistic parameters, whereas reward curves measure the downstream decision quality induced by the learned model and policy. Both metrics are necessary: parameter accuracy does not necessarily translate into improved decisions along insensitive directions, while high reward alone does not imply accurate recovery of the underlying mechanisms.


We use the state $\pmb s_t=(X_t,\pmb u_t^\top)^\top$, where $X_t$ is cell density and $\pmb u_t$ contains extracellular metabolite concentrations. At decision time $t_i$, the action $a_{t_i}=\pi(\pmb s_{t_i})\in[0,1]$ replaces that fraction of medium, giving $\pmb u_{t_i}^+=a_{t_i}\pmb u_{t_0}+(1-a_{t_i})\pmb u_{t_i}$ for fresh-medium composition $\pmb u_{t_0}$.

In suspension culture, iPSCs form three-dimensional aggregates whose internal microenvironments can differ substantially from the bulk medium. To capture this heterogeneity, the digital twin couples a population-balance model for aggregate-size evolution with a reaction--diffusion model for intra-aggregate nutrient and metabolite transport. Aggregate growth and coalescence determine the size distribution, while diffusion limitations create nutrient gradients that alter local metabolic activity, growth, and death.

Population- and volume-weighted averaging of these local responses produces the effective growth rate $\bar\mu$, death rate $\bar\mu_d$, and reaction-rate vector $\overline{\pmb{v}}$. With stoichiometric matrix $\pmb N$, the resulting macroscopic dynamics are
\[ dX_t=(\bar\mu-\bar\mu_d)X_t\,dt
 +\sqrt{(\bar\mu+\bar\mu_d)X_t}\,dW_t^X, 
 \qquad
d\pmb u_t=\pmb N\overline{\pmb{v}}X_t\,dt
 +\pmb N\sqrt{\operatorname{diag}(\overline{\pmb{v}}X_t)}\,dW_t.
\]
Consequently, adequate bulk nutrient levels may still coexist with nutrient-starved aggregate cores. This multiscale coupling allows feeding decisions to influence both overall productivity and spatial heterogeneity.

The policy balances biomass production and aggregate heterogeneity. Let $q_t^{\mathrm{core}}$ denote the fraction of cells in aggregates with nutrient-starved cores (central glucose concentration below $2.5$ mM), and define $\Delta q_{t_i}^{\mathrm{core}}=q_{t_{i+1}}^{\mathrm{core}}-q_{t_i}^{\mathrm{core}}$. The aggregate-culture reward is
\begin{equation}
r(\pmb s_{t_i}, \pmb a_{t_i}) = c_r \Delta X_{t_i} - c_m b_{t_i} - c_l \Delta\mbox{[ELAC]}_{t_i} - c_h \Delta q_{t_i}^{\mathrm{core}},
\end{equation}
where $\Delta X_{t_i}=X_{t_{i+1}}-X_{t_i}$ and $\Delta\mbox{[ELAC]}_{t_i}$ denotes the change in extracellular lactate concentration. The weights $c_r$, $c_m$, $c_l$, and $c_h$ quantify the relative importance of biomass production, medium consumption, lactate accumulation, and nutrient-starved core expansion.

Experiments use 48-h episodes with 4-h decision intervals, discount factor $\gamma=0.99$, and 20, 30, or 40 unknown parameters. For the proposed method, we set the conditional-moment expansion order to $l=2$. Initial estimates $\widehat\theta_{0,k}$ are sampled uniformly from $[0,4\theta_k^\ast]$. For each setting, the Actor--Simulator has a trajectory budget of $N_{\max}=160$, including $n_0=10$ pilot trajectories before uncertainty-guided exploration. The remaining episodes follow Algorithm~\ref{actorsimulatoralgo} in Appendix~\ref{sec:algorithm}. We compare against Gaussian-process and random-exploration baselines, evaluating calibration accuracy using $\left\Vert(\widehat{\bm\theta}_n-\bm\theta^\ast)/\bm\theta^\ast\right\Vert$ (componentwise division). Confidence intervals and independent policy evaluation follow the protocol above.

As shown in Figure~\ref{fig: calibration performance_agg}, the proposed Actor--Simulator framework scales effectively to the more complex Bio-SoS setting. Compared with the Gaussian Process surrogate and Random Policy baselines, it consistently delivers more accurate calibration and parameter recovery. These gains in model fidelity translate into improved control performance, as illustrated by the reward trajectories in Figure~\ref{fig: reward_agg}. Across all three parameter settings, the Actor--Simulator achieves faster policy improvement and higher final rewards. 

\begin{figure}[ht]
 \centering
 \vspace{-0.in}
 \includegraphics[width=0.8\textwidth]{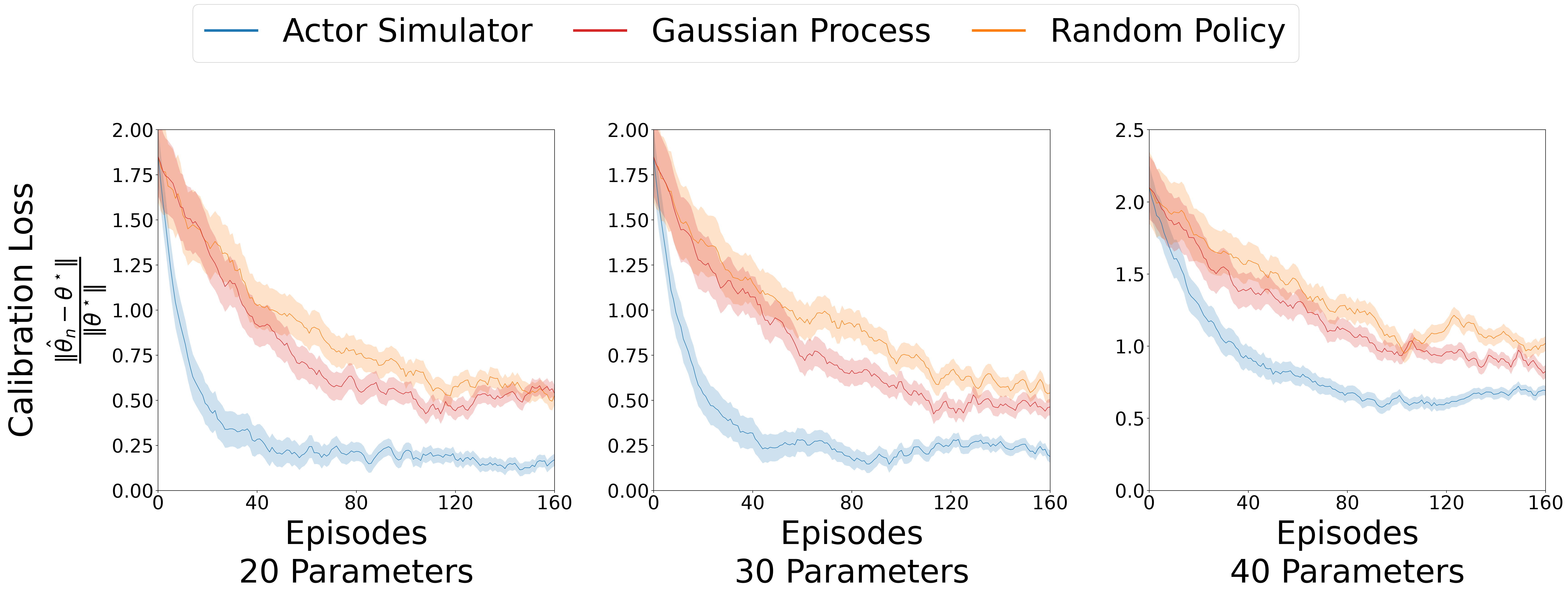}
 \vspace{-0.in}
 \caption{Comparison of calibration parameter estimation error
 $\left\Vert \frac{\widehat{\bm\theta}_n-\bm\theta^\ast}{\bm\theta^\ast}\right\Vert$,
 obtained by the three candidate approaches in the 20-, 30-, and 40-unknown-parameter settings. Solid curves report the mean and shaded regions denote pointwise 95\% confidence intervals.
 }\label{fig: calibration performance_agg}
 \vspace{-0.in}
\end{figure}

Accurate prediction of \emph{Bio-SoS heterogeneity} is essential for biomanufacturing and quality control. As aggregate size increases, diffusion limitations generate distinct core and peripheral microenvironments that can induce different metabolic states. We evaluate whether the calibrated digital twin captures this spatial heterogeneity by comparing inner- and outer-shell fluxes in the 40-parameter aggregate-culture case study. Using the calibrated parameters from the 15th RL iteration of each random seed, we simulate each of the five parameter estimates separately and then average the resulting flux-heterogeneity measures pointwise. Averaging outputs rather than parameters preserves nonlinear effects in the digital twin.

\begin{figure}[ht]
 \centering
 \vspace{-0.in}
 \includegraphics[width=0.8\textwidth]{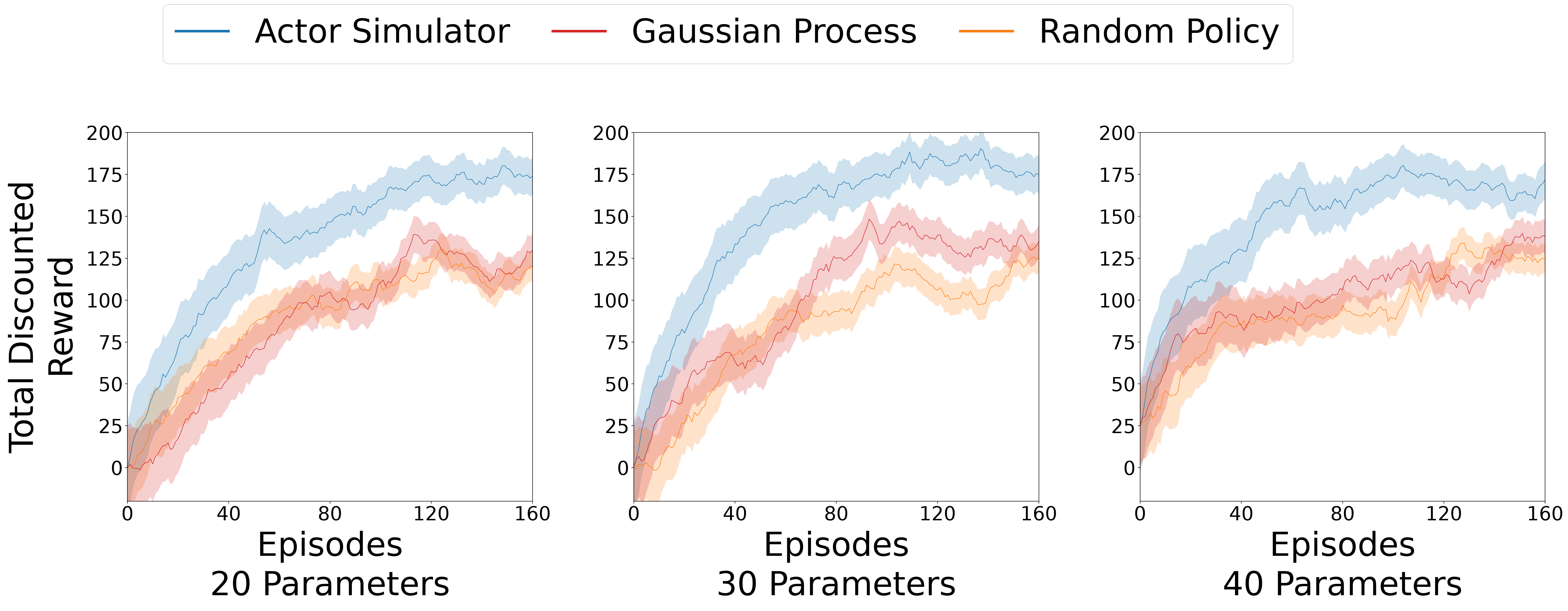}
 \vspace{-0.in}
 \caption{Comparison of policy optimization performance in terms of $J(\widehat\pi_n; \mathcal{M}^p)$ obtained by the three candidate approaches in the 20-, 30-, and 40-unknown-parameter settings. Solid curves report the mean and shaded regions denote pointwise 95\% confidence intervals.
 }\label{fig: reward_agg}
 \vspace{-0.in}
\end{figure}

Although policy training uses 48-h episodes, the calibrated mechanistic models are propagated for up to 73 h with a 1-h update interval, so the 72-h evaluations assess predictive performance beyond the training horizon. For spatial flux analysis, the aggregate simulator assumes a single-cell radius of 7.5~$\mu$m, radial discretization of 15~$\mu$m, aggregate porosity of 0.27, and tortuosity of 1.46. Inner- and outer-shell fluxes are extracted at 24, 48, and 72 h for nominal aggregate radii of 60, 120, 240, and 360~$\mu$m. For reaction or transport flux $j$, time $t$, and aggregate radius $R$, we define the flux-heterogeneity contrast as
\[
C^{\mathrm{flux}}_{j,t,R}
:=
\frac{\nu^{\mathrm{inner}}_{j,t,R}-\nu^{\mathrm{outer}}_{j,t,R}}
{|\nu^{\mathrm{inner}}_{j,t,R}|+|\nu^{\mathrm{outer}}_{j,t,R}|}.
\]
Positive values indicate greater flux activity in the aggregate interior, whereas negative values indicate greater activity near the surface.

As shown in Figure~\ref{fig:metabolic-heterogeneity-mechanism}, the calibrated digital twin preserves the key pathway-level heterogeneity patterns observed in the physical system, 
including the increasingly negative pyruvate-transport contrast with aggregate radius at 24 and 48 h, the pronounced negative SAL and biomass contrasts at larger radii at 48 h, and the opposing PGK and PK spatial trends at 72 h. The model also captures the localized 72-h heterogeneity signatures of malate-associated reactions and PC near the nominal 240~$\mu$m radius.

\begin{figure}[htbp]
    \centering
    \vspace{-0.05in}
    \includegraphics[width=\textwidth]{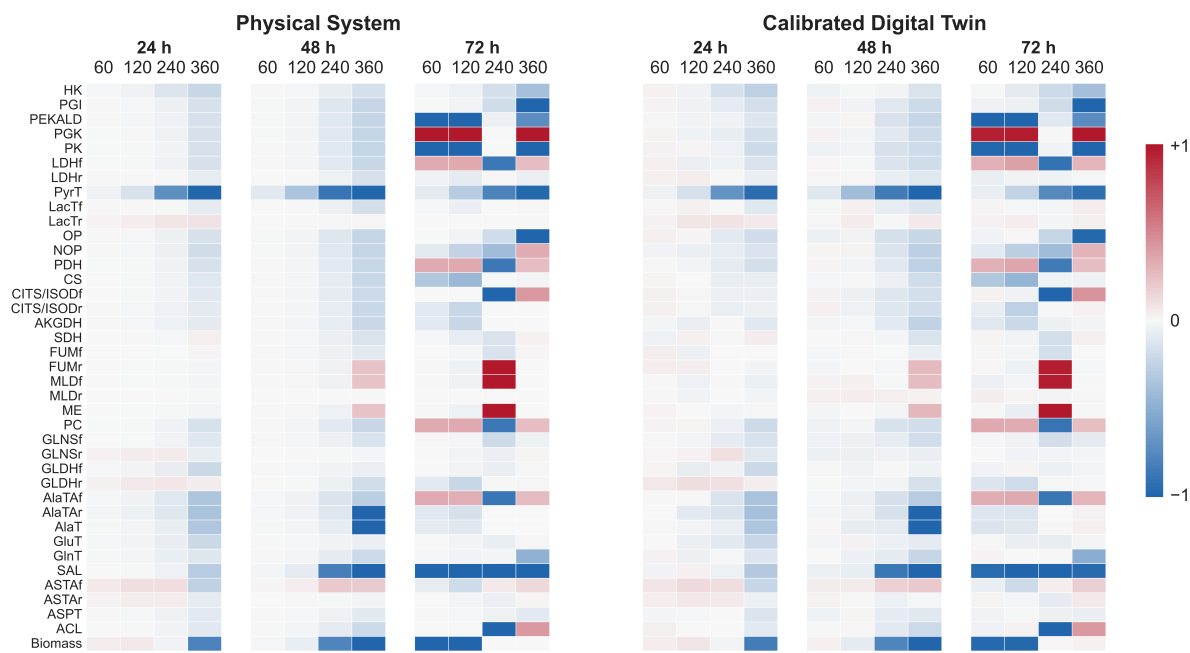}
    \vspace{-0.1in}
    \caption{Reference and calibrated inner--outer metabolic-flux contrasts at 24, 48, and 72 h for aggregate radii of 60, 120, 240, and 360~$\mu$m. At the 15th iteration, the absolute error in the normalized contrast is below 0.05 for every displayed entry.}
    \label{fig:metabolic-heterogeneity-mechanism}
\end{figure}

In both the aggregate-culture study and the monolayer study (Appendix~\ref{subsec: kinetic model}), the bias-aware Actor-Simulator framework achieves faster mechanistic parameter calibration and higher policy rewards than Gaussian-process and random-exploration baselines under the same trajectory budget, demonstrating superior sample efficiency. Spatial flux analysis further shows that the calibrated digital twin accurately captures spatiotemporal metabolic heterogeneity within aggregates. This mechanistic fidelity is particularly important for gene and cell therapy manufacturing, organoid production, and 3D bioprinting, where local microenvironments strongly influence product quality and consistency. Overall, the results highlight the benefits of jointly integrating uncertainty-aware digital-twin calibration and policy optimization within mechanistic Bio-SoS models.

\section{Conclusion}
\label{sec:conclusion}
This paper develops a bias-aware Actor--Simulator framework that integrates uncertainty-aware calibration, adaptive experimental design, and policy optimization for SDE-based digital twins. The proposed approach combines generator-based moment expansions to characterize truncation-induced calibration bias with trajectory-level adjoint sensitivity analysis to quantify how parameter uncertainty propagates to action values and policy performance. These components enable policy-directed experiment selection and uncertainty-aware policy optimization through a second-order Gaussian-averaged objective.
Theoretical analysis establishes improved dense-sampling calibration rates, characterizes the asymptotic limit of the proposed exploration score, and proves asymptotically optimal physical-system performance. Numerical studies on iPSC monolayer and aggregate-culture systems demonstrate faster calibration, greater sample efficiency, and superior policy performance relative to Gaussian-process and random-exploration baselines. Overall, the proposed framework establishes a principled link between uncertainty quantification, model calibration, and decision optimization in mechanistic Bio-SoS digital twins.

\begin{sloppypar}
\section*{Funding Statement}
The authors acknowledge support from the National Institute of Standards and Technology (Grants 70NANB24H293, 70NANB17H002, 70NANB21H086) and the National Science Foundation (CMMI-2442970).
\end{sloppypar}

\bibliography{wscbib,adjoint_references,uncertainty_references,biosos_references}
\clearpage

\appendix
\numberwithin{figure}{section}
\numberwithin{table}{section}
\numberwithin{theorem}{section}

\section{Reaction-Network Motivation}
\label{app:reaction-network}

\begin{sloppypar}
Consider a reaction network with $M$ species and $R$ reactions. Let
$\pmb N=(\pmb N_1,\ldots,\pmb N_R)\in\mathbb R^{M\times R}$
denote the stoichiometric matrix, and let
$\pmb v(\pmb s,\bm\theta)$ denote the vector of parameterized reaction rates.
Over a microscopic interval $(t,t+\Delta\tau]$, the state evolves according to
\begin{equation}
\pmb s_{t+\Delta\tau}
=\pmb s_t+\pmb N\pmb R_t,
\qquad
\pmb R_t\sim\operatorname{Poisson}
\!\left(\pmb v(\pmb s_t,\bm\theta)\Delta\tau\right),
\label{app:eq:reaction-count}
\end{equation}
where the components of $\pmb R_t$ are conditionally independent reaction counts. For sufficiently small $\Delta\tau$, the reaction-count vector admits the Gaussian approximation
\[
\pmb s_{t+\Delta\tau}\mid\pmb s_t
\ \dot\sim\
N\!\left(
\pmb s_t+\pmb N\pmb v(\pmb s_t,\bm\theta)\Delta\tau,
\pmb N\operatorname{diag}(\pmb v(\pmb s_t,\bm\theta))\pmb N^\top\Delta\tau
\right).
\]
Taking $\Delta\tau\downarrow0$ yields the chemical Langevin equation underlying Equation~\eqref{SDE_state}. The framework considered in this paper allows more general drift and diffusion specifications, enabling reaction, transport, population-balance, and spatial modules to be coupled within a unified Bio-SoS model.

\end{sloppypar}

\section{iPSC Monolayer Culture Case Study}
\label{subsec: kinetic model}

Control of iPSC cultures aims to improve cell growth, viability, and product quality consistency. Because culture performance is highly sensitive to nutrient depletion and metabolic waste accumulation, feeding policies play a critical role in maintaining productivity and product quality. We consider a setting in which the metabolic-network structure is known, but key kinetic and regulatory parameters are uncertain. To model this system, we adopt the metabolic network of \cite{wang2023metabolic}, comprising 30 central metabolic reactions, 34 state variables, and 73 parameters. Figure~\ref{fig: culture} illustrates the metabolic network and medium-exchange control problem. The control action is the medium-exchange ratio.

\begin{figure}[htbp]
\centering
\vspace{-0.1in}
\includegraphics[width=0.8\textwidth]{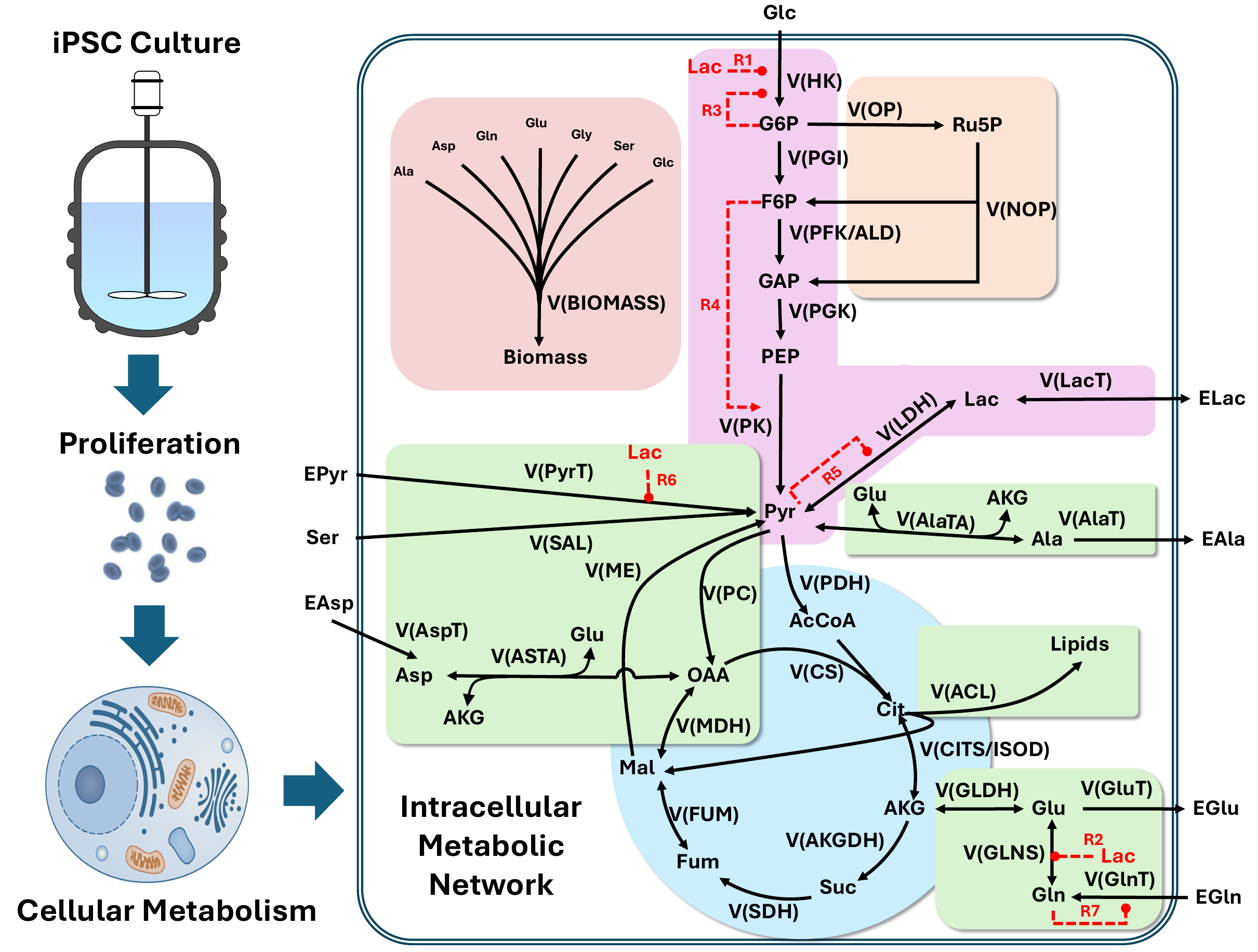}
\vspace{-0.in}
\caption{Schematic of the iPSC culture metabolisms and medium-exchange decision problem.}
\label{fig: culture}
\end{figure}

\subsection{iPSC Model, Parameter Settings, and Experimental Setup}
\label{app:example-model}

The monolayer-culture model uses the state
$\pmb s_t=(X_t,\pmb u_t^\top)^\top$, where $X_t$ denotes cell density
(cells/cm$^2$) and $\pmb u_t$ contains extracellular metabolite
concentrations. Following the study \citep{kornecki2018process}, cell growth is promoted by glucose ([GLC]) and glutamine ([EGLN]) and inhibited by lactate ([ELAC]):
\begin{equation*}
\begin{aligned}
dX_t
&=(\mu-\mu_d)X_t\,dt
+\sqrt{(\mu+\mu_d)X_t}\,dW_t^X,\\
\mu_d
&=k_d\frac{\mbox{[ELAC]}}{\mbox{[ELAC]}+K_{Dlac}},\\
\mu
&=\mu_{max}
\frac{\mbox{[GLC]}}{K_{glc}+\mbox{[GLC]}}
\frac{\mbox{[EGLN]}}{K_{gln}+\mbox{[EGLN]}}
\cdot\frac{K_{Ilac}}{K_{Ilac}+\mbox{[ELAC]}}.
\end{aligned}
\end{equation*}
Here $\mu_{max}$ and $k_d$ denote the maximum growth and death rates,
$K_{glc}$, $K_{gln}$, $K_{Dlac}$, and $K_{Ilac}$ are the corresponding
half-saturation and inhibition constants. Following the feeding update, extracellular metabolites evolve according to
\begin{equation*}
\begin{aligned}
d\pmb u_t
=\pmb N\pmb v(\pmb s_t,\bm\theta)X_t\,dt
+\pmb N\sqrt{\operatorname{diag}
(\pmb v(\pmb s_t,\bm\theta)X_t)}\,dW_t,
\end{aligned}
\end{equation*}
where \(\pmb N\) and \(\pmb v\) denote the stoichiometric matrix and reaction-rate vector, respectively (Appendix~\ref{app:reaction-network}). The network structure and nominal parameter values are adopted from \citet{wang2023metabolic} and validated against experimental data.

We consider nested calibration problems with 20, 30, and 40 unknown parameters, summarized in Table~\ref{table: calibration parameter}. The table reports the corresponding ground-truth values, while all remaining parameters are fixed at the baseline values from \citet{wang2023metabolic}.

\begin{table}[htbp]
\centering
\caption{Calibration parameters used in the monolayer and aggregate experiments.}
\label{table: calibration parameter}
\begin{tabular}{@{}llllll@{}}
\toprule
\multicolumn{6}{c}{20 parameters in all settings}\\
\midrule
$v_{max,HK}$&2.92&$v_{max,PGI}$&1.43&$v_{max,PFK/ALD}$&2.16\\
$v_{max,PGK}$&4.00&$v_{max,PK}$&3.98&$v_{max,fLDH}$&3.28\\
$v_{max,PyrT}$&0.17&$v_{max,fLacT}$&2.97&$v_{max,OP}$&0.01\\
$v_{max,NOP}$&0.02&$v_{max,PDH}$&0.22&$v_{max,CS}$&0.43\\
$v_{max,ME}$&0.51&$v_{max,fMDH}$&1.44&$v_{max,GlnT}$&1.81\\
$K_{m,NH4}$&0.17&$K_{m,ALA}$&0.20&$K_{m,GLC}$&1.46\\
$K_{m,GLN}$&0.26&$K_{m,GLU}$&0.30&&\\
\midrule
\multicolumn{6}{c}{Additional 10 parameters in the 30- and 40-parameter settings}\\
\midrule
$v_{max,fCITS/ISOD}$&1.32&$v_{max,AKGDH}$&2.84&$v_{max,SDH}$&0.32\\
$v_{max,fFUM}$&0.32&$v_{max,PC}$&0.06&$v_{max,fGLNS}$&1.14\\
$v_{max,fGLDH}$&0.26&$v_{max,fAlaTA}$&0.82&$v_{max,AlaT}$&0.47\\
$v_{max,GluT}$&0.17&&&&\\
\midrule
\multicolumn{6}{c}{Additional 10 parameters in the 40-parameter setting}\\
\midrule
$K_{m,SER}$&0.01&$v_{m,SAL}$&0.01&$K_{m,Ru5P}$&0.02\\
$K_{m,PYR}$&0.21&$K_{m,AcCoA}$&0.09&$K_{m,OAA}$&0.08\\
$K_{m,CIT}$&0.39&$K_{m,AKG}$&2.92&$K_{m,MAL}$&0.11\\
$K_{m,EGLN}$&1.00&&&&\\
\bottomrule
\end{tabular}
\end{table}

Observations are collected every $\Delta t=4$ h as transitions $(\pmb{s}_{t_i},\pmb{a}_{t_i},\pmb{s}_{t_{i+1}})$, where the post-decision state is given by $\pmb{s}_{t_i^+}=\pmb f(\pmb{s}_{t_i},\pmb a_{t_i})$. The continuous-time dynamics satisfy
\begin{equation}
  \pmb{s}_{t_{i+1}} = \pmb{s}_{t_i^+}
  +\int_{t_i}^{t_{i+1}}\pmb{\mu}(\pmb{s}_{u},\bm{\theta})\,du
  +\int_{t_i}^{t_{i+1}}\pmb{\sigma}(\pmb{s}_{u},\bm{\theta})\,dW_u,
  \label{eq.stateTransitionExample}
\end{equation}
and are simulated using an internal step size $h_{\mathrm{sim}}\le\Delta t$.

At decision epoch $t_i$, the policy selects a medium-exchange fraction $ a_{t_i}=\pi(\pmb{s}_{t_i})\in[0,1]$, yielding the post-intervention update $\pmb{u}_{t_i}^+ = a_{t_i}\pmb{u}_{t_0} + (1-a_{t_i})\pmb{u}_{t_i}$, where $\pmb{u}_{t_0}$ denotes the fresh-medium composition.
The reward for taking action $ a_{t_i}$ at state $\pmb{s}_{t_i}$ during the $i$-th interval is
\begin{equation}
r(\pmb{s}_{t_i}, a_{t_i}) = c_r \Delta X_{t_i} - c_m a_{t_i} - c_l \Delta\mbox{[ELAC]}_{t_i},
\label{eq.exampleReward}
\end{equation}
where $\Delta X_{t_i}=X_{t_{i+1}}-X_{t_i}$ and $\Delta\mbox{[ELAC]}_{t_i}$ is the change in extracellular lactate concentration; $c_r$, $c_m$, and $c_l$ weight cell production, medium use, and lactate accumulation, respectively. The objective is to maximize the discounted return $J(\pi;\mathcal M^p)$ with discount factor $\gamma=0.99$.

Digital-twin calibration targets the mechanistic parameters governing iPSC growth and metabolism. Equation~\eqref{eq.stateTransitionExample} captures nutrient- and lactate-dependent cell-growth dynamics together with extracellular-metabolite evolution, while $\pmb f$ represents the instantaneous feeding intervention. The model equations and calibration parameter values are given above.

\subsection{Finite-Sample Performance}
\label{subsec: Finite-Sample Performance}

For each  $k$-th unknown parameter, the initial estimate $\widehat\theta_{0,k}$ is sampled uniformly from $[0,4\theta_k^\ast]$. Training proceeds over repeated 48-hour culture episodes, each comprising $T=12$ four-hour observation and decision intervals. Calibration accuracy is measured by the relative error $\left\Vert(\widehat{\bm\theta}_n-\bm\theta^\ast)/\bm\theta^\ast\right\Vert$, with componentwise division.

We compare the proposed approach against two baselines: a \textit{Gaussian Process} surrogate model and a \textit{Random Policy}. The Gaussian Process represents a standard black-box surrogate commonly used in Bayesian optimization, whereas the Random Policy provides a benchmark for unguided exploration.

The calibration and policy-optimization results are reported in Figures~\ref{fig: calibration performance} and \ref{fig: RL performance}, respectively. Across all settings, the proposed bias-aware Actor-Simulator framework consistently improves both parameter estimation and process control.
For the 20- and 30-parameter cases, the relative calibration error decreases rapidly to approximately $0.2$--$0.25$ within 20 episodes and further to about $0.1$ by episode 80. In contrast, the Gaussian Process and Random Policy baselines remain near $0.2$--$0.3$ and $0.3$--$0.4$, respectively. The advantage persists in the more challenging 40-parameter setting, where the proposed method attains a final calibration error of roughly $0.35$, compared with approximately $0.65$ for the Gaussian Process and $0.8$ for the Random Policy.

\begin{figure}[ht]
 \centering
 \includegraphics[width=0.85\textwidth]{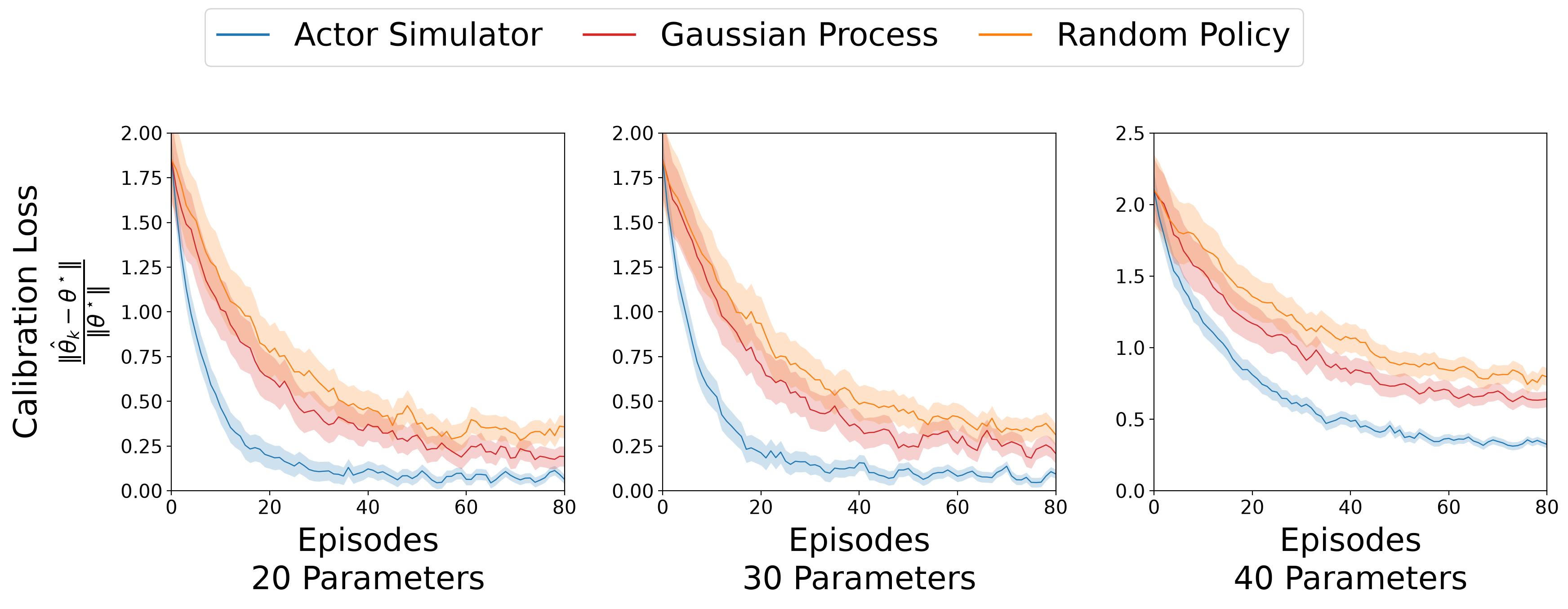}
 \vspace{-0.in}
 \caption{Comparison of calibration performance, measured by the relative parameter-estimation error
 $\left\Vert \frac{\widehat{\bm\theta}_n-\bm\theta^\ast}{\bm\theta^\ast}\right\Vert$,
 for the three methods under 20-, 30-, and 40-parameter calibration settings. Solid curves denote means and shaded regions indicate pointwise 95\% confidence intervals.
 }\label{fig: calibration performance}
\end{figure}

\begin{figure}[ht]
 \centering
 \includegraphics[width=0.85\textwidth]{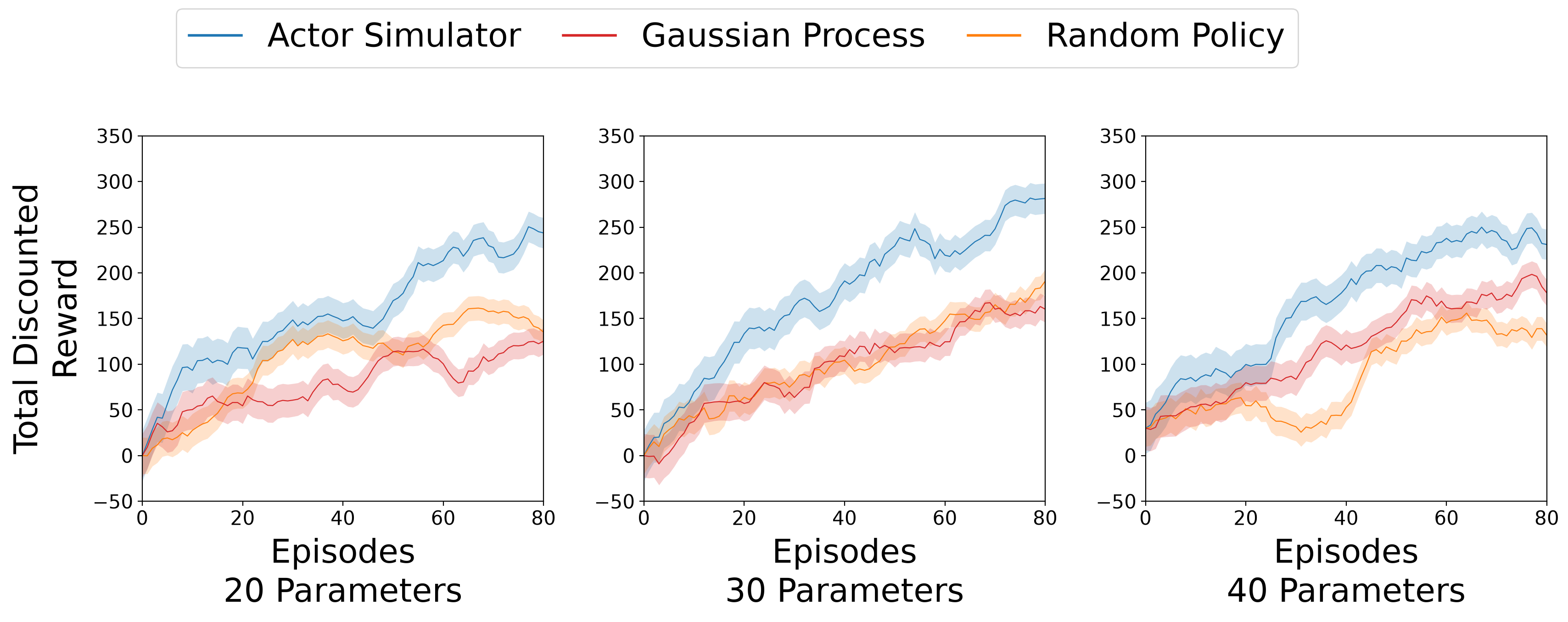}
 \vspace{-0.in}
 \caption{Comparison of policy optimization performance, measured by $J(\widehat{\pi}_n;\mathcal{M}^p)$, for the three methods under 20-, 30-, and 40-parameter calibration settings. Solid curves show means across five training seeds; shaded regions denote pointwise 95\% confidence intervals.}\label{fig: RL performance}
\end{figure}

This improved model fidelity translates into superior control performance. As shown in Figure~\ref{fig: RL performance}, the proposed method consistently achieves the highest total discounted reward. By episode 80, it attains mean rewards of approximately $230$--$280$ across the 20- and 30-parameter settings, compared with roughly $125$--$190$ for the Gaussian Process and Random Policy baselines. The advantage persists in the more challenging 40-parameter case, where the Actor--Simulator reaches about $230$, versus approximately $180$ for the Gaussian Process and $135$ for the Random Policy.

\section{Transition-Level Score Bias and Parameter Bias}
\label{app:score-bias}

We begin with the full transition quasi-score used in the bias calculation. Its $k$-th component is
\begin{align}
\partial_{\theta_k}\ell_{l,i}^j(\bm\theta)
={}&
(\partial_{\theta_k}m_{1,i}^j)^\top
(m_{2,i}^j)^{-1}\varepsilon_i^j
\nonumber\\
&+\frac12(\varepsilon_i^j)^\top
(m_{2,i}^j)^{-1}(\partial_{\theta_k}m_{2,i}^j)
(m_{2,i}^j)^{-1}\varepsilon_i^j
-\frac12\operatorname{tr}
\!\left((m_{2,i}^j)^{-1}\partial_{\theta_k}m_{2,i}^j\right),
\label{app:eq:quasi-score}
\end{align}
where all moment functions are evaluated at $\bm\theta$. The trajectory score preserves temporal dependence within a trajectory, while the independent sampling unit in the sandwich covariance remains the complete trajectory.

Fix a transition $(i,j)$ and evaluate the truncated moments at $\bm\theta^\ast$. Define
\[
\widetilde M_1=m_{1,i}^j(\bm\theta^\ast),\qquad
\widetilde M_2=m_{2,i}^j(\bm\theta^\ast),\qquad
G_{q,k}=\partial_{\theta_k}m_{q,i}^j(\bm\theta^\ast),
\]
and let $R_1$ and $R_2$ denote the corresponding conditional-moment truncation errors. Conditional on the post-decision state,
\[
\mathbb E[\pmb s_{t_{i+1}}^j\mid\pmb s_{t_i^+}^j]
=\widetilde M_1+R_1,
\qquad
\operatorname{Var}(\pmb s_{t_{i+1}}^j\mid\pmb s_{t_i^+}^j)
=\widetilde M_2+R_2.
\]

Let $\varepsilon := \pmb s_{t_{i+1}}^j-\widetilde M_1$. Then,
\begin{equation}
\begin{aligned}
\mathbb E[\varepsilon\mid\pmb s_{t_i^+}^j]&=R_1,\\
\mathbb E[\varepsilon\varepsilon^\top\mid\pmb s_{t_i^+}^j]
&=\widetilde M_2+R_2+R_1R_1^\top.
\end{aligned}
\label{app:eq:score-residual-moments}
\end{equation}
Equations~\eqref{app:eq:quasi-score} and
\eqref{app:eq:score-residual-moments} cancel the two terms involving
$\widetilde M_2$ and yield
\begin{align}
\delta_{l,i,j}^{(k)}
:={}&
\mathbb E[\partial_{\theta_k}\ell_{l,i}^j(\bm\theta^\ast)\mid\pmb s_{t_i^+}^j]\nonumber\\
={}&
G_{1,k}^\top\widetilde M_2^{-1}R_1
+\frac12\operatorname{tr}
\!\left(\widetilde M_2^{-1}G_{2,k}\widetilde M_2^{-1}R_2\right)
+\frac12\operatorname{tr}
\!\left(\widetilde M_2^{-1}G_{2,k}\widetilde M_2^{-1}R_1R_1^\top\right).
\label{app:eq:transition-score-bias}
\end{align}
Equation~\eqref{app:eq:transition-score-bias} shows that the expected Gaussian quasi-score is unbiased whenever the exact conditional moments are used. Indeed, setting $R_1=R_2=0$ yields $\delta_{l,i,j}^{(k)}=0$ without requiring the transition distribution itself to be Gaussian.

Let $\bm\delta_{l,i,j}$ collect the $K$ components in
Equation~\eqref{app:eq:transition-score-bias}. Summing over all stages
yields Equation~\eqref{eq:population-score-bias}. Applying the fundamental
theorem of calculus to $G_l$ along the segment joining $\bm\theta^\ast$
and $\widetilde{\bm\theta}_l$ gives
\begin{equation}
G_l(\widetilde{\bm\theta}_l)-G_l(\bm\theta^\ast)
=
\overline\Lambda_l
(\widetilde{\bm\theta}_l-\bm\theta^\ast).
\label{app:eq:population-score-mean-value}
\end{equation}
Equation~\eqref{eq:exact-parameter-bias} follows because
$G_l(\widetilde{\bm\theta}_l)=0$ and
$G_l(\bm\theta^\ast)=\bm\delta_l$ make
Equation~\eqref{app:eq:population-score-mean-value} become
\[
-\bm\delta_l
=\overline\Lambda_l
(\widetilde{\bm\theta}_l-\bm\theta^\ast),
\]

\begin{sloppypar}
\vspace{0.1in}
\paragraph{Small-interval scaling}
For observation interval $\Delta t$, Assumption~\ref{Regularity Cond} implies
$R_{1,l},R_{2,l}=\mathcal O((\Delta t)^{l+1})$ uniformly on the relevant state-parameter domain. Moreover, the leading transition covariance satisfies
$\widetilde M_2=\mathcal O(\Delta t)$ and
$\widetilde M_2^{-1}=\mathcal O((\Delta t)^{-1})$.
For every parameter coordinate, $G_{1,k}=\mathcal O(\Delta t)$ and
$G_{2,k}=\mathcal O(\Delta t)$ under the differentiated moment expansion.

The three terms in Equation~\eqref{app:eq:transition-score-bias}
therefore have orders
\begin{equation}
\begin{aligned}
G_{1,k}^\top\widetilde M_2^{-1}R_1
&=\mathcal O((\Delta t)^{l+1}),\\
\operatorname{tr}(\widetilde M_2^{-1}G_{2,k}
\widetilde M_2^{-1}R_2)
&=\mathcal O((\Delta t)^l),\\
\operatorname{tr}(\widetilde M_2^{-1}G_{2,k}
\widetilde M_2^{-1}R_1R_1^\top)
&=\mathcal O((\Delta t)^{2l}).
\end{aligned}
\label{app:eq:score-term-orders}
\end{equation}
Summing Equation~\eqref{app:eq:score-term-orders} over $T$
observation intervals gives
\begin{equation}
\|\bm\delta_l\|=\mathcal O\!\bigl(T(\Delta t)^l\bigr).
\label{app:eq:score-bias-rate}
\end{equation}

Under the full-rank covariance condition $\lambda_{\min}\{\mathcal I_\sigma(\bm\theta^\ast)\}>0$ in
Equation~\eqref{eq:dense-covariance-information} and the uniform convergence assumed in Appendix~\ref{app:dense-sampling-rate}, the population score
sensitivity obeys
\[ -T^{-1}\overline\Lambda_l \to \mathcal I_\sigma(\bm\theta^\ast) \qquad (\Delta t\to0), \] implying \[ \|\overline\Lambda_l^{-1}\| = \mathcal O(T^{-1}). \] Combining this with Equation~\eqref{eq:exact-parameter-bias} yields
\begin{equation}
\|\bm b_{\Delta,l}\| = \mathcal O((\Delta t)^l).
\label{app:eq:parameter-bias-rate}
\end{equation}
Thus an order-$(l+1)$ conditional-moment truncation error induces an order-$l$ parameter bias because the covariance contribution to the score is amplified by the factor $\widetilde M_2^{-1}=\mathcal O((\Delta t)^{-1})$.

\end{sloppypar}

\vspace{0.1in}

\paragraph{Implementable score-bias estimate}
The exact residuals $R_{q,l}$ in Equation~\eqref{app:eq:transition-score-bias}
are unknown. We approximate them using the next term in the generator expansion and evaluate all quantities at the QMLE $\widehat{\bm\theta}_n$:
\begin{equation}
\begin{aligned}
\widehat R_{q,i,j}
&:=\widehat R_{q,l}^{\mathrm{next}}
(\Delta t,\pmb s_{t_i^+}^j,\widehat{\bm\theta}_n)\\
&=\widetilde{\pmb m}_{q,l+1}
(\Delta t,\pmb s_{t_i^+}^j,\widehat{\bm\theta}_n)-\widetilde{\pmb m}_{q,l}
(\Delta t,\pmb s_{t_i^+}^j,\widehat{\bm\theta}_n),\\
\widehat M_{2,i,j}
&:=m_{2,i}^j(\widehat{\bm\theta}_n),\qquad
\widehat G_{q,k,i,j}
:=\partial_{\theta_k}m_{q,i}^j(\widehat{\bm\theta}_n),
\quad q=1,2.
\end{aligned}
\label{app:eq:estimated-score-residuals}
\end{equation}

Equations~\eqref{app:eq:transition-score-bias} and
\eqref{app:eq:estimated-score-residuals} give the estimated contribution
to the $k$th score component:
\begin{align}
\widehat\delta_{l,i,j}^{(k)}
:={}&\widehat G_{1,k,i,j}^\top
\widehat M_{2,i,j}^{-1}\widehat R_{1,i,j}
+\frac12\operatorname{tr}\!\left(
\widehat M_{2,i,j}^{-1}\widehat G_{2,k,i,j}
\widehat M_{2,i,j}^{-1}\widehat R_{2,i,j}
\right)\nonumber\\
&+\frac12\operatorname{tr}\!\left(
\widehat M_{2,i,j}^{-1}\widehat G_{2,k,i,j}
\widehat M_{2,i,j}^{-1}
\widehat R_{1,i,j}\widehat R_{1,i,j}^\top
\right).
\label{eq:empirical-transition-score-bias}
\end{align}

Let
$\widehat{\bm\delta}_{l,i,j}
:=(\widehat\delta_{l,i,j}^{(1)},\ldots,
\widehat\delta_{l,i,j}^{(K)})^\top$.
Averaging over the $n$ independent trajectories while retaining all $T$ within-trajectory transitions gives
\begin{equation}
\widehat{\bm\delta}_n
:=\frac1n\sum_{j=1}^{n}\sum_{i=0}^{T-1}
\widehat{\bm\delta}_{l,i,j}.
\label{eq:empirical-score-bias}
\end{equation}
Applying the one-step curvature correction to $\widehat{\bm\delta}_n$ yields the bias estimator $\widehat{\bm b}_n$ defined in Equation~\eqref{eq:empirical-parameter-bias}.

\section{Adjoint Sensitivity Derivations and Proofs}
\label{app:adjoint}

This section derives the backward adjoint equations and proves Proposition~\ref{prop:shared-adjoint-sensitivities} on each intervention-free interval, then composes the sensitivities across decision times for finite-horizon returns. On each interval $[u,v]$, the forward and backward dynamics use the same Brownian realization $\xi$.

\subsection{Regularity Conditions for Adjoint Sensitivities}
\label{app:adjoint-regularity}
The results below use three levels of regularity. First, on each finite intervention-free interval, assume the coefficients of the augmented state equation generate a nonexplosive stochastic flow $\mathcal F_{u,t}$ that is almost surely twice continuously differentiable in $(\pmb s,\bm\theta)$. At fixed $\bm\theta$, its state Jacobian is invertible and its inverse state flow is twice continuously differentiable; the coefficient derivatives in the backward equations exist and the Stratonovich chain and inverse rules apply. Assume also that the terminal function $\rho$ is twice continuously differentiable on the reachable state domain. These are the pathwise-flow conditions used in Proposition~\ref{prop:shared-adjoint-sensitivities}; standard coefficient-level smoothness, growth, and moment conditions that ensure such a flow are given in \citep{kunita2019flowsjump}.

Second, when interval sensitivities are composed across decision times, assume the intervention map $\pmb f$, reward $r$, and policy $\pi$ are twice continuously differentiable on the encountered state--action domain. Finally, to identify derivatives of $J$ or $Q_i^\pi$ with expectations of pathwise derivatives, assume that the first- and second-order derivatives of the corresponding pathwise return, uniformly over a neighborhood of the parameter of interest, are dominated by integrable random variables.

\subsection{Forward and Inverse Flows}
\label{app:adjoint-flows}

\begin{sloppypar}
Let $R$ denote the number of Brownian components, let $\pmb\sigma_k$ be the $k$th column of the diffusion matrix $\pmb\sigma$, and define the Stratonovich drift
\[
\pmb\mu^{\mathrm{Strat}}(\pmb s,\bm\theta)
:=\pmb\mu(\pmb s,\bm\theta)
-\frac12\sum_{k=1}^{R}
D_s\pmb\sigma_k(\pmb s,\bm\theta)\,
\pmb\sigma_k(\pmb s,\bm\theta).
\]
The forward state flow satisfies
\begin{equation}
\begin{aligned}
\Phi_{u,v}(\pmb x;\bm\theta)
={}&\pmb x
+\int_u^v
\pmb\mu^{\mathrm{Strat}}(\Phi_{u,t},\bm\theta)\,dt+\sum_{k=1}^{R}\int_u^v
\pmb\sigma_k(\Phi_{u,t},\bm\theta)\circ dW_t^k.
\end{aligned}
\label{app:eq:forward-flow}
\end{equation}
Under the pathwise-flow conditions in Subsection~\ref{app:adjoint-regularity}, the flow is almost surely a diffeomorphism and satisfies the composition property
\[
\Phi_{u,w}(\pmb x;\bm\theta,\xi)
=
\Phi_{v,w}\!\left(
\Phi_{u,v}(\pmb x;\bm\theta,\xi);
\bm\theta,\xi
\right),
\qquad u\leq v\leq w.
\]
\end{sloppypar}

For a terminal state $\pmb y$, define the inverse flow
\[\psi_{u,v}(\pmb y;\bm\theta,\xi)
:=\Phi_{u,v}(\cdot;\bm\theta,\xi)^{-1}(\pmb y).
\]
Using the same Brownian realization as the forward flow, $\psi_{u,v}$ satisfies the backward Stratonovich equation
\begin{equation}
\begin{aligned}
\psi_{u,v}(\pmb y;\bm\theta)
=\pmb y
-\int_u^v
\pmb\mu^{\mathrm{Strat}}(\psi_{t,v},\bm\theta)\,dt
-\sum_{k=1}^{R}\int_u^v
\pmb\sigma_k(\psi_{t,v},\bm\theta)
\circ d\widetilde W_t^k,
\end{aligned}
\label{app:eq:inverse-flow}
\end{equation}
where $\widetilde W_t=W_t-W_v$ and the stochastic integrals are backward Stratonovich integrals with respect to the backward filtration generated by increments on $[t,v]$. The map $\psi_{u,v}$ uses the same Brownian path as $\Phi_{u,v}$.

\subsection{State and Parameter Sensitivities}
\label{app:adjoint-state-parameter}

Use Jacobians whose rows index output coordinates and define the state and parameter sensitivity matrices
\[
{ F_{u,t}:=D_s\Phi_{u,t}(\pmb s;\bm\theta)
\in\mathbb R^{M\times M},
\qquad
P_{u,t}:=D_\theta\Phi_{u,t}(\pmb s;\bm\theta)
\in\mathbb R^{M\times K}. }
\]
For a coefficient $v$, let $v_s:=D_sv$ and $v_\theta:=D_\theta v$. In the forward sensitivity equations, these derivatives are evaluated at $(\Phi_{u,t}(\pmb s;\bm\theta),\bm\theta)$. Differentiating Equation~\eqref{app:eq:forward-flow} yields
\begin{equation}
\begin{aligned}
F_{u,v}
&=I_M
+\int_u^v\mu_s^{\mathrm{Strat}}(t)F_{u,t}\,dt
+\sum_{k=1}^{R}\int_u^v
\sigma_{k,s}(t)F_{u,t}\circ dW_t^k,\\
P_{u,v}
&=\int_u^v
\left[\mu_s^{\mathrm{Strat}}(t)P_{u,t}
+\mu_\theta^{\mathrm{Strat}}(t)\right]dt+\sum_{k=1}^{R}\int_u^v
\left[\sigma_{k,s}(t)P_{u,t}
+\sigma_{k,\theta}(t)\right]\circ dW_t^k.
\end{aligned}
\label{app:eq:forward-sensitivity-flow}
\end{equation}
The initial conditions are $F_{u,u}=I_M$ and $P_{u,u}=0$. Although $\bm\theta$ is constant in time, $P_{u,t}$ is generally nonzero due to the forcing terms $\mu_\theta^{\mathrm{Strat}}$ and $\sigma_{k,\theta}$.

For the inverse flow $\psi_{u,v}$, define
\[
J_{u,v}(\pmb y;\bm\theta)
:=D_y\psi_{u,v}(\pmb y;\bm\theta),
\qquad
K_{u,v}(\pmb y;\bm\theta):=J_{u,v}(\pmb y;\bm\theta)^{-1}.
\]
In the backward equations below, coefficient derivatives are evaluated at $(\psi_{t,v}(\pmb y;\bm\theta),\bm\theta)$. Differentiating Equation~\eqref{app:eq:inverse-flow} with respect to the terminal state gives
\begin{equation}
J_{u,v}
=I_M
-\int_u^v\mu_s^{\mathrm{Strat}}(t)J_{t,v}\,dt
-\sum_{k=1}^{R}\int_u^v
\sigma_{k,s}(t)J_{t,v}\circ d\widetilde W_t^k.
\label{app:eq:inverse-jacobian}
\end{equation}

The inverse-function identity and the Stratonovich matrix-inverse rule give
\begin{equation}
K_{u,v}
=
D_x\Phi_{u,v}\!\left(
\psi_{u,v}(\pmb y;\bm\theta);\bm\theta
\right),
\qquad
d(J^{-1})=-J^{-1}(dJ)J^{-1}.
\label{app:eq:inverse-function-matrix-rule}
\end{equation}
Since $J_{t,v}K_{t,v}=I_M$, Equations~\eqref{app:eq:inverse-jacobian}
and \eqref{app:eq:inverse-function-matrix-rule} yield
\begin{equation}
K_{u,v}
=I_M
+\int_u^vK_{t,v}\mu_s^{\mathrm{Strat}}(t)\,dt
+\sum_{k=1}^{R}\int_u^v
K_{t,v}\sigma_{k,s}(t)\circ d\widetilde W_t^k.
\label{app:eq:inverse-jacobian-inverse}
\end{equation}

\subsection{Pathwise First- and Second-Order Adjoints}
\label{app:adjoint-pathwise}

For a terminal augmented state $z_v=(\pmb y,\bm\theta)$, define
\[
\mathcal F_{t,v}^{-1}(z_v,\xi)
:=(\psi_{t,v}(\pmb y;\bm\theta,\xi),\bm\theta),
\]
and the endpoint-parametrized adjoint fields
\begin{equation}
\begin{aligned}
\widetilde p_{t,v}(z_v,\xi)
&:=p_{t,v}(\mathcal F_{t,v}^{-1}(z_v,\xi),\xi),\\
\widetilde{\mathscr H}_{t,v}(z_v,\xi)
&:=\mathscr H_{t,v}(\mathcal F_{t,v}^{-1}(z_v,\xi),\xi).
\end{aligned}
\label{app:eq:backward-field-definitions}
\end{equation}
Here $p$ and $\mathscr H$ denote the forward first- and second-order sensitivity fields in Equation~\eqref{eq:augmented-pathwise-sensitivities}. In Equation~\eqref{eq:shared-adjoint-composition}, the backward fields are evaluated at the realized endpoint $\mathcal F_{u,v}(z,\xi)$, whereas the forward derivatives are evaluated at the initial state $z$.

For the scalar functional $\rho$ and augmented state $z=(\pmb s,\bm\theta)$
defined in Section~\ref{sec:adjoint-engine}, the forward chain rule yields the row-vector sensitivity blocks of $D_z\rho(\Phi_{u,v}(\pmb s;\bm\theta))$:
\begin{equation}
\begin{aligned}
A^s_{u,v}(\pmb x;\bm\theta,\xi)
&=D_s\rho(\Phi_{u,v})F_{u,v},\\
A^\theta_{u,v}(\pmb x;\bm\theta,\xi)
&=D_s\rho(\Phi_{u,v})P_{u,v}.
\end{aligned}
\label{app:eq:pathwise-sensitivities}
\end{equation}
At a fixed endpoint $\pmb y$, the inverse-function theorem gives
$\widetilde A^s_{u,v}(\pmb y;\bm\theta)
=D_s\rho(\pmb y)K_{u,v}$.
Note that $\widetilde A^\theta$ is not the total parameter derivative of $\rho(\Phi_{u,v}(\psi_{u,v}(\pmb y)))=\rho(\pmb y)$; it represents only the parameter-sensitivity block of the adjoint. Evaluating coefficient derivatives along $\psi_{t,v}$, the adjoint blocks satisfy
\begin{equation}
\begin{aligned}
\widetilde A^s_{u,v}
={}&D_s\rho(\pmb y)
+\int_u^v\widetilde A^s_{t,v}\mu_s^{\mathrm{Strat}}(t)\,dt
+\sum_{k=1}^{R}\int_u^v
\widetilde A^s_{t,v}\sigma_{k,s}(t)\circ d\widetilde W_t^k,\\
\widetilde A^\theta_{u,v}
={}&\int_u^v\widetilde A^s_{t,v}\mu_\theta^{\mathrm{Strat}}(t)\,dt
+\sum_{k=1}^{R}\int_u^v
\widetilde A^s_{t,v}\sigma_{k,\theta}(t)\circ d\widetilde W_t^k.
\end{aligned}
\label{gradbflow_wA}
\end{equation}

Let $v_0=\pmb\mu^{\mathrm{Strat}}$ and $v_k=\pmb\sigma_k$, and define
the augmented vector fields for $\mathcal F$ by
$V_q(z):=(v_q(\pmb s,\bm\theta),\bm0)$, $q=0,\ldots,R$.
For a fixed endpoint $(\pmb y,\bm\theta)$, write
$z_t:=\mathcal F_{t,v}^{-1}((\pmb y,\bm\theta))
=(\psi_{t,v}(\pmb y;\bm\theta),\bm\theta)$ in this backward-flow calculation.
Multiplying Equation~\eqref{app:eq:inverse-jacobian-inverse} by
$D_s\rho(\pmb y)$ gives the state block of
Equation~\eqref{gradbflow_wA}. For these augmented vector fields, the
coefficient Jacobian has block form
\begin{equation}
D_zV_q(z_t)
=\begin{pmatrix}
D_s v_q(z_t)&D_\theta v_q(z_t)\\
0&0
\end{pmatrix},
\qquad q=0,\ldots,R.
\label{app:eq:augmented-field-jacobian}
\end{equation}
Thus
\begin{equation}
(\widetilde A^s_{t,v},\widetilde A^\theta_{t,v})D_zV_q
=(\widetilde A^s_{t,v}D_s v_q,
\widetilde A^s_{t,v}D_\theta v_q).
\label{app:eq:adjoint-block-product}
\end{equation}
The forward sensitivities in
Equation~\eqref{app:eq:pathwise-sensitivities} are recovered by
composing the backward fields in Equation~\eqref{gradbflow_wA}
with the realized endpoint. Equation~\eqref{app:eq:adjoint-block-product}
gives the two coefficient blocks in the adjoint equations.

For the second-order adjoint, let
$p_{u,v}$ and $\mathscr H_{u,v}$ denote the pathwise gradient and Hessian in Equation~\eqref{eq:augmented-pathwise-sensitivities}. Define the
endpoint-parametrized adjoints $\widetilde p_{t,v}$ and
$\widetilde{\mathscr H}_{t,v}$ from
Equation~\eqref{app:eq:backward-field-definitions}.
With $d=M+K$, define the symmetric matrix operator
\begin{equation}
\mathscr B_q(z;p,H)
:=
D_zV_q(z)^\top H+H D_zV_q(z)
+\sum_{a=1}^{d}p_aD_z^2V_{q,a}(z).
\label{app:eq:second-adjoint-operator}
\end{equation}
Only the first $M$ components of $V_q$ contribute to the final sum.
In particular, if $H$ is partitioned into state and parameter blocks,
the parameter--parameter block of this operator is
\[
[\mathscr B_q(z;p,H)]_{\theta\theta}
=
(D_\theta v_q)^\top H_{s\theta}
+H_{\theta s}D_\theta v_q
+\sum_{a=1}^{M}p_{s,a}D_{\theta\theta}^2v_{q,a},
\]
showing that the parameter Hessian generally depends on the mixed blocks $H_{s\theta}$ and $H_{\theta s}$.

The coupled backward Stratonovich equations are
\begin{equation}
\begin{aligned}
\widetilde p_{u,v}
={}&\begin{pmatrix}\nabla_s\rho(\pmb y)\\\bm0\end{pmatrix}
+\int_u^v D_zV_0(z_t)^\top\widetilde p_{t,v}\,dt
+\sum_{k=1}^{R}\int_u^v
D_zV_k(z_t)^\top\widetilde p_{t,v}
\circ d\widetilde W_t^k,\\
\widetilde{\mathscr H}_{u,v}
={}&\operatorname{diag}(\nabla_s^2\rho(\pmb y),0_{K\times K})
+\int_u^v
\mathscr B_0(z_t;\widetilde p_{t,v},
\widetilde{\mathscr H}_{t,v})\,dt
+\sum_{k=1}^{R}\int_u^v
\mathscr B_k(z_t;\widetilde p_{t,v},
\widetilde{\mathscr H}_{t,v})
\circ d\widetilde W_t^k.
\end{aligned}
\label{app:eq:backward-second-adjoint}
\end{equation}
To derive Equation~\eqref{app:eq:backward-second-adjoint}, use the flow
composition
$\mathcal F_{u,v}=\mathcal F_{t,v}\circ\mathcal F_{u,t}$.
For a smooth scalar $h$ and vector map $G$, the second-order chain rule is
\begin{equation}
D_z^2(h\circ G)
=D_zG^\top D_z^2h(G)D_zG
+\sum_{a=1}^{d}\partial_a h(G)D_z^2G_a.
\label{app:eq:second-order-chain-rule}
\end{equation}
Apply it with $h(z)=\rho(\Phi_{t,v}(\pmb s;\bm\theta))$ and
$G=\mathcal F_{u,t}$. For this forward-flow calculation, write
$\zeta_t:=\mathcal F_{u,t}(z_u)$ and let
$\Gamma_{u,t}:=D_z\mathcal F_{u,t}(z_u)$ and
$\Xi_{u,t}^{(a)}:=D_z^2\mathcal F_{u,t}^{,a}(z_u)$.
Flow composition and Equation~\eqref{app:eq:second-order-chain-rule} give the
exact identities
\begin{equation}
\begin{aligned}
p_{u,v}(z_u)
&=\Gamma_{u,t}^\top p_{t,v}(\zeta_t),\\
\mathscr H_{u,v}(z_u)
&=\Gamma_{u,t}^\top\mathscr H_{t,v}(\zeta_t)\Gamma_{u,t}
+\sum_{a=1}^{d}p_{t,v,a}(\zeta_t)\Xi_{u,t}^{(a)}.
\end{aligned}
\label{app:eq:second-adjoint-chain-rule}
\end{equation}
The Stratonovich variational equations for the two flow derivatives are
\begin{equation}
\begin{aligned}
d\Gamma_{u,t}
&=D_zV_0(\zeta_t)\Gamma_{u,t}\,dt
+\sum_{k=1}^{R}D_zV_k(\zeta_t)\Gamma_{u,t}\circ dW_t^k,\\
d\Xi_{u,t}^{(a)}
&=\left\{\sum_{b=1}^{d}\partial_bV_{0,a}(\zeta_t)
\Xi_{u,t}^{(b)}
+\Gamma_{u,t}^\top D_z^2V_{0,a}(\zeta_t)\Gamma_{u,t}\right\}dt\\
&+\sum_{k=1}^{R}\left\{
\sum_{b=1}^{d}\partial_bV_{k,a}(\zeta_t)
\Xi_{u,t}^{(b)}
+\Gamma_{u,t}^\top D_z^2V_{k,a}(\zeta_t)\Gamma_{u,t}
\right\}\circ dW_t^k.
\end{aligned}
\label{app:eq:second-order-variational-flow}
\end{equation}
Compose Equations~\eqref{app:eq:second-adjoint-chain-rule} and
\eqref{app:eq:second-order-variational-flow} with the
map $\mathcal F_{t,v}^{-1}$ at the fixed endpoint $(\pmb y,\bm\theta)$. The Stratonovich chain
rule gives $D_zV_q^\top H+H D_zV_q$ from the two Jacobian factors
and $\sum_a p_aD_z^2V_{q,a}$ from the second flow derivative. This
is exactly Equation~\eqref{app:eq:second-adjoint-operator} for the
drift and each diffusion vector field, proving
Equation~\eqref{app:eq:backward-second-adjoint}. At the realized endpoint
$z_v=\mathcal F_{u,v}(z_u,\xi)$, the backward fields therefore satisfy
\[
\begin{aligned}
\widetilde p_{u,v}(z_v,\xi)
&=\left.\nabla_z[\rho(\Phi_{u,v}(\pmb s;\bm\theta,\xi))]\right|_{z=z_u},\\
\widetilde{\mathscr H}_{u,v}(z_v,\xi)
&=\left.D_z^2[\rho(\Phi_{u,v}(\pmb s;\bm\theta,\xi))]\right|_{z=z_u},
\end{aligned}
\]proving Proposition~\ref{prop:shared-adjoint-sensitivities}.

The $\theta$ block of $p_{u,v}$ is the transpose of
$A^\theta_{u,v}$ in Equation~\eqref{app:eq:pathwise-sensitivities}.
With $P_{u,v}:=D_\theta\Phi_{u,v}$, the ordinary second-order chain rule gives
\begin{equation}
\begin{aligned}
\mathcal H_{u,v}
:=D_\theta^2[\rho(\Phi_{u,v})]
=P_{u,v}^\top\nabla_s^2\rho(\Phi_{u,v})P_{u,v}
+\sum_{j=1}^{M}
\frac{\partial\rho}{\partial s_j}(\Phi_{u,v})
D_\theta^2\Phi_{u,v}^{\,j}.
\end{aligned}
\label{app:eq:pathwise-hessian}
\end{equation}
Thus $[\mathscr H_{u,v}]_{\theta\theta}=\mathcal H_{u,v}$.


\subsection{Finite-Horizon and Expected Sensitivities}
\label{app:adjoint-finite-horizon}

For a fixed state--action pair $(\pmb s,\pmb a)$ at decision time $t_i$, define the conditional pathwise return
\begin{equation}
\mathcal R_i^\pi(\pmb s,\pmb a;\bm\theta,\xi)
:=\sum_{k=i}^{T-1}\gamma^{k-i}
r(\pmb s_{t_k}^{(i)},\pmb a_{t_k}^{(i)}),
\label{app:eq:conditional-pathwise-return}
\end{equation}
where $(\pmb s_{t_i}^{(i)},\pmb a_{t_i}^{(i)})=(\pmb s,\pmb a)$ and $\pmb a_{t_k}^{(i)}=\pi(\pmb s_{t_k}^{(i)})$ for $k>i$. 
For the finite-horizon return, compose interval flows with the
deterministic post-decision maps. On the augmented state, define
\[
\mathcal D_i(\pmb s,\bm\theta)
:=
(\pmb f(\pmb s,\pi(\pmb s)),\bm\theta),
\qquad
c_i(\pmb s,\bm\theta)
:=\gamma^i r(\pmb s,\pi(\pmb s)).
\]
Let $L_i^+(z^+)$ denote the remaining pathwise return from the
post-decision augmented state $z^+$, and define
\[
p^+:=\left.\nabla_{z^+}L_i^+(z^+)\right|_{z^+=\mathcal D_i(z)},
\qquad
H^+:=\left.D_{z^+}^2L_i^+(z^+)\right|_{z^+=\mathcal D_i(z)}.
\]
Applying the first- and second-order chain rules to $c_i(z)+L_i^+(\mathcal D_i(z))$ gives
\begin{equation}
\begin{aligned}
p^-&=\nabla_z c_i+D_z\mathcal D_i^\top p^+,\\
H^-&=D_z^2c_i
+D_z\mathcal D_i^\top H^+D_z\mathcal D_i
+\sum_{a=1}^{M+K}p_a^+D_z^2\mathcal D_{i,a}.
\end{aligned}
\label{app:eq:second-adjoint-decision-jump}
\end{equation}
At ordinary decision times, the state Jacobian of the decision map is $f_s+f_aD_s\pi$. For the conditional rollout defining $Q_i^\pi$, replace $\pi(\pmb s)$ at stage $i$ by the fixed action $\pmb a$ and use the discount powers in Equation~\eqref{eq:conditional-Q}; subsequent decisions follow $\pi$.

Applying Equations~\eqref{app:eq:backward-second-adjoint} and \eqref{app:eq:second-adjoint-decision-jump} across a realized trajectory gives the derivatives of the full return $\mathcal R^\pi$ and conditional return $\mathcal R_i^\pi$. 

Under the integrable-domination condition in Subsection~\ref{app:adjoint-regularity}, differentiation and expectation can be interchanged. For the single-interval terminal functional, this gives
\begin{equation}
\begin{aligned}
D_s\mathbb E_\xi[\rho(\Phi_{u,v})]
&=\mathbb E_\xi[A^s_{u,v}],\\
D_\theta\mathbb E_\xi[\rho(\Phi_{u,v})]
&=\mathbb E_\xi[A^\theta_{u,v}],\\
\nabla_\theta^2\mathbb E_\xi[\rho(\Phi_{u,v})]
&=\mathbb E_\xi[\mathcal H_{u,v}],
\end{aligned}
\label{app:eq:expected-adjoint-sensitivities}
\end{equation}
Applying the same interchange to the full return and averaging its
pathwise derivatives over $M_{\mathrm{MC}}$ independent rollouts gives
the Monte Carlo estimators:
\begin{equation}
\begin{aligned}
\widehat g_n^\pi
&:=\frac1{M_{\mathrm{MC}}}\sum_{r=1}^{M_{\mathrm{MC}}}
\nabla_\theta\mathcal R^\pi(\widehat{\bm\theta}_n,\xi_r),\\
\widehat H_n^\pi
&:=\frac1{M_{\mathrm{MC}}}\sum_{r=1}^{M_{\mathrm{MC}}}
D_\theta^2\mathcal R^\pi(\widehat{\bm\theta}_n,\xi_r).
\end{aligned}
\label{app:eq:full-return-MC}
\end{equation}
The corresponding conditional-return estimators are
\begin{equation}
\begin{aligned}
\widehat{\bm g}_{n,i}^\pi(\pmb s,\pmb a)
&:=\frac1{M_{\mathrm{MC}}}\sum_{r=1}^{M_{\mathrm{MC}}}
\nabla_\theta\mathcal R_i^\pi
(\pmb s,\pmb a;\widehat{\bm\theta}_n,\xi_r),\\
\widehat H_{n,i}^\pi(\pmb s,\pmb a)
&:=\frac1{M_{\mathrm{MC}}}\sum_{r=1}^{M_{\mathrm{MC}}}
D_\theta^2\mathcal R_i^\pi
(\pmb s,\pmb a;\widehat{\bm\theta}_n,\xi_r).
\end{aligned}
\label{app:eq:conditional-MC}
\end{equation}


\subsection{Proof of Proposition~\ref{prop:adjoint-size}}
\label{app:adjoint-size-proof}

Fix a Brownian realization and consider the arithmetic required to propagate sensitivities through a single numerical SDE step. For each vector field $v_q$, $q=0,\ldots,R$, define
$A_q:=D_s v_q\in\mathbb R^{M\times M}$ and
$B_q:=D_\theta v_q\in\mathbb R^{M\times K}$.
The coefficient derivatives and the forward state are treated as
available inputs.

For first-order forward sensitivities,
$P_t:=D_\theta\pmb s_t\in\mathbb R^{M\times K}$ satisfies the update
$A_qP_t+B_q$ in Equation~\eqref{app:eq:forward-sensitivity-flow}.
The matrix product $A_qP_t$ requires $\mathcal O(M^2K)$ operations per
field. In contrast, the backward adjoint update in
Equation~\eqref{gradbflow_wA} involves the vector--Jacobian products
$(\widetilde A^s A_q,\widetilde A^s B_q)$ from
Equation~\eqref{gradbflow_wA}, which require only $\mathcal O(M^2+MK)$ operations.
Moreover, all $K$ parameter-gradient components are obtained in the same backward sweep. The adjoint pair has $M+K$ entries, whereas the
forward parameter-sensitivity matrix $P_t$ has $MK$ entries; this count
excludes storage or replay of the forward path.

For second-order forward sensitivities, let
$\tau_t^{(a)}:=D_\theta^2s_t^a\in\mathbb R^{K\times K}$,
$a=1,\ldots,M$, and define
$U_t:=\bigl(P_t^\top,I_K\bigr)^\top\in\mathbb R^{(M+K)\times K}$.
Let $\mathcal T_t\in\mathbb R^{M\times K^2}$ have row $a$ equal to
$\operatorname{vec}(\tau_t^{(a)})^\top$.
The second-order variational update contains the terms $A_q\mathcal T_t$ and the $M$ contractions
$U_t^\top D_z^2v_{q,a}U_t$. 
Their per-field costs satisfy
\begin{equation}
\mathcal O(M^2K^2)+\mathcal O\bigl(MK(M+K)^2\bigr)
=\mathcal O\bigl(MK(M+K)^2\bigr).
\label{app:eq:forward-hessian-cost}
\end{equation}

For the backward second-order adjoint, $\widetilde{\mathscr H}_t\in\mathbb R^{(M+K)\times(M+K)}$ is propagated through the operator $\mathscr B_q$ in Equation~\eqref{app:eq:second-adjoint-operator}. Since $D_zV_q$ has only its first $M$ rows potentially nonzero, each Hessian--Jacobian product \[ D_zV_q^\top \widetilde{\mathscr H}_t, \qquad \widetilde{\mathscr H}_t D_zV_q, \] requires $\mathcal O(M(M+K)^2)$ operations. The curvature term \[ \sum_{a=1}^{M} p_a D_z^2V_{q,a} \] has the same order. Hence each backward second-order update costs \[ \mathcal O\!\left( M(M+K)^2 \right) \] per vector field.

Summing the first- and second-order costs over the $R+1$ vector fields,
including Equation~\eqref{app:eq:forward-hessian-cost}, gives the
per-step bounds in Proposition~\ref{prop:adjoint-size}.


\section{Gaussian Calibration-Error Approximation}
\label{app:gaussian-error}

Equation~\eqref{eq:gaussian-error-surrogate} specifies
$\widetilde{\bm e}_n\mid\mathcal D_n
\sim N(\widehat{\bm b}_n,\widehat\Sigma_n)$.
It can be simulated as
\begin{equation}
\widetilde{\bm e}_n
=\widehat{\bm b}_n+\widehat\Sigma_n^{1/2}Z,
\qquad Z\sim N(\bm0,I_K),
\label{app:eq:gaussian-error-draw}
\end{equation}
where $Z$ is independent of $\mathcal D_n$ and
$\widehat\Sigma_n^{1/2}$ is any positive-semidefinite square root.
To approximate Equation~\eqref{eq:implementable-uncertainty}, use the
derivatives in Equation~\eqref{app:eq:conditional-MC} and
$M_{\mathrm{err}}$ independent draws $\widetilde{\bm e}_{n,m}$ from
Equation~\eqref{app:eq:gaussian-error-draw}. The resulting Monte Carlo
score is
\begin{equation}
\begin{aligned}
\widehat u_{n,i}^{\pi,\mathrm{MC}}(\pmb s,\pmb a)
:=\frac1{M_{\mathrm{err}}}
\sum_{m=1}^{M_{\mathrm{err}}}\Bigl|
\bigl(\widehat{\bm g}_{n,i}^\pi(\pmb s,\pmb a)\bigr)^\top
\widetilde{\bm e}_{n,m} +\frac12\widetilde{\bm e}_{n,m}^\top
\widehat H_{n,i}^\pi(\pmb s,\pmb a)
\widetilde{\bm e}_{n,m}\Bigr|.
\end{aligned}
\label{eq:uncertainty-MC}
\end{equation}

For fixed gradient $\bm g$ and symmetric Hessian $H$, let
$q=\bm g^\top\widetilde{\bm e}_n
+\tfrac12\widetilde{\bm e}_n^\top H\widetilde{\bm e}_n$.
The Gaussian moment identities \citep{magnus1978moments} give
\begin{equation}
\begin{aligned}
\mathbb E[q\mid\mathcal D_n]
&=\bm g^\top\widehat{\bm b}_n
+\tfrac12\widehat{\bm b}_n^\top H\widehat{\bm b}_n
+\tfrac12\operatorname{tr}(H\widehat\Sigma_n),\\
\operatorname{Var}(q\mid\mathcal D_n)
&=(\bm g+H\widehat{\bm b}_n)^\top\widehat\Sigma_n
(\bm g+H\widehat{\bm b}_n)
+\tfrac12\operatorname{tr}(H\widehat\Sigma_nH\widehat\Sigma_n).
\end{aligned}
\label{app:eq:gaussian-quadratic-moments}
\end{equation}
Exploration uses $\mathbb E[|q|\mid\mathcal D_n]$, as in
Equation~\eqref{eq:implementable-uncertainty}.

For the full return, set
$\bm g=\nabla_\theta J(\pi;\widehat{\bm\theta}_n)$ and
$H=\nabla_\theta^2J(\pi;\widehat{\bm\theta}_n)$.
If $J(\pi;\widehat{\bm\theta}_n+\widetilde{\bm e}_n)$ is integrable,
Taylor's theorem gives
\begin{equation}
\begin{aligned}
J_n^{\mathrm G}(\pi)
&=J(\pi;\widehat{\bm\theta}_n)
+\mathbb E[q\mid\mathcal D_n]
+\mathbb E[R_{3,n}^{\pi}(\widetilde{\bm e}_n)
\mid\mathcal D_n],\\
R_{3,n}^{\pi}(\bm e)
&:=J(\pi;\widehat{\bm\theta}_n+\bm e)
-J(\pi;\widehat{\bm\theta}_n)
-\bm g^\top\bm e-\tfrac12\bm e^\top H\bm e.
\end{aligned}
\label{app:eq:gaussian-return-taylor}
\end{equation}
Equation~\eqref{eq.policyOptimizationObjective} retains the first two
terms of Equation~\eqref{app:eq:gaussian-return-taylor}. Thus, if the
third derivatives of $J$ are uniformly bounded along the stated line
segments,
\begin{equation}
\begin{aligned}
|J_n^{\mathrm G}(\pi)-\bar J_n(\pi)|
&=\left|\mathbb E[R_{3,n}^{\pi}(\widetilde{\bm e}_n)
\mid\mathcal D_n]\right|
\le C\,\mathbb E[\|\widetilde{\bm e}_n\|^3
\mid\mathcal D_n].
\end{aligned}
\label{app:eq:gaussian-return-remainder}
\end{equation}


\section{Actor--Simulator Algorithm}
\label{sec:algorithm}

Algorithm~\ref{actorsimulatoralgo} combines trajectory acquisition, calibration, and Gaussian-averaged target-policy optimization using the shared adjoint sensitivities.

\refstepcounter{breakablealgorithm}\label{actorsimulatoralgo}
\begingroup
\makeatletter
\def\@captype{figure}
\@makecaption{Algorithm~\thebreakablealgorithm}{\textbf{Bias-Aware Actor--Simulator for Digital-Twin Calibration and Policy Optimization.} Each iteration uses trajectory-level adjoint sensitivities to guide a physical experiment, updates the calibrated model and its uncertainty, and optimizes the next target policy.}
\vspace{-0.16in}
\makeatother
\endgroup
\par\smallskip\noindent\rule{\linewidth}{0.4pt}

\algrenewcommand{\algorithmicrequire}{\textbf{Input:}}
\algrenewcommand{\algorithmicensure}{\textbf{Output:}}
\begin{algorithmic}[1]
\Require Trajectory budget $N_{\max}$, pilot size $n_0$, reward $r$, moment order $l$, observation interval $\Delta t$, simulation step $h_{\rm sim}$, sensitivity sample size $M_{\rm MC}$, and calibration-error sample size $M_{\rm err}$.

\Ensure Calibrated parameters $\widehat{\bm\theta}_{N_{\max}}$ and target policy $\widehat\pi_{N_{\max}}$.
\State Collect $n_0$ complete pilot trajectories to form $\mathcal D_{n_0}$ and fit $\widehat{\bm\theta}_{n_0}$.
\State Compute $\widehat C_{l,n_0}$ by Equation~\eqref{eq:uncertainty-covariance}; set $\widehat\Sigma_{n_0}=\widehat C_{l,n_0}/n_0$.
\State Compute $\widehat{\bm\delta}_{n_0}$ by Equation~\eqref{eq:empirical-score-bias}; set $\widehat{\bm b}_{n_0}=\widehat\Lambda_{l,n_0}^{-1}\widehat{\bm\delta}_{n_0}$.
\State Set $\widehat\pi_{n_0}\in\arg\max_{\pi\in\Pi}\bar J_{n_0}(\pi)$.
\For{$n=n_0+1,\ldots,N_{\max}$}
    \Statex \hspace*{\algorithmicindent}\textbf{\makebox[2.4em][l]{(i)}Trajectory-Level Adjoint Sensitivity Analysis}
    \State On each simulated trajectory, propagate the stochastic adjoints $\widetilde p$ and $\widetilde{\mathscr H}$ backward to evaluate the gradient and Hessian estimates required in Steps~(ii) and~(iv).

    \Statex \textbf{\makebox[2.4em][l]{(ii)}Policy-Relevant Digital Twin Calibration}
    \State Generate $\pmb s_{t_0}^n\sim p_0$.
    \State Draw $M_{\mathrm{err}}$ independent calibration-error samples from $N(\widehat{\bm b}_{n-1},\widehat\Sigma_{n-1})$ as in Equation~\eqref{eq:uncertainty-MC}.
    \For{$i=0,\ldots,T-1$}
        \State For each candidate action $\pmb a$, perform conditional rollouts from $(\pmb s_{t_i}^n,\pmb a)$ and follow $\widehat\pi_{n-1}$ thereafter.
        \State Use Step~(i) to evaluate $\widehat{\bm g}_{n-1,i}^{\widehat\pi_{n-1}}$ and $\widehat H_{n-1,i}^{\widehat\pi_{n-1}}$ by Equation~\eqref{app:eq:conditional-MC}.
        \State Compute $\widehat u_{n-1,i}^{\widehat\pi_{n-1},\mathrm{MC}}$ by Equation~\eqref{eq:uncertainty-MC}.
        \State Select and execute the action:
        $\displaystyle
        \begin{aligned}[t]
        \pmb a_{t_i}^n&=\eta_{n,i}(\pmb s_{t_i}^n)\in\arg\max_{\pmb a\in\mathcal A}
        \widehat u_{n-1,i}^{\widehat\pi_{n-1},\mathrm{MC}}(\pmb s_{t_i}^n,\pmb a).
        \end{aligned}$
    \EndFor
    \State Append the resulting trajectory $\pmb\tau_n$ to update $\mathcal D_n=\mathcal D_{n-1}\cup\{\pmb\tau_n\}$.
    \Statex \textbf{\makebox[2.4em][l]{(iii)}Bias-Aware Model Estimation}
    \State Estimate model parameters: $\displaystyle
    \widehat{\bm\theta}_n
    \in\arg\max_{\bm\theta\in\Theta}\frac1n\sum_{j=1}^n\ell_l(\pmb\tau_j;\bm\theta).$
    \State Update $\widehat C_{l,n}$ using Equation~\eqref{eq:uncertainty-covariance} and set $\widehat\Sigma_n=\widehat C_{l,n}/n$.
    \State Implement Equation~\eqref{eq:empirical-parameter-bias}: evaluate $\widehat{\bm\delta}_n$ using Equation~\eqref{eq:empirical-score-bias}, then set $\widehat{\bm b}_n=\widehat\Lambda_{l,n}^{-1}\widehat{\bm\delta}_n$.
    \Statex \textbf{\makebox[2.4em][l]{(iv)}Bias-Corrected Policy Optimization}
    \State For each candidate policy $\pi$, use Step~(i) to evaluate $\widehat g_n^\pi$ and $\widehat H_n^\pi$ by Equation~\eqref{app:eq:full-return-MC}.
    \State Approximate the objective by
    \Statex $\displaystyle
    \begin{aligned}[t]
    \bar J_n(\pi)\approx {}&J(\pi;\widehat{\bm\theta}_n)
    +(\widehat g_n^\pi)^\top\widehat{\bm b}_n+\frac12\widehat{\bm b}_n^\top\widehat H_n^\pi\widehat{\bm b}_n
    +\frac12\operatorname{tr}(\widehat H_n^\pi\widehat\Sigma_n).
    \end{aligned}$
    \State Update $\widehat\pi_n\in\arg\max_{\pi\in\Pi}\bar J_n(\pi)$.
\EndFor
\end{algorithmic}
\noindent\rule{\linewidth}{0.4pt}


\section{Dense-Sampling Information and QMLE Rate}
\label{app:dense-sampling-rate}

This section gives sufficient conditions for the sampling rate used
in Section~\ref{sec:convergence}. Keep the physical duration
$t_T-t_0=T\Delta t$ fixed while $n,T\to\infty$, and let
$x_{i,j}:=\pmb s_{t_i^+}^j$ and $Y_{i,j}:=\pmb s_{t_{i+1}}^j$.
All expectations below are under the true parameter and $\eta_T$. The parameter dimension and expansion order $l\geq1$ do not
increase with $n$ or $T$.

For comparison with the order-$l$ criterion, use the exact moments
$\pmb M_1$ and $\pmb M_2$ from
Equation~\eqref{eq:exact-conditional-moments} in the quasi-likelihood
of Equation~\eqref{eq:transition-quasilikelihood}; call the resulting
criterion $\ell_{0,i}^j(\bm\theta)$.
At $\bm\theta^\ast$, define
\begin{equation}
\begin{aligned}
Z_{i,j}&:=
\pmb M_2(\Delta t,x_{i,j},\bm\theta^\ast)^{-1/2}
\{Y_{i,j}-\pmb M_1(\Delta t,x_{i,j},\bm\theta^\ast)\},\\
c_{k,i,j}&:=
\pmb M_2^{-1/2}\partial_{\theta_k}\pmb M_1,
\qquad
D_{k,i,j}:=
\pmb M_2^{-1/2}
(\partial_{\theta_k}\pmb M_2)
\pmb M_2^{-1/2},
\end{aligned}
\label{app:eq:dense-standardization}
\end{equation}
where all moment quantities on the second line are evaluated at
$(\Delta t,x_{i,j},\bm\theta^\ast)$.
A direct differentiation of the exact-moment Gaussian criterion, followed by whitening through $\pmb M_2^{-1/2}$, yields
\begin{equation}
\partial_{\theta_k}\ell_{0,i}^j(\bm\theta^\ast)
=c_{k,i,j}^\top Z_{i,j}
+\frac12\{Z_{i,j}^\top D_{k,i,j}Z_{i,j}
-\operatorname{tr}(D_{k,i,j})\}.
\label{app:eq:dense-score-decomposition}
\end{equation}
The exact conditional moments give
\begin{equation}
\mathbb E[Z_{i,j}\mid x_{i,j}]=0,
\qquad
\mathbb E[Z_{i,j}Z_{i,j}^\top\mid x_{i,j}]=I.
\label{app:eq:dense-standardized-moments}
\end{equation}
Thus both terms in Equation~\eqref{app:eq:dense-score-decomposition}
have conditional mean zero given the state--action history through $x_{i,j}$.

We impose the following joint dense-sampling conditions. The diffusion covariance $\mathsf a(x;\bm\theta)$ is uniformly positive definite on the visited state--parameter domain. Uniformly in an integrated $L^4$ sense over this domain, the exact conditional moments and their first two parameter derivatives satisfy
\begin{equation}
\pmb M_1=x+\pmb\mu(x,\bm\theta)\Delta t
+\mathcal O((\Delta t)^2),\qquad
\pmb M_2=\mathsf a(x;\bm\theta)\Delta t
+\mathcal O((\Delta t)^2).
\label{app:eq:dense-moment-expansions}
\end{equation}
The same leading covariance expansion holds for the order-$l$ approximation in Equation~\eqref{eq:moment-expansion-summary}. Indeed, the generator identities \[ \mathcal L_{\bm\theta}\pmb s = \pmb\mu(\pmb s,\bm\theta), \qquad \mathcal L_{\bm\theta}(\pmb s\pmb s^\top) = \pmb s\pmb\mu^\top +\pmb\mu\pmb s^\top +\mathsf a(\pmb s;\bm\theta) \] imply that the drift cross terms cancel when the outer product of the approximate first moment is subtracted from the approximate raw second moment.
Consequently, \[ \widetilde{\pmb m}_{2,l} = \mathsf a(x;\bm\theta)\Delta t +\mathcal O((\Delta t)^2). \]

Uniform parameter-differentiability and positive definiteness give
\[
\partial_{\theta_k}\pmb M_2
=(\partial_{\theta_k}\mathsf a)\Delta t
+\mathcal O((\Delta t)^2),
\qquad
\pmb M_2^{-1}
=(\Delta t)^{-1}\mathsf a^{-1}+\mathcal O(1).
\]
Multiplication yields
\begin{equation}
\pmb M_2^{-1}\partial_{\theta_k}\pmb M_2
=\mathsf a^{-1}\partial_{\theta_k}\mathsf a
+\mathcal O(\Delta t).
\label{app:eq:dense-covariance-derivative}
\end{equation}

The conditional standardized third moment of $Z_{i,j}$ is $\mathcal O((\Delta t)^{1/2})$, its fourth moment equals the Gaussian fourth moment up to $\mathcal O(\Delta t)$, and the exact-score array satisfies a conditional Lindeberg condition. Moreover, \[ c_{k,i,j} = \pmb M_2^{-1/2}\partial_{\theta_k}\pmb M_1 = \mathcal O((\Delta t)^{1/2}), \] so the mean component in Equation~\eqref{app:eq:dense-score-decomposition} contributes only $\mathcal O(\Delta t)$ information per transition. Under the fixed-horizon scaling $T\Delta t=t_T-t_0$, this contribution vanishes after normalization by $nT$. 

Therefore, the leading dense-sampling information arises from the covariance component.
Define 
\[ \mathsf a_{i,j} := \mathsf a(x_{i,j};\bm\theta^\ast), \qquad \mathsf a_{i,j,k} := \left. \partial_{\theta_k}\mathsf a(x_{i,j};\bm\theta) \right|_{\bm\theta=\bm\theta^\ast}. 
\] 
The empirical average of the leading conditional information obeys,
for every $k,h$, \begin{equation} \frac{1}{2nT} \sum_{j=1}^{n} \sum_{i=0}^{T-1} \operatorname{tr}\!\left\{ \mathsf a_{i,j}^{-1}\mathsf a_{i,j,k}\, \mathsf a_{i,j}^{-1}\mathsf a_{i,j,h} \right\} \xrightarrow{p} \bigl[\mathcal I_\sigma(\bm\theta^\ast)\bigr]_{kh}. \label{app:eq:dense-information-lln} \end{equation} We further assume \[ \lambda_{\min}\{\mathcal I_\sigma(\bm\theta^\ast)\}>0. \]

To transfer the exact-moment analysis to the order-$l$ criterion, define
the score perturbation
\begin{equation}
\bm D_{i,j}^{(l)}
:=
\nabla_\theta\ell_{l,i}^j(\widetilde{\bm\theta}_l)
-\nabla_\theta\ell_{0,i}^j(\bm\theta^\ast),
\label{app:eq:dense-score-perturbation-definition}
\end{equation}
which captures both the differentiated moment-truncation remainder and the displacement of the pseudo-true parameter $\widetilde{\bm\theta}_l$ from $\bm\theta^\ast$.
Assume that these two effects satisfy the uniform bound \begin{equation} \sup_{i,j} \mathbb E\|\bm D_{i,j}^{(l)}\|^2 \le C(\Delta t)^{2l}. \label{app:eq:dense-score-perturbation} \end{equation}

Also assume joint consistency
$\widehat{\bm\theta}_n-\widetilde{\bm\theta}_l=o_p(1)$ and a uniform
law of large numbers for the normalized observed Hessian on a shrinking
neighborhood of $\widetilde{\bm\theta}_l$ containing
$\widehat{\bm\theta}_n$ with probability tending to one. Specifically,
if $\mathcal N_{n,T}$ is such a neighborhood, require
\begin{equation}
\sup_{\bm\theta\in\mathcal N_{n,T}}
\left\|
-\frac1{nT}\sum_{j=1}^n
\nabla_\theta^2\ell_l(\pmb\tau_j;\bm\theta)
-\mathcal I_\sigma(\bm\theta^\ast)
\right\|
\xrightarrow{p}0.
\label{app:eq:dense-hessian-lln}
\end{equation}
The moment expansions and parameter-derivative bounds imply that the difference between the truncated and exact expected Hessians is $\mathcal O((\Delta t)^l)$ per transition. Condition \eqref{app:eq:dense-hessian-lln} strengthens this deterministic approximation by controlling the corresponding centered empirical fluctuations uniformly over a neighborhood of $\widetilde{\bm\theta}_l$.

\begin{proposition}[Joint dense-sampling QMLE rate]
\label{app:prop:dense-qmle}
Under the preceding joint-limit conditions, including the
conditional Lindeberg condition and
Equations~\eqref{app:eq:dense-information-lln}--\eqref{app:eq:dense-hessian-lln},
suppose $\widetilde{\bm\theta}_l$ and $\widehat{\bm\theta}_n$ are
interior points of the parameter space with probability tending to one.
Then
\begin{equation}
\sqrt{nT}\,
(\widehat{\bm\theta}_n-\widetilde{\bm\theta}_l)
\xrightarrow{d}
N\!\left(\bm0,\mathcal I_\sigma(\bm\theta^\ast)^{-1}\right).
\label{app:eq:dense-qmle-limit}
\end{equation}
If, in addition, the empirical trajectory-score outer product divided
by $T$ satisfies a triangular-array law of large numbers, then
\begin{equation}
\widehat C_{l,n}
=T^{-1}\mathcal I_\sigma(\bm\theta^\ast)^{-1}
+o_p(T^{-1}).
\label{app:eq:dense-sandwich-limit}
\end{equation}
\end{proposition}

\begin{proof}
Equations~\eqref{app:eq:dense-standardization},
\eqref{app:eq:dense-moment-expansions}, and
\eqref{app:eq:dense-covariance-derivative} give
\begin{equation}
\begin{aligned}
c_{k,i,j}&=\mathcal O((\Delta t)^{1/2}),\\
D_{k,i,j}
&=\mathsf a_{i,j}^{-1/2}
\,\mathsf a_{i,j,k}\,
\mathsf a_{i,j}^{-1/2}
+\mathcal O(\Delta t).
\end{aligned}
\label{app:eq:dense-score-coefficient-orders}
\end{equation}
For the linear term in Equation~\eqref{app:eq:dense-score-decomposition},
the conditional variance is $\mathcal O(\Delta t)$.
For coordinates $k$ and $h$, the conditional covariance of the
quadratic terms in Equation~\eqref{app:eq:dense-score-decomposition} is
\begin{equation}
\begin{aligned}
&\operatorname{Cov}\!\left(
\tfrac12\{Z_{i,j}^\top D_{k,i,j}Z_{i,j}
-\operatorname{tr}(D_{k,i,j})\},
\tfrac12\{Z_{i,j}^\top D_{h,i,j}Z_{i,j}
-\operatorname{tr}(D_{h,i,j})\}
\mid x_{i,j}\right)\\
&\quad=\frac12\operatorname{tr}\!\left\{
\mathsf a_{i,j}^{-1}\mathsf a_{i,j,k}\,
\mathsf a_{i,j}^{-1}\mathsf a_{i,j,h}
\right\}
+\mathcal O(\Delta t).
\end{aligned}
\label{app:eq:dense-quadratic-score-covariance}
\end{equation}
The conditional covariance between the linear and quadratic terms is
$\mathcal O(\Delta t)$ by the third-moment bound.
Thus Equations~\eqref{app:eq:dense-information-lln} and
\eqref{app:eq:dense-quadratic-score-covariance} give the limiting
conditional covariance $\mathcal I_\sigma(\bm\theta^\ast)$ after
normalization by $nT$.
Within each trajectory, the transition scores form a martingale-difference
array, and the trajectories are independent across $j$.
The conditional Lindeberg condition and
Equation~\eqref{app:eq:dense-information-lln} therefore give
\begin{equation}
\frac1{\sqrt{nT}}
\sum_{j=1}^n\sum_{i=0}^{T-1}
\nabla_\theta\ell_{0,i}^j(\bm\theta^\ast)
\xrightarrow{d}
N\!\left(\bm0,\mathcal I_\sigma(\bm\theta^\ast)\right).
\label{app:eq:dense-exact-score-clt}
\end{equation}

For a fixed trajectory,
$\sum_i\bm D_{i,j}^{(l)}$ has mean zero because both the population score at $\widetilde{\bm\theta}_l$ and the exact-moment population score at $\bm\theta^\ast$ vanish. Minkowski's inequality and
Equation~\eqref{app:eq:dense-score-perturbation} give
\[
\mathbb E\left\|\sum_{i=0}^{T-1}\bm D_{i,j}^{(l)}\right\|^2
\leq C T^2(\Delta t)^{2l}.
\]
Independence across trajectories therefore yields
\begin{equation}
\mathbb E\left\|
\frac1{\sqrt{nT}}\sum_{j=1}^n\sum_{i=0}^{T-1}
\bm D_{i,j}^{(l)}
\right\|^2
\leq C T(\Delta t)^{2l}
=C(t_T-t_0)^{2l}T^{1-2l}\longrightarrow0
\label{app:eq:dense-perturbation-vanishes}
\end{equation}
for $l\geq1$. Equations~\eqref{app:eq:dense-exact-score-clt} and
\eqref{app:eq:dense-perturbation-vanishes} give the same limiting
distribution for the truncated score at $\widetilde{\bm\theta}_l$.

Applying an integral Taylor expansion to the score equation,
\begin{equation}
\sqrt{nT}(\widehat{\bm\theta}_n-\widetilde{\bm\theta}_l)
=
\left\{
-\frac1{nT}\int_0^1\sum_{j=1}^n
\nabla_\theta^2\ell_l
\!\left(\pmb\tau_j;
\widetilde{\bm\theta}_l
+u(\widehat{\bm\theta}_n-\widetilde{\bm\theta}_l)
\right)du
\right\}^{-1}
\frac1{\sqrt{nT}}\sum_{j=1}^n
\bm\psi_l(\pmb\tau_j;\widetilde{\bm\theta}_l).
\label{app:eq:dense-score-taylor}
\end{equation}
Equation~\eqref{app:eq:dense-hessian-lln} makes the inverse factor in
Equation~\eqref{app:eq:dense-score-taylor} converge in probability to
$\mathcal I_\sigma(\bm\theta^\ast)^{-1}$.
The limit in Equation~\eqref{app:eq:dense-qmle-limit} follows by
Slutsky's theorem from Equations~\eqref{app:eq:dense-exact-score-clt}
and \eqref{app:eq:dense-perturbation-vanishes}.

For the sandwich covariance estimator, the empirical trajectory-score outer product divided by $T$ converges in probability to $\mathcal I_\sigma(\bm\theta^\ast)$, while \[ \frac1T B_{l,n}(\widehat{\bm\theta}_n) \to_p -\mathcal I_\sigma(\bm\theta^\ast). \] 
Substitution into Equation~\eqref{eq:uncertainty-covariance} proves
Equation~\eqref{app:eq:dense-sandwich-limit}.
\end{proof}


\section{One-Step Bias Estimation and Approximation} 
\label{app:one-step-bias}

Building on the dense-sampling QMLE rate in
Appendix~\ref{app:dense-sampling-rate}, we analyze the one-step bias
estimate and its approximation to the exact parameter bias. For a generic trajectory
$\pmb\tau$, define
\[m_{q,i}(\bm\theta)
:=\widetilde{\pmb m}_{q,l}(\Delta t,\pmb s_{t_i^+},\bm\theta),
\qquad
G_{q,k,i}(\bm\theta):=\partial_{\theta_k}m_{q,i}(\bm\theta), 
\]
and
\[ R_{q,i}^{\mathrm{next}}(\bm\theta)
:=\widehat R_{q,l}^{\mathrm{next}}
(\Delta t,\pmb s_{t_i^+},\bm\theta) 
\qquad
q=1,2.
\]
The $k$th component
of the trajectory-level score-correction vector is
\begin{equation}
\begin{aligned}
d_{l,T}^{(k)}(\pmb\tau;\bm\theta)
:={}&\sum_{i=0}^{T-1}\biggl[
G_{1,k,i}^\top m_{2,i}^{-1}R_{1,i}^{\mathrm{next}}+\frac12\operatorname{tr}\!\left(
m_{2,i}^{-1}G_{2,k,i}m_{2,i}^{-1}R_{2,i}^{\mathrm{next}}
\right)
+\frac12\operatorname{tr}\!\left(
m_{2,i}^{-1}G_{2,k,i}m_{2,i}^{-1}
R_{1,i}^{\mathrm{next}}(R_{1,i}^{\mathrm{next}})^\top
\right)\biggr],
\end{aligned}
\label{app:eq:one-step-score-correction}
\end{equation}
where all quantities are evaluated at $\bm\theta$. Let
$d_{l,T}:=(d_{l,T}^{(1)},\ldots,d_{l,T}^{(K)})^\top$. 
The transition-level correction in
Equation~\eqref{eq:empirical-transition-score-bias} is aggregated
along a trajectory in Equation~\eqref{app:eq:one-step-score-correction}. Thus
\begin{equation}
\widehat{\bm\delta}_n =\frac1n\sum_{j=1}^n d_{l,T}(\pmb\tau_j;\widehat{\bm\theta}_n).
\label{app:eq:one-step-empirical-score}
\end{equation}

Define
\begin{equation}
\begin{aligned}
\bm\delta_{l,T}^{\mathrm{plug}}
&:=\mathbb E[d_{l,T}(\pmb\tau_1;\widetilde{\bm\theta}_l)],\\
\Lambda_{l,T}^{\mathrm{plug}}
&:=B_l(\widetilde{\bm\theta}_l),\\
\bm b_{\Delta,l}^{\mathrm{plug}}
&:=(\Lambda_{l,T}^{\mathrm{plug}})^{-1}
\bm\delta_{l,T}^{\mathrm{plug}}.
\end{aligned}
\label{app:eq:one-step-population-target}
\end{equation}
The plug-in rate uses the following bounds:
\begin{equation}
\begin{aligned}
\|\bm\delta_{l,T}^{\mathrm{plug}}\|
&=\mathcal O(T(\Delta t)^l),\\
\mathbb E\|d_{l,T}(\pmb\tau_1;\widetilde{\bm\theta}_l)\|^2
&=\mathcal O(T^2(\Delta t)^{2l}),\\
\|(\Lambda_{l,T}^{\mathrm{plug}})^{-1}\|
&=\mathcal O(T^{-1}),\\
\sup_{\bm\theta\in\mathcal N_{n,T}}
\left\|\frac1n\sum_{j=1}^n
\nabla_\theta d_{l,T}(\pmb\tau_j;\bm\theta)\right\|
&=\mathcal O_p(T(\Delta t)^l),\\
\|B_{l,n}(\widehat{\bm\theta}_n)
-\Lambda_{l,T}^{\mathrm{plug}}\|
&=\mathcal O_p\!\left(
Tn^{-1/2}
+T\|\widehat{\bm\theta}_n-\widetilde{\bm\theta}_l\|
\right).
\end{aligned}
\label{app:eq:one-step-regularity-bounds}
\end{equation}

\begin{proposition}[One-step bias plug-in rate]
\label{app:prop:dense-bias}
Under the preceding bounds and
Proposition~\ref{app:prop:dense-qmle},
\[
\widehat{\bm b}_n-\bm b_{\Delta,l}^{\mathrm{plug}}
=\mathcal O_p((nT)^{-1/2}).
\]
\end{proposition}
\begin{proof}
Independence of the trajectories and the second-moment bound in Equation~\eqref{app:eq:one-step-regularity-bounds} give
\begin{equation}
\frac1n\sum_{j=1}^n
d_{l,T}(\pmb\tau_j;\widetilde{\bm\theta}_l)
-\bm\delta_{l,T}^{\mathrm{plug}}
=\mathcal O_p(T(\Delta t)^l n^{-1/2}).
\label{app:eq:one-step-score-fluctuation}
\end{equation}
Combining the QMLE rate in Equation~\eqref{eq:dense-qmle-rate},
the score representation in Equation~\eqref{app:eq:one-step-empirical-score},
the derivative bound in Equation~\eqref{app:eq:one-step-regularity-bounds},
and Equation~\eqref{app:eq:one-step-score-fluctuation} gives
\begin{equation}
\widehat{\bm\delta}_n-\bm\delta_{l,T}^{\mathrm{plug}}
=\mathcal O_p\!\left(
T(\Delta t)^l n^{-1/2}
+T(\Delta t)^l(nT)^{-1/2}
\right).
\label{app:eq:one-step-score-error}
\end{equation}
Equations~\eqref{eq:empirical-parameter-bias} and \eqref{app:eq:one-step-population-target} give the decomposition
\begin{equation}
\begin{aligned}
\widehat{\bm b}_n-\bm b_{\Delta,l}^{\mathrm{plug}}
=(\Lambda_{l,T}^{\mathrm{plug}})^{-1}
(\widehat{\bm\delta}_n-\bm\delta_{l,T}^{\mathrm{plug}})
+\Bigl(\widehat\Lambda_{l,n}^{-1}
-(\Lambda_{l,T}^{\mathrm{plug}})^{-1}\Bigr)\widehat{\bm\delta}_n.
\end{aligned}
\label{app:eq:one-step-bias-decomposition}
\end{equation}
For the second term in Equation~\eqref{app:eq:one-step-bias-decomposition},
\begin{equation}
\widehat\Lambda_{l,n}^{-1}-(\Lambda_{l,T}^{\mathrm{plug}})^{-1}
=\widehat\Lambda_{l,n}^{-1}
(\Lambda_{l,T}^{\mathrm{plug}}-\widehat\Lambda_{l,n})
(\Lambda_{l,T}^{\mathrm{plug}})^{-1}.
\label{app:eq:one-step-inverse-identity}
\end{equation}
Equations~\eqref{app:eq:one-step-regularity-bounds},
\eqref{app:eq:one-step-score-error},
\eqref{app:eq:one-step-bias-decomposition}, and
\eqref{app:eq:one-step-inverse-identity} yield
\begin{equation}
\widehat{\bm b}_n-\bm b_{\Delta,l}^{\mathrm{plug}}
=\mathcal O_p\!\left(
\frac{(\Delta t)^l}{\sqrt n}
+\frac{(\Delta t)^l}{\sqrt{nT}}
\right).
\label{app:eq:one-step-intermediate-rate}
\end{equation}
Since $\Delta t=(t_T-t_0)/T$ and $l\ge1$,
\[
\frac{(\Delta t)^l}{\sqrt n}
=\frac{(t_T-t_0)^l}{\sqrt n\,T^l}
=O((nT)^{-1/2}),
\]
and the second term in Equation~\eqref{app:eq:one-step-intermediate-rate} has the same bound. This proves Proposition~\ref{app:prop:dense-bias}.
\end{proof}

Proposition~\ref{app:prop:dense-bias} controls sampling error relative to
$\bm b_{\Delta,l}^{\mathrm{plug}}$. The next result compares this population
target with $\bm b_{\Delta,l}$.

\begin{proposition}[One-step bias approximation]
\label{app:prop:one-step-accuracy}
Let $T\Delta t=t_T-t_0\in(0,\infty)$ and $l\geq1$. Let $\mathcal U_T$ be a convex neighborhood containing both $\bm\theta^\ast$ and $\widetilde{\bm\theta}_l$. Assume that differentiation can be interchanged with trajectory expectation and, uniformly in $T$,
\begin{equation}
\begin{aligned}
\left\|\mathbb E_{\bm\theta^\ast,\eta_T}
[d_{l,T}(\pmb\tau;\bm\theta^\ast)]-\bm\delta_l\right\|
&\le CT(\Delta t)^{l+1},\\
\sup_{\bm\theta\in\mathcal U_T}
\left\|D_\theta\mathbb E_{\bm\theta^\ast,\eta_T}
[d_{l,T}(\pmb\tau;\bm\theta)]\right\|
&\le CT(\Delta t)^l,\\
\sup_{\bm\theta\in\mathcal U_T}\|D_\theta B_l(\bm\theta)\|
&\le CT.
\end{aligned}
\label{app:eq:one-step-accuracy-assumptions}
\end{equation}
Suppose further that
$\|\bm b_{\Delta,l}\|\leq C(\Delta t)^l$ and
$\|B_l(\widetilde{\bm\theta}_l)^{-1}\|
+\|\overline\Lambda_l^{-1}\|\leq C/T$.
Then
\[
\bm b_{\Delta,l}^{\mathrm{plug}}-\bm b_{\Delta,l}
=\mathcal O((\Delta t)^{l+1}).
\]
\end{proposition}
\begin{proof}
Comparing the exact residual score in
Equation~\eqref{app:eq:transition-score-bias} with the
omitted-generator score in Equation~\eqref{app:eq:one-step-score-correction}
gives the first bound in Equation~\eqref{app:eq:one-step-accuracy-assumptions}:
the covariance score differs by
$\mathcal O((\Delta t)^{l+1})$ per transition, and the mean and quadratic
mean-residual differences are of at least that order.

Using the population target in Equation~\eqref{app:eq:one-step-population-target},
the first two bounds in Equation~\eqref{app:eq:one-step-accuracy-assumptions},
and $\widetilde{\bm\theta}_l-\bm\theta^\ast=-\bm b_{\Delta,l}$ gives
\begin{equation}
\begin{aligned}
\|\bm\delta_{l,T}^{\mathrm{plug}}-\bm\delta_l\|
&\le CT(\Delta t)^l\|\bm b_{\Delta,l}\|
+CT(\Delta t)^{l+1}\\
&\le CT\{(\Delta t)^{2l}+(\Delta t)^{l+1}\}.
\end{aligned}
\label{app:eq:one-step-score-approximation}
\end{equation}
Likewise, Equation~\eqref{app:eq:parameter-bias-rate} and the third bound
in Equation~\eqref{app:eq:one-step-accuracy-assumptions} imply
\begin{equation}
\|B_l(\widetilde{\bm\theta}_l)-\overline\Lambda_l\|
\le CT\|\bm b_{\Delta,l}\|
=\mathcal O(T(\Delta t)^l).
\label{app:eq:one-step-curvature-approximation}
\end{equation}

Equations~\eqref{eq:exact-parameter-bias} and
\eqref{app:eq:one-step-population-target} yield
\begin{equation}
\begin{aligned}
\|\bm b_{\Delta,l}^{\mathrm{plug}}-\bm b_{\Delta,l}\|
\le
\|B_l(\widetilde{\bm\theta}_l)^{-1}\|
\|\bm\delta_{l,T}^{\mathrm{plug}}-\bm\delta_l\|
+\|B_l(\widetilde{\bm\theta}_l)^{-1}
-\overline\Lambda_l^{-1}\|\,\|\bm\delta_l\|.
\end{aligned}
\label{app:eq:one-step-accuracy-decomposition}
\end{equation}
Use the score rate in Equation~\eqref{app:eq:score-bias-rate}
within Equation~\eqref{app:eq:one-step-accuracy-decomposition}.
The inverse-matrix identity, together with
Equations~\eqref{app:eq:one-step-score-approximation} and
\eqref{app:eq:one-step-curvature-approximation}, gives
\[
\|\bm b_{\Delta,l}^{\mathrm{plug}}-\bm b_{\Delta,l}\|
\le C\{(\Delta t)^{l+1}+(\Delta t)^{2l}\}
=\mathcal O((\Delta t)^{l+1}),
\]
because $2l\ge l+1$ for $l\ge1$.
\end{proof}


\section{Proofs of the Main Convergence Results}
\label{app:proofs}

\begin{proof}[Proof of Theorem~\ref{convergence1}]
Let $r_{n,T}=(nT)^{-1/2}$ under the conditions of
Propositions~\ref{app:prop:dense-qmle}--\ref{app:prop:dense-bias}.
By Equation~\eqref{eq:dense-qmle-rate},
\begin{equation}
\bm e_n-\bm b_{\Delta,l}
=\widetilde{\bm\theta}_l-\widehat{\bm\theta}_n
=\mathcal O_p(r_{n,T}).
\label{app:eq:exploration-parameter-error}
\end{equation}
The derivative bounds in Theorem~\ref{convergence1} and
Equation~\eqref{app:eq:exploration-parameter-error} give
\begin{equation}
\bm g_{n,i}^\pi-\bm g_{\infty,i}^\pi
=\mathcal O_p(r_{n,T}),
\qquad
H_{n,i}^\pi-H_{\infty,i}^\pi
=\mathcal O_p(r_{n,T}).
\label{app:eq:exploration-derivative-errors}
\end{equation}
Using $\bigl||x|-|y|\bigr|\le|x-y|$ in
Equation~\eqref{eq:value-discrepancy-criterion},
the difference in the first assertion of Theorem~\ref{convergence1} is bounded by
\begin{equation}
\begin{aligned}
&\left|u_{n,i}^\pi-
\left| (\bm g_{\infty,i}^\pi)^\top\bm b_{\Delta,l}
+\tfrac12\bm b_{\Delta,l}^\top H_{\infty,i}^\pi
\bm b_{\Delta,l}\right|\right|\\
&\le
\|\bm g_{n,i}^\pi-\bm g_{\infty,i}^\pi\|\,\|\bm e_n\|
+\|\bm g_{\infty,i}^\pi\|\,\|\bm e_n-\bm b_{\Delta,l}\|\\
&\quad+\tfrac12\|H_{n,i}^\pi-H_{\infty,i}^\pi\|\,\|\bm e_n\|^2
+\tfrac12\|H_{\infty,i}^\pi\|\,
\|\bm e_n-\bm b_{\Delta,l}\|\,
(\|\bm e_n\|+\|\bm b_{\Delta,l}\|)
=\mathcal O_p(r_{n,T}).
\end{aligned}
\label{app:eq:exploration-deterministic-comparison}
\end{equation}

For the Gaussian surrogate in Equation~\eqref{eq:gaussian-error-surrogate},
condition on $\mathcal D_n$ and write
$\widetilde{\bm e}_n=\widehat{\bm b}_n+\epsilon_n$ with
$\epsilon_n\mid\mathcal D_n\sim N(\bm0,\widehat\Sigma_n)$.
Equations~\eqref{app:eq:one-step-intermediate-rate} and
\eqref{app:eq:exploration-derivative-errors}, together with the covariance
rate in Proposition~\ref{app:prop:dense-qmle}, imply
\begin{equation}
\begin{aligned}
\widehat u_{n,i}^\pi(\pmb s,\pmb a)
=\mathbb E\Bigl[\bigl|
(\bm g_{\infty,i}^\pi)^\top
(\bm b_{\Delta,l}^{\mathrm{plug}}+\epsilon_n)
+\tfrac12(\bm b_{\Delta,l}^{\mathrm{plug}}+\epsilon_n)^\top
H_{\infty,i}^\pi
(\bm b_{\Delta,l}^{\mathrm{plug}}+\epsilon_n)
\bigr|\mid\mathcal D_n\Bigr]
+\mathcal O_p(r_{n,T}).
\end{aligned}
\label{app:eq:exploration-gaussian-comparison}
\end{equation}
This is the first Gaussian representation in Theorem~\ref{convergence1}.

By Proposition~\ref{app:prop:dense-qmle},
$\widehat\Sigma_n=\widehat C_{l,n}/n=\mathcal O_p((nT)^{-1})$.
Consequently,
\begin{equation}
\mathbb E[\|\epsilon_n\|\mid\mathcal D_n]
=\mathcal O_p((nT)^{-1/2}),
\qquad
\mathbb E[\|\epsilon_n\|^2\mid\mathcal D_n]
=\mathcal O_p((nT)^{-1}).
\label{app:eq:exploration-gaussian-moments}
\end{equation}
Expanding the quadratic expression in
Equation~\eqref{eq:implementable-uncertainty} gives the exact difference
\begin{equation}
\begin{aligned}
&\widehat u_{n,i}^\pi(\pmb s,\pmb a)
-\left| (\bm g_{n,i}^\pi)^\top\widehat{\bm b}_n
+\tfrac12\widehat{\bm b}_n^\top H_{n,i}^\pi
\widehat{\bm b}_n\right|\\
&=\mathbb E\Bigl[
\Bigl|(\bm g_{n,i}^\pi)^\top\widehat{\bm b}_n
+\tfrac12\widehat{\bm b}_n^\top H_{n,i}^\pi
\widehat{\bm b}_n
+(\bm g_{n,i}^\pi+H_{n,i}^\pi\widehat{\bm b}_n)^\top\epsilon_n
+\tfrac12\epsilon_n^\top H_{n,i}^\pi\epsilon_n\Bigr|
-\Bigl|(\bm g_{n,i}^\pi)^\top\widehat{\bm b}_n
+\tfrac12\widehat{\bm b}_n^\top H_{n,i}^\pi
\widehat{\bm b}_n\Bigr|
\,\Bigm|\mathcal D_n\Bigr].
\end{aligned}
\label{app:eq:exploration-gaussian-difference}
\end{equation}
Thus $\bigl||x|-|y|\bigr|\le|x-y|$ and
Equation~\eqref{app:eq:exploration-gaussian-moments} yield
\begin{equation}
\begin{aligned}
&\left|\widehat u_{n,i}^\pi(\pmb s,\pmb a)
-\left|(\bm g_{n,i}^\pi)^\top\widehat{\bm b}_n
+\tfrac12\widehat{\bm b}_n^\top H_{n,i}^\pi
\widehat{\bm b}_n\right|\right|\\
&\le\|\bm g_{n,i}^\pi+H_{n,i}^\pi\widehat{\bm b}_n\|
\mathbb E[\|\epsilon_n\|\mid\mathcal D_n]
+\tfrac12\|H_{n,i}^\pi\|
\mathbb E[\|\epsilon_n\|^2\mid\mathcal D_n]
=\mathcal O_p(r_{n,T}).
\end{aligned}
\label{app:eq:exploration-gaussian-bound}
\end{equation}
Equations~\eqref{app:eq:one-step-intermediate-rate} and
\eqref{app:eq:exploration-derivative-errors} also give
\begin{equation}
\begin{aligned}
&\Bigl|
\left|(\bm g_{n,i}^\pi)^\top\widehat{\bm b}_n
+\tfrac12\widehat{\bm b}_n^\top H_{n,i}^\pi\widehat{\bm b}_n\right|\\
&\quad-\left|(\bm g_{\infty,i}^\pi)^\top
\bm b_{\Delta,l}^{\mathrm{plug}}
+\tfrac12(\bm b_{\Delta,l}^{\mathrm{plug}})^\top
H_{\infty,i}^\pi\bm b_{\Delta,l}^{\mathrm{plug}}\right|
\Bigr|=\mathcal O_p(r_{n,T}).
\end{aligned}
\label{app:eq:exploration-fitted-population-comparison}
\end{equation}
Equations~\eqref{app:eq:exploration-gaussian-bound} and
\eqref{app:eq:exploration-fitted-population-comparison} prove the second
assertion of Theorem~\ref{convergence1}.
For fixed $T$, replacing $r_{n,T}$ by $n^{-1/2}$ in these comparisons gives
the stated fixed-$T$ rate.
\end{proof}

\paragraph{Uniform value-smoothness assumption}
For some $T_0\ge 1$, let $\mathcal N$ be a parameter neighborhood that contains the line segment joining $\widetilde{\bm\theta}_l$ and $\bm\theta^\ast$ for all $T\ge T_0$. The value-smoothness condition is
\begin{equation}
\sup_{T\geq T_0}\sup_{\pi\in\Pi}\sup_{\bm\theta\in\mathcal N}
\|\nabla_\theta^rJ_T(\pi;\bm\theta)\|<\infty,
\qquad r=1,2,3.
\label{app:eq:uniform-value-smoothness}
\end{equation}

\begin{proof}[Proof of Theorem~\ref{convergence2}]
Equation~\eqref{app:eq:uniform-value-smoothness} and Taylor's theorem give
\begin{equation}
\sup_{\pi\in\Pi}
\bigl|J(\pi;\bm\theta^\ast)
-J(\pi;\widetilde{\bm\theta}_l)
-C_{\mathrm{bias}}(\pi)\bigr|
=\mathcal O(\|\bm b_{\Delta,l}\|^3)
=\mathcal O((\Delta t)^{3l}).
\label{app:eq:uniform-value-taylor-remainder}
\end{equation}
The last equality uses Equation~\eqref{app:eq:parameter-bias-rate}.

The uniform Hessian bound in Equation~\eqref{app:eq:uniform-value-smoothness}
and $C_l(\widetilde{\bm\theta}_l)=\mathcal O(T^{-1})$ from
Appendix~\ref{app:dense-sampling-rate} imply
\begin{equation}
\sup_{\pi\in\Pi}\frac1{2n}
\left|\operatorname{tr}\!\left(
\nabla_\theta^2J(\pi;\widetilde{\bm\theta}_l)
C_l(\widetilde{\bm\theta}_l)\right)\right|
=\mathcal O((nT)^{-1}).
\label{app:eq:uniform-value-variance-remainder}
\end{equation}
Equation~\eqref{eq:oracle-value-remainder} follows from
Equations~\eqref{app:eq:uniform-value-taylor-remainder} and
\eqref{app:eq:uniform-value-variance-remainder}.

Define the population one-step bias correction by \[ C_{\mathrm{bias}}^{\mathrm{1step},l}(\pi) := \nabla_\theta J(\pi;\widetilde{\bm\theta}_l)^\top \bm b_{\Delta,l}^{\mathrm{plug}} +\frac12 (\bm b_{\Delta,l}^{\mathrm{plug}})^\top \nabla_\theta^2J(\pi;\widetilde{\bm\theta}_l) \bm b_{\Delta,l}^{\mathrm{plug}}, \] and denote
its difference from the oracle correction
by
\[ \epsilon_{\mathrm{bias},l} := \sup_{\pi\in\Pi} \bigl| C_{\mathrm{bias}}^{\mathrm{1step},l}(\pi) -C_{\mathrm{bias}}(\pi) \bigr|. \]

Separate the two corrections in Equation~\eqref{eq.policyOptimizationObjective} as
\begin{align*}
\widehat C_{\mathrm{bias},n}(\pi) :={}&
\nabla_\theta J(\pi;\widehat{\bm\theta}_n)^\top\widehat{\bm b}_n
+\frac12\widehat{\bm b}_n^\top
\nabla_\theta^2J(\pi;\widehat{\bm\theta}_n)\widehat{\bm b}_n,\\
\widehat C_{\mathrm{var},n}(\pi) :={}&
\frac12\operatorname{tr}\!\left(
\nabla_\theta^2J(\pi;\widehat{\bm\theta}_n)\widehat\Sigma_n\right).
\end{align*}
The corresponding population variance correction is
\[
C_{\mathrm{var},n}^{\mathrm{pop}}(\pi)
:=\frac1{2n}\operatorname{tr}\!\left(
\nabla_\theta^2J(\pi;\widetilde{\bm\theta}_l)
C_l(\widetilde{\bm\theta}_l)\right).
\]

Let 
$\bm d_l:=\bm b_{\Delta,l}^{\mathrm{plug}}-\bm b_{\Delta,l}$.
For each policy, with the value derivatives evaluated at
$\widetilde{\bm\theta}_l$,
\begin{equation}
C_{\mathrm{bias}}^{\mathrm{1step},l}(\pi)
-C_{\mathrm{bias}}(\pi)
=\bigl(\nabla_\theta J+\nabla_\theta^2J\,\bm b_{\Delta,l}\bigr)^\top
\bm d_l
+\frac12\bm d_l^\top\nabla_\theta^2J\,\bm d_l.
\label{app:eq:one-step-value-bias-difference}
\end{equation}
By Equations~\eqref{app:eq:uniform-value-smoothness} and
\eqref{app:eq:one-step-value-bias-difference}, together with
Proposition~\ref{app:prop:one-step-accuracy},
\begin{equation}
\epsilon_{\mathrm{bias},l}=\mathcal O((\Delta t)^{l+1}).
\label{app:eq:one-step-value-bias-rate}
\end{equation}

Equation~\eqref{eq:dense-qmle-rate} and the uniform gradient bound in
Equation~\eqref{app:eq:uniform-value-smoothness} give
\begin{equation}
\sup_{\pi\in\Pi}
|J(\pi;\widehat{\bm\theta}_n)-J(\pi;\widetilde{\bm\theta}_l)|
=\mathcal O_p((nT)^{-1/2}).
\label{app:eq:uniform-fitted-value-error}
\end{equation}
Equation~\eqref{eq:dense-qmle-rate}, the derivative bounds in
Equation~\eqref{app:eq:uniform-value-smoothness}, and
Proposition~\ref{app:prop:dense-bias} then give
\begin{equation}
\sup_{\pi\in\Pi}
\bigl|\widehat C_{\mathrm{bias},n}(\pi)
-C_{\mathrm{bias}}^{\mathrm{1step},l}(\pi)\bigr|
=\mathcal O_p((nT)^{-1/2}).
\label{app:eq:uniform-fitted-bias-error}
\end{equation}

The uniform Hessian bound in Equation~\eqref{app:eq:uniform-value-smoothness},
fixed parameter dimension, and
$\widehat\Sigma_n=\mathcal O_p((nT)^{-1})$ from
Proposition~\ref{app:prop:dense-qmle} give
\begin{equation}
\sup_{\pi\in\Pi}|\widehat C_{\mathrm{var},n}(\pi)|
\leq\frac K2
\sup_{\pi\in\Pi}\|\nabla_\theta^2J(\pi;\widehat{\bm\theta}_n)\|_{\mathrm{op}}
\|\widehat\Sigma_n\|_{\mathrm{op}}
=\mathcal O_p((nT)^{-1}).
\label{app:eq:uniform-variance-correction}
\end{equation}
Using Equation~\eqref{app:eq:uniform-value-variance-remainder} for the
population term and Equation~\eqref{app:eq:uniform-variance-correction}
for the fitted term gives
\begin{equation}
\sup_{\pi\in\Pi}
|\widehat C_{\mathrm{var},n}(\pi)
-C_{\mathrm{var},n}^{\mathrm{pop}}(\pi)|
=\mathcal O_p((nT)^{-1}).
\label{app:eq:uniform-variance-error}
\end{equation}
Equations~\eqref{eq:oracle-value-remainder},
\eqref{app:eq:one-step-value-bias-rate},
\eqref{app:eq:uniform-fitted-value-error},
\eqref{app:eq:uniform-fitted-bias-error}, and
\eqref{app:eq:uniform-variance-error} yield
\begin{equation}
\begin{aligned}
&\sup_{\pi\in\Pi}
|\bar J_n(\pi)-J(\pi;\bm\theta^\ast)|\\
&\leq\epsilon_{\Delta t,n}+\epsilon_{\mathrm{bias},l}
+\mathcal O_p((nT)^{-1/2})+\mathcal O_p((nT)^{-1})
=\mathcal O((\Delta t)^{l+1})
+\mathcal O_p((nT)^{-1/2}),
\end{aligned}
\label{app:eq:uniform-policy-value-error}
\end{equation}
since $3l\geq l+1$ for $l\geq1$.

By Equation~\eqref{eq:bias_corrected_objective_intro},
$\widehat\pi_n$ maximizes $\bar J_n$, while $\pi^\ast$ maximizes
$J(\pi;\bm\theta^\ast)$. Hence
\begin{equation}
J(\pi^\ast;\bm\theta^\ast)
-J(\widehat\pi_n;\bm\theta^\ast)
\le 2\sup_{\pi\in\Pi}
|\bar J_n(\pi)-J(\pi;\bm\theta^\ast)|.
\label{app:eq:policy-regret-comparison}
\end{equation}
Substitution of Equation~\eqref{app:eq:uniform-policy-value-error}
into Equation~\eqref{app:eq:policy-regret-comparison} completes the proof.
\end{proof}

\end{document}